\documentclass[sigconf,nonacm]{acmart}

\usepackage{algorithm}
\usepackage{algpseudocode}
\usepackage{bm}
\usepackage{booktabs}
\usepackage{color}
\usepackage{comment}
\usepackage{enumerate}
\usepackage{multirow}
\usepackage{subcaption}
\usepackage{url}

\algnewcommand{\Initialize}[1]{%
  \State \textbf{Initialize:}
  \State \hspace*{\algorithmicindent}\parbox[t]{0.8\linewidth}{\raggedright #1}
}

\newcommand{\midheader}[2]{%
        \midrule\topmidheader{#1}{#2}}
\newcommand\topmidheader[2]{\multicolumn{#1}{c}{#2}\\%
                \addlinespace[0.5ex]}

\newcommand\blfootnote[1]{%
  \begingroup
  \renewcommand\thefootnote{}\footnote{#1}%
  \addtocounter{footnote}{-1}%
  \endgroup
}

\def \model{\textsc{GCNmf}}

\AtBeginDocument{%
  \providecommand\BibTeX{{%
    \normalfont B\kern-0.5em{\scshape i\kern-0.25em b}\kern-0.8em\TeX}}}

\begin{document}
\newcommand{\IfExtendedElse}[2]{%
  \ifnum\pdfstrcmp{\extendedVersion}{True}=0
    \ifnum\pdfstrcmp{}{#1}=0 {} \else\protect{#1}\fi%
  \else
    \ifnum\pdfstrcmp{}{#2}=0 {} \else\protect{#2}\fi%
  \fi\ignorespaces}

\newcommand{\IfExtended}[1]{%
  \ifnum\pdfstrcmp{\extendedVersion}{False}=0 {} \else \protect{#1}\fi\ignorespaces}

\def\extendedVersion{False} 

\newcommand{\wrt}{w.r.t.~}
\newcommand{\ie}{i.e.,~}
\newcommand{\eg}{e.g.~}
\newcommand{\aka}{a.k.a.~}
\newcommand{\eat}[1]{}
\newcommand{\footurl}[1]{\footnote{\textsf{\small #1}}}

\newcommand{\sidenote}[4]{\todo[fancyline,inline,prepend,caption={\textbf{[#1]}}]{ #3}\textcolor{#2}{#4}}
\newcommand{\liu}[2][]{\sidenote{XL}{orange}{#1}{#2}}

\newcommand{\remove}[1]{}
\newcommand{\final}[1]{#1}
\newcommand{\sys}{{\textsc{Ours}}}
\newcommand{\afact}{A-fact}
\newcommand{\powerset}[1]{{\ensuremath{{\mathcal P}(#1)}}}
\newcommand{\years}{{\textsc{Years}}}
\newcommand{\intervals}{{\ensuremath{\mathcal T}}}
\newcommand{\sources}{{\textsc{Src}}}
\newcommand{\emptyinterval}{{[~]}}
\newcommand{\upperBound}{\ensuremath{\top}}
\newcommand{\T}{{\ensuremath{\mathcal T}}}
\newcommand{\layer}[1]{{\texttt{#1}}}
\newcommand{\D}{\layer{D}}
\newcommand{\Y}{\layer{Y}}
\newcommand{\Q}{\layer{Q}}
\newcommand{\M}{\layer{M}}
\newcommand{\W}{\layer{W}}
\renewcommand{\d}{\layer{d}}
\newcommand{\h}{\layer{h}}
\newcommand{\m}{\layer{m}}
\newcommand{\subject}{\layer{s}}
\newcommand{\predicate}{\layer{p}}
\newcommand{\object}{\layer{o}}

\newcommand{\ourmodel}{\textsc{GCNmf}}
\newcommand{\realnumbers}{{\rm I\!R}}
\newcommand{\integernumbers}{{\rm I\!N}}
\newcommand{\revision}[1]{#1}

\title[Graph Convolutional Networks for Graphs Containing Missing Features]{Graph Convolutional Networks for Graphs\\Containing Missing Features}

\settopmatter{authorsperrow=1}

\author{Hibiki Taguchi$^{1,2}$, Xin Liu$^{2,3,*}$, Tsuyoshi Murata$^{1,3}$}

\affiliation{
 \institution{$^1$Dept. of Computer Science, Tokyo Institute of Technology, Japan}
 \institution{$^2$Artificial Intelligence Research Center, National Institute of Advanced Industrial Science and Technology, Japan}
 \institution{$^3$AIST-Tokyo Tech Real World Big-Data Computation Open Innovation Laboratory, Japan}
 \institution{taguchi@net.c.titech.ac.jp,  xin.liu@aist.go.jp,  murata@c.titech.ac.jp}
}

\renewcommand{\shortauthors}{H. Taguchi et al.}

\begin{abstract}
Graph Convolutional Network (GCN) has experienced great success in graph analysis tasks. It works by smoothing the node features across the graph. The current GCN models overwhelmingly assume that the node feature information is complete. However, real-world graph data are often incomplete and containing missing features. Traditionally, people have to estimate and fill in the unknown features based on imputation techniques and then apply GCN. However, the process of feature filling and graph learning are separated, resulting in degraded and unstable performance. This problem becomes more serious when a large number of features are missing. We propose an approach that adapts GCN to graphs containing missing features. In contrast to traditional strategy, our approach integrates the processing of missing features and graph learning within the same neural network architecture. Our idea is to represent the missing data by Gaussian Mixture Model (GMM) and calculate the expected activation of neurons in the first hidden layer of GCN, while keeping the other layers of the network unchanged. This enables us to learn the GMM parameters and network weight parameters in an end-to-end manner. Notably, our approach does not increase the computational complexity of GCN and it is consistent with GCN when the features are complete. We demonstrate through extensive experiments that our approach significantly outperforms the imputation based methods in node classification and link prediction tasks. We show that the performance of our approach for the case with a low level of missing features is even superior to GCN for the case with complete features.
\blfootnote{$*$Corresponding author.}
\end{abstract}

%
%


\maketitle

\section{Introduction}
\label{sec:introduction}
Graphs are used in many branches of science as a way to represent the patterns of connections between the components of complex systems, including social analysis, product recommendation, web search, disease identification, brain function analysis, and many more. 

In recent years there is a surge of interest in learning on graph data. Graph embedding \cite{cui2018survey_net_emb,wang2017kgsurvey,ji2020survey} aims to learn low-dimensional vector representations for nodes or edges. The learned representations encode structural and semantic information transcribed from the graph and can be used directly as the features for downstream graph analysis tasks. Representative works on graph embedding include random walk and skip-gram model based methods \cite{Perozzi2014deep}, matrix factorization based approaches \cite{qiu2017unifying_dw_n2v_line, Liu2019NetEmbGRAbetta}, edge reconstruction based methods \cite{tang2015line}, and deep learning based algorithms \cite{wang2016SDNE, pan2018adversarially}, etc.

Meanwhile, graph neural network (GNN) \cite{scarselli2008gnn_model, zhou2018gnn_review, Wu2019survey_gnn}, as a type of neural network architectures that can operate on graph structure, has achieved superior performance in graph analysis and shown promise in various applications such as visual question answering \cite{narasimhan2018VisualQA_GCN}, point clouds classification and segmentation \cite{simonovsky2017DynamicFilterGCN}, fraud detection \cite{liu2020alleviating}, machine translation \cite{bastings2017MachineTrans_GCN}, molecular fingerprints prediction \cite{duvenaud2015GCN_MolecularFingerprints}, protein interface prediction \cite{fout2017ProteinInterfacePrediction_GCN}, topic modeling \cite{yang2020gaton}, and social recommendation \cite{ying2018GCN_Recommendation}.

Among various kinds of GNNs, graph convolutional network (GCN) \cite{kipf2016GCN}, a simplified version of spectral graph convolutional networks \cite{Shuman2013signal_processing_graph}, has attracted a large amount of attention. GCN and its subsequent variants can be interpreted as smoothing the node features in the neighborhoods guided by the graph structure, and have experienced great success in graph analysis tasks, such as node classification \cite{kipf2016GCN}, graph classification \cite{zhang2018DLGraphClassification}, link prediction \cite{kipf2016variational}, graph similarity estimation \cite{bai2019Simgnn}, node ranking \cite{sunil2019BetweennessGNN, chen2019CentralityInfoGraphConvolution}, and community detection \cite{jin2019GCN_MRF_CommunityDet, Choong2018CommunityVAE}.

\begin{figure*}[!t]
  \centering
  \includegraphics[width=0.90\textwidth]{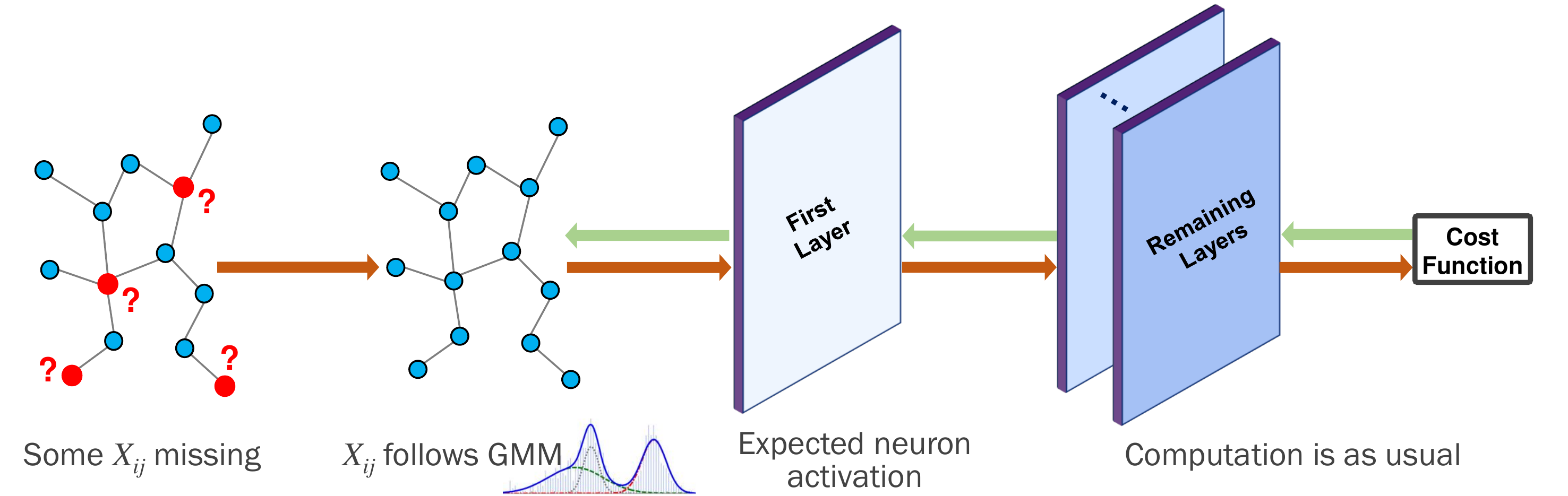}
  \caption{The architecture of our model.}
  \label{fig:archi}
\end{figure*}

The current GCN-like models assume that the node feature information is complete. However, real-world graph data are often incomplete and containing missing node features. Much of the missing features arise from the following sources. First, some features can be missing because of mechanical/electronic failures or human errors during the data collection process. Secondly, it can be prohibitively expensive or even impossible to collect the complete data due to its large size. For example, social media companies such as Twitter and Facebook have restricted the crawlers to collect the whole data. Thirdly, we cannot obtain sensitive personal information. In a social network, many users are unwilling to provide information such as address, nationality, and age to protect personal privacy. Finally, graphs are dynamic in nature, and thus newly joined nodes often have very little information. All these aspects result in graphs containing missing features.

To deal with the above problem, the traditional strategy is to estimate and fill in the unknown values before applying GCN. For this purpose, people have proposed imputation techniques such as mean imputation \cite{GarciaLaencina2010,yi2020why}, soft imputation based on singular value decomposition \cite{mazumder2010soft}, and machine learning methods such as \textit{k-NN} model \cite{batista2002study}, random forest \cite{Stekhoven2011missforest}, autoencoder \cite{Kingma2014,spinelli2019ginn}, generative adversarial network (GAN) \cite{yoon2018gain,Luo2018Multi,li2018learning}. However, the process of feature filling and graph learning are separated. Our experiments reveal that this strategy results in degraded and unstable performance, especially when a large number of features are missing.

In this paper, we propose an approach that can adapt GCN to graphs containing missing features. In contrast to traditional strategy, our approach integrates the processing of missing features and graph learning within the same neural network architecture and thus can enhance the performance. Our approach is motivated by Gaussian Mixture Model Compensator (GMMC) \cite{smieja2018GMMC} for processing missing data in neural networks. The main idea is to represent the missing data by Gaussian Mixture Model (GMM) and calculate the expected activation of neurons in the first hidden layer, while keeping the other layers of the network architecture unchanged (Figure \ref{fig:archi}). Although this idea is implemented in simple neural networks such as autoencoder and multilayer perceptron, it has not yet been extended to complex neural networks such as RNN, CNN, GNN, and sequence-to-sequence models. The main reason is due to the difficulty in unifying the representation of missing data and calculation of the expected activation of neurons. In particular, simply using GMM to represent the missing data will even complicate the network architecture, which hinders us from calculating the expected activation in closed form. We propose a novel way to unify the representation of missing features and calculation of the expected activation of the first layer neurons in GCN. Specifically, we skillfully represent the missing features by introducing only a small number of parameters in GMM and derive the analytic solution of the expected activation of neurons. As a result, our approach can arm GCN against missing features without increasing the computational complexity and our approach is consistent with GCN when the features are complete.

Our contributions are summarized as follows:

\begin{itemize}
\item We propose an elegant and unified way to transform the incomplete features to variables that follow mixtures of Gaussian distributions.
\item Based on the transformation, we derive the analytic solution to calculate the expected activation of neurons in the first layer of GCN.
\item We propose the whole network architecture for learning on graphs containing missing features. We prove that our model is consistent with GCN when the features are complete.
\item We perform extensive experiments and demonstrate that our approach significantly outperforms imputation based methods.
\end{itemize}

The rest of the paper is organized as follows. The next section summarizes the recent literature on GCN and methods for processing missing data. Section~\ref{sec:preliminary} reviews GCN. Section~\ref{sec:approach} introduces our approach. Section~\ref{sec:experiment} reports experiment results. Finally, Section~\ref{sec:conclusion} presents our concluding remarks.

\section{Related Work}
\label{sec:related}

\subsection{Graph Convolutional Networks}

GNNs are deep learning models aiming at addressing graph-related tasks \cite{scarselli2008gnn_model, zhou2018gnn_review, Wu2019survey_gnn}. Among various kinds of GNNs, GCN \cite{kipf2016GCN}, which simplifies the previous spectral graph convolutional networks \cite{Shuman2013signal_processing_graph} by restricting the filters to operate in one-hop neighborhood, has attracted a large amount of attention due to its simplicity and high performance. GCN can be interpreted as smoothing the node features in the neighborhoods, and this model achieves great success in the node classification task.

There are a series of works following GCN. GAT extends GCN by imposing the attention mechanism on the neighboring weight assignment \cite{velivckovic2017GAT}. AGCN learns hidden structural relations unspecified by the graph adjacency matrix and constructs a residual graph adjacency matrix \cite{li2018AdaptiveGCN}. TO-GCN utilizes potential information by jointly refining the network topology \cite{yang2019topology}. GCLN introduces ladder-shape architecture to increase the depth of GCN while overcoming the over-smoothing problem \cite{wan2020goingdeep}. MixHop introduces higher-order feature aggregation, which enables us to capture mixing neighbors' information \cite{sami2019mixhop}. There is also work on extending GCN to handle noisy and sparse node features \cite{shi2019featureattention}.

Training GCN usually requires to save the whole graph data into memory. To solve this problem, sampling strategy \cite{chen2018fastgcn} and batch training \cite{chiang2019cluster} are proposed. Moreover, FastGCN reduces the complexity of GCN through successively removing nonlinearities and collapsing weight matrices between consecutive layers \cite{chen2018fastgcn}.

While achieving excellent performance in graph analysis tasks, GCN is known to be vulnerable to adversarial attacks \cite{zugner2018adversarial,dai2018adversarial}. To address this problem, researchers have proposed robust models such as RGCN
that adopts Gaussian distributions as the hidden representations of nodes in each convolutional layer \cite{zhu2019robust} and a new learning principle that improves the robustness of GCN \cite{zugner2019certifiable}.

We note that all of the models mentioned above assume that the node feature information is complete.

\subsection{Learning with Missing Data}

Incomplete and missing data is common in real-world applications. Methods for handling such data can be categorized into two classes. The first class completes the missing data before using conventional machine learning algorithms. Imputation techniques are widely used for data completion, such as mean imputation \cite{GarciaLaencina2010}, matrix completion via matrix factorization \cite{koren2009mf} and singular value decomposition (SVD) \cite{mazumder2010soft}, and multiple imputation \cite{rubin2004multiple,buuren2010mice}. Machine learning models are also employed to estimate missing values, such as \textit{k-NN} model \cite{batista2002study}, random forest \cite{Stekhoven2011missforest}, autoencoder \cite{Kingma2014,spinelli2019ginn}, generative adversarial network (GAN) \cite{yoon2018gain,Luo2018Multi,li2018learning}. However, imputation methods are not always competent to handle this problem, especially when the missing rate is high \cite{Che2018rec}.

The second class directly trains a model based on the missing data without any imputation, and there are a range of research along this line. Che et al. improve Gated Recurrent Unit (GRU) to address the multivariate time series missing data \cite{Che2018rec}. Jiang et al. divide missing data into complete sub-data and then applied them to ensemble classifiers \cite{jiang2005ensemble}. Pelckmans et al. modify the loss function of Support Vector Machine (SVM) to address the uncertainty issue arising from missing data \cite{pelckmans2005svm}. Moreover, there are some research on building improved machine learning models such as logistic regression \cite{david2005lr}, kernel methods \cite{Smola2005KernelMF, Smieja2019}, and autoencoder and multilayer perceptron \cite{smieja2018GMMC} on top of representing missing values with probabilistic density.

To the best of our knowledge, there is no related work on how to adapt GNNs to graphs containing missing features. Hence, we propose an approach to address this problem.

\section{Preliminaries}
\label{sec:preliminary}

In this section, we briefly review GCN, which paves the way for the next discussion.

\subsection{Notations}
Let us consider an undirected graph $\mathcal{G} = ({\mathcal{V}}, \mathcal{E})$, where $\mathcal{V}=\{v_i \mid i=1,\cdots,N\}$ is the node set, and $\mathcal{E} \subseteq \mathcal{V} \times \mathcal{V}$ is the edge set. $\mathbf{A} \in \mathbb{R}^{N \times N}$ denotes the adjacency matrix, where $A_{ij}=A_{ji}$, $A_{ij}=0$ if $(v_{i},v_{j})\notin\mathcal{E}$, and $A_{ij}>0$ if $(v_{i},v_{j})\in\mathcal{E}$. $\mathbf{X} \in \mathbb{R}^{N \times D}$ is the node feature matrix and $D$ is the number of features. $\mathcal{S} \subseteq \{(i,j)|i=1,\dots,N, j=1,\dots,D\}$ is a set for the index of missing features: $\forall (i,j) \in \mathcal{S}$, $X_{ij}$ is not known.

\subsection{Graph Convolutional Network}
GCN-like models consist of aggregators and updaters. The aggregator gathers information guided by the graph structure, and the updater updates nodes’ hidden states according to the gathered information. Specifically, the graph convolutional layer is based on the following equation:
\begin{align} \label{eq:GraphConv}
    \mathbf{H}^{(l+1)} = \sigma (\mathbf{L} \mathbf{H}^{(l)} \mathbf{W}^{(l)})
\end{align}
where $\mathbf{L} \in \mathbb{R}^{N \times N}$ is the aggregation matrix, $\mathbf{H}^{(l)} =  (\bm{h}_1^{(l)}, \dots, \bm{h}_N^{(l)})^{\top} \in \mathbb{R}^{N \times D^{(l)}}$ is the node representation matrix in $l$-th layer, $\mathbf{H}^{(0)}=\mathbf{X}$, $\mathbf{W}^{(l)} \in \mathbb{R}^{D^{(l)} \times D^{(l+1)}}$ is the trainable weight matrix in $l$-th layer, and $\sigma(\cdot)$ is the activation function such as ReLU, LeakyReLU, and ELU.

GCN \cite{kipf2016GCN} adopts the re-normalized graph Laplacian $\mathbf{\hat{A}}$ as the aggregator:
\begin{align}
    \mathbf{L} = \mathbf{\hat{A}} \triangleq \mathbf{\tilde{D}}^{-1/2}\mathbf{\tilde{A}}\mathbf{\tilde{D}}^{-1/2},
\end{align}
where $\mathbf{\tilde{A}} = \mathbf{A} + \mathbf{I}$ and $\mathbf{\tilde{D}} = \mathrm{diag}(\sum_i \tilde{A}_{1i}, \dots, \sum_i \tilde{A}_{Ni})$. Empirically, 2-layer GCN with ReLU activation shows the best performance on node classification, defined as:
\begin{align}
    \mathrm{GCN}(\mathbf{X}, \mathbf{A}) = \mathrm{softmax}(\mathbf{L}(\mathrm{ReLU}(\mathbf{L}\mathbf{X} \mathbf{W}^{(0)}))\mathbf{W}^{(1)})
\end{align}

\section{Proposed Approach}
\label{sec:approach}

In this section, we propose our approach for training GCN on graphs containing missing features. We follow GMMC \cite{smieja2018GMMC} to represent the missing data by GMM and calculate the expected activation of neurons in the first hidden layer. Although this idea is implemented in simple neural networks such as autoencoder and multilayer perceptron, it has not yet been extended to complex neural networks such as RNN, CNN, GNN, and sequence-to-sequence models. The principal difficulty lies in the fact that simply using GMM to represent the missing data will even complicate the network architecture, which hinders us from calculating the expected activation in closed form. In the following, we propose a novel way to unify the representation of missing features and calculation of the expected activation of the first layer neurons in GCN. Specifically, we skillfully represent the missing features by introducing only a small number of parameters in GMM and derive the analytic solution of the expected activation, enabling us to integrate the processing of missing features and graph learning within the same neural network architecture.

\subsection{Representing Node Features Using GMM}
\label{subsec:representation}
Suppose $\boldsymbol{X} \in \mathbb{R}^{D}$ is a random variable for node features. We assume $\boldsymbol{X}$ is generated from the mixture of (degenerate) Gaussians:
\begin{align}
    & \boldsymbol{X} \sim \sum_{k=1}^K \pi_k \mathcal{N}(\boldsymbol{\mu}^{[k]}, \boldsymbol{\Sigma}^{[k]}) \\
    & \boldsymbol{\mu}^{[k]} = (\mu_1^{[k]}, \dots, \mu_D^{[k]})^{\top} \\
    & \boldsymbol{\Sigma}^{[k]} = \mathrm{diag}\left((\sigma_1^{[k]})^2, \dots, (\sigma_D^{[k]})^2\right),
\end{align}
where $K$ is the number of components, $\pi_k$ is the mixing parameter with the constraint that $\sum_k \pi_k = 1$, $\mu_j^{[k]}$ and $(\sigma_j^{[k]})^2$ denote the $j$-th element of mean and variance of the $k$-th Gaussian component, respectively. Further, we introduce a mean matrix $\mathbf{M}^{[k]} \in \mathbb{R}^{N \times D}$ and a variance matrix $\mathbf{S}^{[k]} \in \mathbb{R}^{N \times D}$ for each component as:
\begin{align} \label{eq:meanvar}
    M_{ij}^{[k]}&=\left\{ \begin{array}{ll}
    \mu_j^{[k]} & \text{if } X_{ij} \text{ is missing;} \\
    X_{ij} & \text{otherwise} \\
    \end{array} \right.
    \\
    S_{ij}^{[k]}&=\left\{ \begin{array}{ll}
    (\sigma_j^{[k]})^2 & \text{if } X_{ij} \text{ is missing;} \\
    0 & \text{otherwise} \\
    \end{array} \right.
\end{align}
This enables us to represent each $X_{ij}$ with:
\begin{equation}
X_{ij} \sim \sum_{k=1}^K \pi_k \mathcal{N}(M_{ij}^{[k]}, S_{ij}^{[k]}),
\end{equation}
no matter whether $X_{ij}$ is missing or not. Thus, we skillfully transform the input of our model into fixed $\mathbf{A}$ and unfixed $X_{ij}$ that follows the mixture of Gaussian distributions. The next layer is based on calculation of the expected activation of neurons, which is discussed in the next section.

\subsection{The Expected Activation of Neurons}
\label{subsec:activation}
Let us first identify some symbols that will be used. Suppose $x \sim F_x$ is a random variable and $F_x$ is the probability density function. We define
\begin{align}
    \sigma[x] \triangleq \sigma[F_x] \triangleq \mathbb{E}[\sigma(x)],
\end{align}
which is the expected value of $\sigma$ activation on $x$.

\begin{theorem}
\label{theorem1}
Let $x \sim \mathcal{N}(\mu, \sigma^2)$. Then:
\begin{align} \label{eq:geneNR}
    \mathrm{ReLU} [\mathcal{N}(\mu, \sigma^2)] = \sigma \mathrm{NR} \left( \frac{\mu}{\sigma} \right),
\end{align}
where 
\begin{align}
& \mathrm{NR}(z) = \frac{1}{\sqrt{2\pi}} \exp \left( -\frac{z^2}{2} \right) + \frac{z}{2} \left( 1 + \mathrm{erf} \left( \frac{z}{\sqrt{2}} \right) \right)\\
& \mathrm{erf}(z) = \frac{2}{\sqrt\pi} \int_{0}^{z} \exp(-t^2)dt.
\end{align}
\end{theorem}

\begin{proof}
Please see \cite{smieja2018GMMC} for a proof.
\end{proof}

\begin{lemma}
\label{lemma2}
Let $X_{ij} \sim \sum_{k=1}^K \pi_k \mathcal{N}(M_{ij}^{[k]}, S_{ij}^{[k]})$. Given the aggregation matrix $\mathbf{L}$ and the weight matrix $\mathbf{W}$, then:
\begin{align}
    \mathrm{ReLU}[(\mathbf{LXW})_{ij}] &= \sum_{k=1}^K \pi_k \sqrt{\hat{S}_{ij}^{[k]}} \mathrm{NR} \bigg( \frac{\hat{M}_{ij}^{[k]}}{\sqrt{\hat{S}_{ij}^{[k]}}} \bigg)                                                  \label{neu_act1}  \\
    \mathrm{LeakyReLU}[(\mathbf{LXW})_{ij}] &= \sum_{k=1}^K \pi_k \Bigg( \sqrt{\hat{S}_{ij}^{[k]}} \mathrm{NR} \bigg( \frac{\hat{M}_{ij}^{[k]}}{\sqrt{\hat{S}_{ij}^{[k]}}} \bigg) \nonumber \\
                                            & -\alpha \sqrt{\hat{S}_{ij}^{[k]}} \mathrm{NR} \bigg( -\frac{\hat{M}_{ij}^{[k]}}{\sqrt{\hat{S}_{ij}^{[k]}}} \bigg) \Bigg), 
                                            \label{neu_act2}
\end{align}
where $\odot$ is element-wise multiplication, $\alpha$ is the negative slope parameter of $\mathrm{LeakyReLU}$ activation, and
\begin{align}
    & \mathbf{\hat{M}}^{[k]} = \mathbf{LM}^{[k]}\mathbf{W} \label{eq:hatM}\\
    & \mathbf{\hat{S}}^{[k]} = (\mathbf{L} \odot \mathbf{L})\mathbf{S}^{[k]}(\mathbf{W} \odot \mathbf{W}) \label{eq:hatS}.
\end{align}
\end{lemma}

\begin{proof}
The element of matrix $\mathbf{LXW}$ can be expressed as:
\begin{align} \label{eq:gcn1st}
    (\mathbf{LXW})_{ij} = \sum_{d=1}^D \sum_{n=1}^N L_{in} X_{nd} W_{dj}
\end{align}
Based on the property of Gaussian distribution, $(\mathbf{LXW})_{ij}$ also follows a mixture of Gaussian distributions as:
\begin{align}
    &\sum_{d=1}^D \sum_{n=1}^N L_{in} X_{nd} W_{dj}\\
    &\sim \sum_{k=1}^K \pi_k \mathcal{N} \left(\sum_{d=1}^D \sum_{n=1}^N L_{in} M_{nd}^{[k]}W_{dj}, \sum_{d=1}^D \sum_{n=1}^N L_{in}^2 S_{nd}^{[k]} W_{dj}^2 \right) \\
    &= \sum_{k=1}^K \pi_k \mathcal{N} \left( \{\mathbf{LM}^{[k]}\mathbf{W\}}_{ij}, \{(\mathbf{L} \odot \mathbf{L})\mathbf{S}^{[k]}(\mathbf{W} \odot \mathbf{W})\}_{ij} \right) \\
    &= \sum_{k=1}^K \pi_k \mathcal{N} \left( \hat{M}_{ij}^{[k]}, \hat{S}_{ij}^{[k]} \right).
\end{align}
Finally, using the result of \textsc{Theorem}~\ref{theorem1}, we can derive Eq.~\eqref{neu_act1} as:
\begin{align} \label{eq:GCNmod1st}
    \mathrm{ReLU}[(\mathbf{LXW})_{ij}]
    &= \sum_{k=1}^K \pi_k \mathrm{ReLU} \left[\mathcal{N}(\hat{M}_{ij}^{[k]}, \hat{S}_{ij}^{[k]}) \right] \\
    &= \sum_{k=1}^K \pi_k \sqrt{\hat{S}_{ij}^{[k]}} \mathrm{NR} \bigg( \frac{\hat{M}_{ij}^{[k]}}{\sqrt{\hat{S}_{ij}^{[k]}}} \bigg).
\end{align}
Eq.~\eqref{neu_act2} can be proved similarly and the proof is omitted due to lack of space.
\end{proof} 

Thus, we can calculate the expected activation of neurons for the first layer according to \textsc{Lemma}~\ref{lemma2}. Calculation of the subsequent layers remains unchanged.

\subsection{The Network Architecture}
\label{subsec:architecture}

\begin{algorithm}[t]
\caption{Algorithm of \textsc{GCNmf}}
\label{alg:gcnmf}
\begin{algorithmic}[1]
\Require Aggregation matrix $\mathbf{L}$, node feature matrix $\mathbf{X}$ (with some missing elements), the number of layers $L$, the number of Gaussian components $K$
\Ensure According to the task
\Initialize{
$(\pi_k, \boldsymbol{\mu}^{[k]}, \boldsymbol{\Sigma}^{[k]})$ are optimized by EM algorithm w.r.t $\mathbf{X}$
}
\While{not converged}
    \State $\mathbf{H}^{(1)} \gets \mathrm{ReLU}[\mathbf{LXW}^{(0)}]$ 
    \Comment{\textsc{Lemma}~\ref{lemma2}}
    \For{$l \gets 2, \dots, L - 1$}
        \State $\mathbf{H}^{(l)} \gets \mathrm{ReLU}(\mathbf{LH}^{(l-1)}\mathbf{W}^{(l-1)})$
    \EndFor
    \State $\mathbf{Z} \gets final\_layer(\mathbf{LH}^{(L-1)}\mathbf{W}^{(L-1)})$
    \State $\mathcal{L} \gets loss(\mathbf{Z})$
    \State Minimize $\mathcal{L}$ and update GMM parameters and network parameters with a gradient descent optimization algorithm
\EndWhile
\end{algorithmic}
\end{algorithm}

Our approach is named \model{}. We illustrate the model architecture in Figure \ref{fig:archi} and present the pseudo-code in Algorithm~\ref{alg:gcnmf}, with additional explanations below.

\begin{itemize}
\item \textit{Initialize the hyper-parameters}\\
    The additional hyper-parameters include the number of layers $L$, the number of Gaussian components $K$.
\item \textit{Initialize the model parameters}\\
    The model parameters include GMM parameters $(\pi_k, \boldsymbol{\mu}^{[k]}, \boldsymbol{\Sigma}^{[k]})$ and conventional network parameters. GMM parameters are initialized by EM algorithm \cite{dempster1977} that explores the data density\footnote{The algorithm implementation is provided by scikit-learn: \url{https://scikit-learn.org/}}.
\item \textit{Forward propagation}\\
    Calculate the first layer according to \textsc{Lemma}~\ref{lemma2}, and calculate the other layers as usual.
\item \textit{Backward propagation}\\
    Apply a gradient descent optimization algorithm to jointly learn the GMM parameters and network parameters by minimizing a cost function that is created based on a specific task.
\item \textit{Consistency}\\
    \model{} is consistent with GCN when the features are complete. Suppose $\mathcal{S}=\emptyset$. It follows that $\sigma[(\mathbf{LXW})_{ij}] = \sigma\big((\mathbf{LXW})_{ij}\big)$ (see the proof below). In other words, the computation of the first layer based on expected activations is equivalent to that based on fixed features. Thus, \model{} degenerates to GCN when the features are complete.
\end{itemize}
\noindent \begin{proof}
Take ReLU activation as an example. When $\mathcal{S}=\emptyset$, we have $X_{ij} \sim \sum_{k=1}^K \pi_k \mathcal{N}(X_{ij}^{[k]}, 0)$, $\hat{S}_{ij}^{[k]}=0$, and $\hat{M}_{ij}^{[k]}=(\mathbf{LXW})_{ij}$. Thus,
\begin{align}
    \mathrm{ReLU}[(\mathbf{LXW})_{ij}] &= \sum_{k=1}^K \pi_k \sqrt{\hat{S}_{ij}^{[k]}} \mathrm{NR} \bigg( \frac{\hat{M}_{ij}^{[k]}}{\sqrt{\hat{S}_{ij}^{[k]}}} \bigg)                                                 \label{eq1} \\
                                       &= \sum_{k=1}^K \pi_k \lim_{\epsilon \to 0+} \Bigg( \sqrt{\epsilon} \mathrm{NR} \bigg( \frac{(\mathbf{LXW})_{ij}}{\sqrt{\epsilon}} \bigg) \Bigg)                                                \label{eq2} \\
                                       &= \sum_{k=1}^K \pi_k \lim_{\epsilon \to 0+} \Bigg( \sqrt{\frac{\epsilon}{2\pi}}
                                       \exp \Big( -\frac{(\mathbf{LXW})_{ij}^2}{2\epsilon} \Big)  \nonumber \\
                                       & + \frac{(\mathbf{LXW})_{ij}}{2} \Big( 1 + \frac{2}{\sqrt\pi} 
                                       \int_{0}^{\frac{(\mathbf{LXW})_{ij}}{\sqrt{2\epsilon}}} \exp(-t^2)dt \Big)  \Bigg)                       \label{eq3} \\
                                       &= \left\{ \begin{array}{ll}
                                            0 & \text{if } (\mathbf{LXW})_{ij}\leq 0 \\
                                            (\mathbf{LXW})_{ij} & \text{otherwise}
                                            \end{array}                                                                                 \label{eq4}   \right.\\
                                       &= \mathrm{ReLU}\big((\mathbf{LXW})_{ij}\big),
\end{align}
\noindent where we have used $\int_{0}^{+\infty}\exp(-x^2)dx = \frac{\sqrt{\pi}}{2}$ and $\int_{0}^{-\infty}\exp(-x^2)dx = -\frac{\sqrt{\pi}}{2}$ in Eq.~\eqref{eq4}.
\end{proof}

\subsubsection*{Time Complexity}
In the following, we analyze the time complexity of the forward propagation. Note that \model{} modifies the original GCN in the first layer, where the calculation of Eq.~\eqref{eq:GraphConv} is replaced by Eq.~\eqref{neu_act1} or \eqref{neu_act2}. We assume that $\mathbf{L}$ is a sparse matrix. The calculation of Eq.~\eqref{eq:GraphConv} takes $\mathcal{O}(|\mathcal{E}|D + NDD^{(1)})$ complexity \cite{chiang2019cluster}.

Eq.~\eqref{neu_act1} or \eqref{neu_act2} requires calculation of Eq.~\eqref{eq:hatM} and \eqref{eq:hatS}. The complexity of Eq.~\eqref{eq:hatM} for all $k$ is $\mathcal{O}(K(|\mathcal{E}|D + NDD^{(1)}))$. The complexity of Eq.~\eqref{eq:hatS} for all $k$ is $\mathcal{O(|\mathcal{E}|)}$ + $\mathcal{O}(DD^{(1)})$ + $\mathcal{O}(K(|\mathcal{E}|D + NDD^{(1)}))$, where the first two terms are for $(\mathbf{L} \odot \mathbf{L})$ and $(\mathbf{W} \odot \mathbf{W})$, respectively. Given $\mathbf{\hat{M}}^{[k]}$ and $\mathbf{\hat{S}}^{[k]}$, Eq.~\eqref{neu_act1} or \eqref{neu_act2} takes $\mathcal{O}(KND^{(1)})$ time for all $i,j$.

Putting them all together, the total complexity of the first layer of \model{} is $\mathcal{O}(K(|\mathcal{E}|D + NDD^{(1)}))$ + $\mathcal{O(|\mathcal{E}|)}$ + $\mathcal{O}(DD^{(1)})$ + $\mathcal{O}(K(|\mathcal{E}|D + NDD^{(1)}))$ + $\mathcal{O}(KND^{(1)})$ = $\mathcal{O}(K(|\mathcal{E}|D + NDD^{(1)}))$. Since the number of components $K$ is usually small, the forward propagation of \model{} has the same complexity as GCN.

\section{Experiments}
\label{sec:experiment}
We conducted experiments on the node classification task and link prediction task to answer the following questions:
\begin{itemize}
\item Does {\textsc{GCNmf}} agree with our intuition and perform well? 
\item Where do imputation based methods fail?
\item Is {\textsc{GCNmf}} sensitive to the hyper-parameters?
\item Is {\textsc{GCNmf}} computationally expensive?
\end{itemize}
In the following, we first explain experimental settings in detail, including baselines and datasets. After that, we discuss the results.

\begin{table*}[!t]
    \centering
    \caption{Statistics of datasets.}
    \label{tab:statdata}
    \begin{tabular}{l|cccc} \toprule
         & Cora & Citeseer & AmaPhoto & AmaComp\\ \midrule
        \#Nodes & 2,708 & 3,327 & 7,650 & 13,752\\
        \#Edges & 5,429 & 4,732 & 143,663 & 287,209\\
        \#Features & 1,433 & 3,703 & 745 & 767\\
        \#Classes & 7 & 6 & 8 & 10\\
        \#Train nodes & 140 & 120 & 320 & 400\\
        \#Validation nodes & 500 & 500 & 500 & 500\\
        \#Test nodes & 1,000 & 1,000 & 6,830 & 12,852\\
        Feature sparsity & 98.73\% & 99.15\% & 65.26\% & 65.16\%  \\
        \bottomrule
    \end{tabular}
\end{table*}

\textbf{Datasets.}
We did experiments on four real-world graph datasets that are commonly used. Descriptions of these graphs are as follows and Table \ref{tab:statdata} summarizes their statistics.
\begin{itemize}
\item Cora and Citeseer \cite{sen2008col}: The citation graphs, where nodes are documents and edges are citation links. Node features are bag-of-words representations of documents. Each node is associated with a label representing the topic of documents.
\item AmaPhoto and AmaComp \cite{amazondata2015}: The product co-purchase graphs, where nodes are products and edges exist between products that are co-purchased by users frequently. Node features are bag-of-words representations of product reviews. Node labels represent the category of products.
\end{itemize}

To prepare graphs with missing features, we pre-processed the datasets and removed a portion of node features according to a missing rate parameter $mr$. We consider the following three cases.

\begin{itemize}

\item \textit{Uniform randomly missing features}\\
    $mr=|\mathcal{S}|/(ND)$ (percentage) of the features are randomly selected and removed from the node feature matrix $\mathbf{X}$. $\mathcal{S}$ was randomly selected with uniform probability.
    
\item \textit{Biased randomly missing features}\\
    90\% of certain features and 10\% of the remaining features are randomly selected and removed from $\mathbf{X}$. In this scenario, the features with 90\% values removed represent sensitive information, which is always missing in practice. For ease of implementation, such sensitive features are randomly selected under the condition $mr=|\mathcal{S}|/(ND)$.
    

\item \textit{Structurally missing features}\\
    The respective features of $mr$ (percentage) random nodes are removed from $\mathbf{X}$. Specifically, $\mathcal{V}^{\prime} \subseteq \mathcal{V}$ was randomly selected with uniform probability, such that $mr=|\mathcal{V}^{\prime}|/N$. Then, $\mathcal{S}=\{(i,j)|v_i \in \mathcal{V}^{\prime}, j=1,\dots,D\}$.

\end{itemize}

\textbf{Baselines.}
We consider the following imputation methods to fill in missing values and then apply GCN on the complete graphs.
\begin{itemize}
    \item \textsc{MEAN} \cite{GarciaLaencina2010}: 
        This method replaces missing values with the mean of observed features based on the respective row of the feature matrix $\mathbf{X}$.
    \item \textsc{K-NN} \cite{batista2002study}:
        This approach samples similar features by $k$-nearest neighbors and then replaces missing values with the mean of these features. We set $k = 5$.
    \item \textsc{MFT} \cite{koren2009mf}:
        This is the imputation method based on factorizing the incomplete matrix into two low-rank matrices.
    \item \textsc{SoftImp} \cite{mazumder2010soft}:
        This method iteratively replaces the missing values with those estimated from a soft-thresholded singular value decomposition (SVD).
    \item \textsc{MICE} \cite{buuren2010mice}:
        This is the multiple imputation method that infers missing values from the conditional distributions by Markov chain Monte Carlo (MCMC) techniques.
    \item \textsc{MissForest} \cite{Stekhoven2011missforest}:
        This is a non-parametric imputation method that utilizes Random Forest to predict missing values.
    \item \textsc{VAE} \cite{Kingma2014}:
        This is a VAE based method for reconstructing missing values.
    \item \textsc{GAIN} \cite{yoon2018gain}:
        This is a GAN-based approach for imputing missing data.
    \item \textsc{GINN} \cite{spinelli2019ginn}:
        This is a imputation method based on graph denoising autoencoder.
\end{itemize}

We employed Optuna \cite{optuna_2019} to tune the hyper-parameters such as learning rate, $L_2$ regularization, and dropout rate. We followed the normalized initialization scheme \cite{glorot10understanding} to initialize the weight matrix. We adopted Adam algorithm \cite{KingmaB14Adam} for optimization. For \textsc{GCNmf}, we simply set the number of Gaussian components to 5 across all datasets. The implementation of all approaches is in Python and PyTorch and we ran the experiments on a single machine with Intel Xeon Gold 6148 Processor @2.40GHz, NVIDIA Tesla V100 GPU, and RAM @64GB. For reproducibility, the source code of \textsc{GCNmf} and the graph datasets are publicly available\footnote{\url{https://github.com/marblet/GCNmf}}.

\begin{table*}[!pt]
    \centering
    \caption{The accuracy results for node classification task in Cora.}
    \label{table_cora}
    \scalebox{.84}{
    \begin{tabular}{c|c|ccccccccc} \toprule
Missing type & Missing rate & 10\% & 20\% & 30\% & 40\% & 50\% & 60\% & 70\% & 80\% & 90\%\\ \midrule
\multirow{13}{*}{\shortstack{Uniform\\Randomly\\Missing}} & \textsc{MEAN} & \underline{80.96} & 80.41 & 79.48 & 78.51 & 77.17 & 73.66 & 56.24 & 20.49 & 13.22\\
& \textsc{K-NN} & 80.45 & 80.10 & 78.86 & 77.26 & 75.34 & 71.55 & 66.44 & 40.99 & 15.11\\
& \textsc{MFT} & 80.70 & 80.03 & 78.97 & 78.12 & 76.43 & 71.33 & 45.82 & 27.22 & 23.98\\
& \textsc{SoftImp} & 80.74 & 80.32 & 79.63 & 78.68 & \underline{77.32} & \underline{74.26} & \underline{70.36} & \underline{64.93} & 41.20\\
& \textsc{MICE} & -- & -- & -- & -- & -- & -- & -- & -- & --\\
& \textsc{MissForest} & 80.68 & 80.43 & \underline{79.74} & \underline{79.27} & 76.12 & 73.70 & 68.31 & 60.92 & 45.89\\
& \textsc{VAE} & 80.91 & \underline{80.47} & 79.18 & 78.38 & 76.84 & 72.41 & 50.79 & 18.12 & 13.27\\
& \textsc{GAIN} & 80.43 & 79.72 & 78.35 & 77.01 & 75.31 & 72.50 & 70.34 & 64.85 & \underline{58.87}\\
& \textsc{GINN} & 80.77 & 80.01 & 78.77 & 76.67 & 74.44 & 70.58 & 58.60 & 18.04 & 13.19\\
& \textsc{GCNmf} & \textbf{81.70} & \textbf{81.66} & \textbf{80.41} & \textbf{79.52} & \textbf{77.91} & \textbf{76.67} & \textbf{74.38} & \textbf{70.57} & \textbf{63.49}\\ \cmidrule{2-11}
& \multirow{3}{*}{\shortstack{Performance gain \\ (\%)}} & 0.91 & 1.48 & 0.84 & 0.32 & 0.76 & 3.25 & 5.71 & 8.69 & 7.85\\
&  & \textbar & \textbar & \textbar & \textbar & \textbar & \textbar & \textbar & \textbar & \textbar\\
&  & 1.58 & 2.43 & 2.63 & 3.72 & 4.66 & 8.63 & 62.33 & 291.19 & 381.35\\
\midrule
\multirow{13}{*}{\shortstack{Biased\\Randomly\\Missing}} & \textsc{Mean} & \underline{81.22} & \underline{80.37} & \underline{78.95} & 77.46 & 75.94 & 72.44 & 53.14 & 20.39 & 13.40\\
& \textsc{K-NN} & 80.75 & 79.94 & 78.33 & 77.17 & 75.62 & 72.66 & 67.05 & 54.71 & 15.13\\
& \textsc{MFT} & 80.75 & 75.01 & 56.28 & 55.76 & 43.81 & 29.31 & 25.88 & 21.79 & 21.07\\
& \textsc{SoftImp} & 81.04 & 80.30 & 78.80 & \underline{78.50} & \underline{75.99} & 73.65 & 61.37 & 60.06 & 46.38\\
& \textsc{MICE} & -- & -- & -- & -- & -- & -- & -- & -- & --\\
& \textsc{MissForest} & 80.90 & 80.10 & 78.79 & 77.54 & 74.66 & 71.04 & 65.28 & 56.65 & 44.30\\
& \textsc{VAE} & 80.92 & 80.33 & 78.86 & 77.25 & 75.74 & 69.29 & 53.53 & 18.11 & 13.27\\
& \textsc{GAIN} & 80.68 & 79.62 & 78.54 & 77.41 & 75.84 & \underline{73.82} & \underline{69.18} & \underline{63.99} & \underline{59.41}\\
& \textsc{GINN} & 80.86 & 80.10 & 78.45 & 76.80 & 74.60 & 72.08 & 65.72 & 50.08 & 13.22\\
& \textsc{GCNmf} & \textbf{82.29} & \textbf{81.09} & \textbf{80.00} & \textbf{79.23} & \textbf{77.33} & \textbf{76.19} & \textbf{72.57} & \textbf{68.19} & \textbf{65.73}\\ \cmidrule{2-11}
& \multirow{3}{*}{\shortstack{Performance gain \\ (\%)}} & 1.32 & 0.90 & 1.33 & 0.93 & 1.76 & 3.21 & 4.90 & 6.56 & 10.64\\
&  & \textbar & \textbar & \textbar & \textbar & \textbar & \textbar & \textbar & \textbar & \textbar\\
&  & 2.00 & 8.11 & 42.15 & 42.09 & 76.51 & 159.95 & 180.41 & 276.53 & 397.20\\
\midrule
\multirow{13}{*}{\shortstack{Structurally\\Missing}} & \textsc{MEAN} & \underline{80.92} & \underline{80.40} & \underline{79.05} & \underline{77.73} & \underline{75.22} & 70.18 & 56.30 & 25.56 & 13.86\\
& \textsc{K-NN} & 80.76 & 80.26 & 78.63 & 77.51 & 74.51 & \underline{70.86} & \underline{63.29} & 37.97 & 13.95\\
& \textsc{MFT} & 80.91 & 80.34 & 78.93 & 77.48 & 74.47 & 69.13 & 52.65 & 29.96 & 17.05\\
& \textsc{SoftImp} & 79.71 & 69.47 & 69.31 & 52.53 & 44.71 & 40.07 & 36.68 & 28.51 & \underline{27.90}\\
& \textsc{MICE} & \underline{80.92} & \underline{80.40} & \underline{79.05} & 77.72 & \underline{75.22} & 70.18 & 56.30 & 25.56 & 13.86\\
& \textsc{MissForest} & 80.48 & 79.88 & 78.54 & 76.93 & 73.88 & 68.13 & 54.29 & 30.82 & 14.05\\
& \textsc{VAE} & 80.63 & 79.98 & 78.57 & 77.42 & 74.69 & 69.95 & 60.71 & 36.59 & 17.27\\
& \textsc{GAIN} & 80.53 & 79.78 & 78.36 & 77.09 & 74.25 & 69.90 & 61.33 & \underline{41.09} & 18.43\\
& \textsc{GINN} & 80.85 & 80.27 & 78.88 & 77.35 & 74.76 & 70.58 & 59.45 & 29.15 & 13.92\\
& \textsc{GCNmf} & \textbf{81.65} & \textbf{80.77} & \textbf{80.67} & \textbf{79.24} & \textbf{77.43} & \textbf{75.97} & \textbf{72.69} & \textbf{68.00} & \textbf{55.64}\\ \cmidrule{2-11}
& \multirow{3}{*}{\shortstack{Performance gain \\ (\%)}} & 0.90 & 0.46 & 2.05 & 1.94 & 2.94 & 7.21 & 14.85 & 65.49 & 99.43\\
&  & \textbar & \textbar & \textbar & \textbar & \textbar & \textbar & \textbar & \textbar & \textbar\\
&  & 2.43 & 16.27 & 16.39 & 50.85 & 73.18 & 89.59 & 98.17 & 166.04 & 301.44\\
\midrule
\multicolumn{2}{c|}{RGCN} & 60.29 & 34.12 & 24.80 & 18.62 & 16.04 & 13.88 & 13.89 & 13.70 & 13.60\\
\midrule
\multicolumn{2}{c|}{GCN}& \multicolumn{9}{c}{81.49}\\ 
\multicolumn{2}{c|}{GCN w/o node features}& \multicolumn{9}{c}{63.22}\\
\bottomrule
    \end{tabular}
}
\end{table*}

\begin{table*}[!pt]
    \centering
    \caption{The accuracy results for node classification task in Citeseer.}
    \label{table_citeseer}
    \scalebox{.84}{
    \begin{tabular}{c|c|ccccccccc} \toprule
Missing type & Missing rate & 10\% & 20\% & 30\% & 40\% & 50\% & 60\% & 70\% & 80\% & 90\%\\ \midrule
\multirow{13}{*}{\shortstack{Uniform\\Randomly\\Missing}} & \textsc{MEAN} & 69.88 & 69.62 & 68.97 & 65.12 & 54.62 & 37.39 & 18.29 & 12.28 & 11.88\\
& \textsc{K-NN} & 69.84 & 69.38 & 68.69 & 67.18 & 62.64 & 54.75 & 32.20 & 14.84 & 12.73\\
& \textsc{MFT} & 69.70 & 69.51 & 68.74 & 65.31 & 60.56 & 41.53 & 34.10 & 17.26 & 19.29\\
& \textsc{SoftImp} & 69.63 & 69.34 & \underline{69.23} & \underline{68.47} & \underline{66.35} & \textbf{65.53} & \textbf{60.86} & \underline{52.23} & 31.08\\
& \textsc{MICE} & -- & -- & -- & -- & -- & -- & -- & -- & --\\
& \textsc{MissForest} & -- & -- & -- & -- & -- & -- & -- & -- & --\\
& \textsc{VAE} & 69.80 & 69.39 & 68.54 & 64.13 & 50.91 & 29.62 & 18.45 & 12.49 & 11.00\\
& \textsc{GAIN} & 69.64 & 68.88 & 67.56 & 65.97 & 63.86 & 60.74 & 55.77 & 52.05 & \underline{42.73}\\
& \textsc{GINN} & \underline{70.07} & \underline{69.79} & 68.87 & 68.14 & 63.21 & 43.61 & 20.74 & 13.26 & 11.31\\
& \textsc{GCNmf} & \textbf{70.93} & \textbf{70.82} & \textbf{69.84} & \textbf{68.83} & \textbf{67.03} & \underline{64.78} & \underline{60.70} & \textbf{55.38} & \textbf{47.78}\\ \cmidrule{2-11}
& \multirow{3}{*}{\shortstack{Performance gain \\ (\%)}} & 1.23 & 1.48 & 0.88 & 0.53 & 1.02 & -1.14 & -0.26 & 6.03 & 11.82\\
&  & \textbar & \textbar & \textbar & \textbar & \textbar & \textbar & \textbar & \textbar & \textbar\\
&  & 1.87 & 2.82 & 3.37 & 7.33 & 31.66 & 118.70 & 231.88 & 350.98 & 334.36\\
\midrule
\multirow{13}{*}{\shortstack{Biased\\Randomly\\Missing}} & \textsc{Mean} & 69.98 & 68.95 & 67.91 & 65.87 & 60.33 & 40.68 & 25.45 & 14.01 & 13.32\\
& \textsc{K-NN} & 70.04 & 68.87 & 68.88 & 67.38 & 64.47 & \underline{62.45} & 52.66 & 32.60 & 12.64\\
& \textsc{MFT} & 69.88 & 67.68 & 63.17 & 45.49 & 25.99 & 20.22 & 20.82 & 18.53 & 18.30\\
& \textsc{SoftImp} & 69.83 & 67.36 & 68.36 & 67.49 & 64.26 & 62.38 & \underline{58.45} & \textbf{55.63} & 32.95\\
& \textsc{MICE} & -- & -- & -- & -- & -- & -- & -- & -- & --\\
& \textsc{MissForest} & -- & -- & -- & -- & -- & -- & -- & -- & --\\
& \textsc{VAE} & \underline{70.05} & 69.13 & 68.21 & 63.44 & 55.71 & 38.55 & 21.98 & 13.34 & 11.17\\
& \textsc{GAIN} & 69.81 & 68.76 & 68.38 & 66.83 & 64.05 & 62.15 & 58.31 & 52.14 & \underline{42.18}\\
& \textsc{GINN} & 69.96 & \underline{69.60} & \underline{69.63} & \underline{68.67} & \underline{64.93} & 62.14 & 55.01 & 31.37 & 12.91\\
& \textsc{GCNmf} & \textbf{71.01} & \textbf{69.99} & \textbf{69.96} & \textbf{68.89} & \textbf{66.30} & \textbf{64.67} & \textbf{61.06} & \underline{54.70} & \textbf{46.14}\\ \cmidrule{2-11}
& \multirow{3}{*}{\shortstack{Performance gain \\ (\%)}} & 1.37 & 0.56 & 0.47 & 0.32 & 2.11 & 3.55 & 4.47 & -1.67 & 9.39\\
&  & \textbar & \textbar & \textbar & \textbar & \textbar & \textbar & \textbar & \textbar & \textbar\\
&  & 1.72 & 3.90 & 10.75 & 51.44 & 155.10 & 219.83 & 193.28 & 310.04 & 313.07\\
\midrule
\multirow{13}{*}{\shortstack{Structurally\\Missing}} & \textsc{MEAN} & 69.55 & \underline{68.31} & \textbf{67.30} & \underline{65.18} & 53.64 & 34.07 & 18.56 & 13.19 & 11.30\\
& \textsc{K-NN} & 69.67 & 67.33 & 66.09 & 63.29 & 56.86 & 31.27 & 19.51 & 13.75 & 11.21\\
& \textsc{MFT} & \underline{69.84} & 68.21 & \underline{66.67} & 63.02 & 51.08 & 34.29 & 16.81 & 14.34 & 15.75\\
& \textsc{SoftImp} & 44.06 & 27.92 & 25.83 & 25.13 & 25.59 & 23.99 & 25.41 & 22.83 & \underline{20.13}\\
& \textsc{MICE} & -- & -- & -- & -- & -- & -- & -- & -- & --\\
& \textsc{MissForest} & -- & -- & -- & -- & -- & -- & -- & -- & --\\
& \textsc{VAE} & 69.63 & 68.07 & 66.34 & 64.33 & \underline{60.46} & \underline{54.37} & 40.71 & 23.14 & 17.20\\
& \textsc{GAIN} & 69.47 & 67.86 & 65.88 & 63.96 & 59.96 & 54.24 & \underline{41.21} & \underline{25.31} & 17.89\\
& \textsc{GINN} & 69.64 & 67.88 & 66.24 & 63.71 & 55.76 & 40.20 & 18.63 & 13.23 & 12.32\\
& \textsc{GCNmf} & \textbf{70.44} & \textbf{68.56} & 66.57 & \textbf{65.39} & \textbf{63.44} & \textbf{60.04} & \textbf{56.88} & \textbf{51.37} & \textbf{39.86}\\ \cmidrule{2-11}
& \multirow{3}{*}{\shortstack{Performance gain \\ (\%)}} & 0.86 & 0.37 & -1.08 & 0.32 & 4.93 & 10.43 & 38.02 & 102.96 & 98.01\\
&  & \textbar & \textbar & \textbar & \textbar & \textbar & \textbar & \textbar & \textbar & \textbar\\
&  & 59.87 & 145.56 & 157.72 & 160.21 & 147.91 & 150.27 & 238.37 & 289.46 & 255.58\\
\midrule
\multicolumn{2}{c|}{RGCN} & 34.37 & 20.69 & 14.16 & 12.15 & 12.01 & 12.34 & 14.36 & 11.97 & 12.57\\
\midrule
\multicolumn{2}{c|}{GCN}& \multicolumn{9}{c}{70.65}\\ 
\multicolumn{2}{c|}{GCN w/o node features}& \multicolumn{9}{c}{40.55}\\
\bottomrule
    \end{tabular}
    }
\end{table*}

\begin{table*}[!pt]
    \centering
    \caption{The accuracy results for node classification task in AmaPhoto.}
    \label{table_amaphoto}
    \scalebox{.85}{
    \begin{tabular}{c|c|ccccccccc} \toprule
Missing type & Missing rate & 10\% & 20\% & 30\% & 40\% & 50\% & 60\% & 70\% & 80\% & 90\%\\ \midrule
\multirow{13}{*}{\shortstack{Uniform\\Randomly\\Missing}} & \textsc{MEAN} & 92.15 & 92.05 & 91.81 & 91.62 & 91.40 & 90.76 & 88.98 & 86.41 & 68.88\\
& \textsc{K-NN} & \underline{92.27} & \underline{92.12} & 91.94 & 91.67 & 91.37 & 90.92 & 90.03 & 87.41 & 81.91\\
& \textsc{MFT} & 92.23 & 92.07 & 91.88 & 91.51 & 91.15 & 90.11 & 88.28 & 85.17 & 75.73\\
& \textsc{SoftImp} & 92.23 & 92.09 & 91.92 & \underline{91.78} & \underline{91.55} & 91.18 & 90.55 & 88.93 & 85.22\\
& \textsc{MICE} & 92.23 & 92.07 & \underline{91.97} & 91.75 & 91.52 & 91.22 & 90.42 & 86.43 & 82.88\\
& \textsc{MissForest} & 92.18 & 92.09 & 91.82 & 91.61 & 91.42 & 90.71 & 89.17 & 86.03 & 82.82\\
& \textsc{VAE} & 92.20 & 92.08 & 91.90 & 91.59 & 91.15 & 90.55 & 89.28 & 86.95 & 81.43\\
& \textsc{GAIN} & 92.23 & 92.11 & 91.90 & 91.73 & 91.49 & \underline{91.24} & \underline{90.72} & \underline{89.49} & \underline{86.96}\\
& \textsc{GINN} & 92.25 & 92.03 & 91.87 & 91.53 & 91.14 & 90.56 & 88.59 & 85.02 & 79.80\\
& \textsc{GCNmf} & \textbf{92.54} & \textbf{92.44} & \textbf{92.20} & \textbf{92.09} & \textbf{92.09} & \textbf{91.69} & \textbf{91.25} & \textbf{90.57} & \textbf{88.96}\\ \cmidrule{2-11}
& \multirow{3}{*}{\shortstack{Performance gain \\ (\%)}} & 0.29 & 0.35 & 0.25 & 0.34 & 0.59 & 0.49 & 0.58 & 1.21 & 2.30\\
&  & \textbar & \textbar & \textbar & \textbar & \textbar & \textbar & \textbar & \textbar & \textbar\\
&  & 0.42 & 0.45 & 0.42 & 0.63 & 1.04 & 1.75 & 3.36 & 6.53 & 29.15\\
\midrule
\multirow{13}{*}{\shortstack{Biased\\Randomly\\Missing}} & \textsc{Mean} & 92.19 & 91.89 & 91.80 & 91.58 & 91.24 & 90.74 & 89.69 & 87.23 & 76.91\\
& \textsc{K-NN} & \underline{92.24} & 92.09 & 91.99 & \underline{91.85} & 91.58 & 91.32 & 90.68 & 89.39 & 81.88\\
& \textsc{MFT} & 92.17 & 92.03 & 91.98 & 91.71 & 91.40 & 90.99 & 89.89 & 87.46 & 75.14\\
& \textsc{SoftImp} & 92.21 & \underline{92.10} & 92.02 & \underline{91.85} & \underline{91.61} & 91.27 & 90.52 & 88.87 & 84.84\\
& \textsc{MICE} & 92.16 & 92.06 & 92.00 & 91.76 & 91.58 & 91.24 & 90.54 & 88.64 & 82.45\\
& \textsc{MissForest} & 92.16 & 92.09 & \underline{92.07} & 91.81 & 91.35 & 90.67 & 89.77 & 86.85 & 82.72\\
& \textsc{VAE} & 92.14 & 92.04 & 91.95 & 91.70 & 91.41 & 91.02 & 90.00 & 88.92 & 83.08\\
& \textsc{GAIN} & 92.22 & 92.02 & 91.87 & 91.76 & 91.58 & \underline{91.43} & \underline{90.88} & \underline{89.99} & \underline{87.11}\\
& \textsc{GINN} & \underline{92.24} & 92.04 & 91.95 & 91.78 & 91.48 & 91.16 & 90.40 & 88.35 & 79.18\\
& \textsc{GCNmf} & \textbf{92.72} & \textbf{92.69} & \textbf{92.55} & \textbf{92.61} & \textbf{92.43} & \textbf{92.33} & \textbf{91.91} & \textbf{91.58} & \textbf{89.35}\\ \cmidrule{2-11}
& \multirow{3}{*}{\shortstack{Performance gain \\ (\%)}} & 0.52 & 0.64 & 0.52 & 0.83 & 0.90 & 0.98 & 1.13 & 1.77 & 2.57\\
&  & \textbar & \textbar & \textbar & \textbar & \textbar & \textbar & \textbar & \textbar & \textbar\\
&  & 0.63 & 0.87 & 0.82 & 1.12 & 1.30 & 1.83 & 2.48 & 5.45 & 18.91\\
\midrule
\multirow{13}{*}{\shortstack{Structurally\\Missing}} & \textsc{MEAN} & 92.06 & 91.80 & \underline{91.59} & 91.20 & 90.59 & 89.83 & 87.66 & 84.60 & 77.41\\
& \textsc{K-NN} & 92.04 & 91.71 & 91.43 & 91.08 & 90.37 & 89.88 & 88.80 & \underline{85.77} & \underline{80.48}\\
& \textsc{MFT} & 92.08 & 91.83 & \underline{91.59} & 91.18 & 90.56 & 89.80 & 87.58 & 84.36 & 77.69\\
& \textsc{SoftImp} & 91.75 & 91.19 & 90.55 & 89.33 & 88.00 & 87.19 & 84.87 & 81.96 & 76.72\\
& \textsc{MICE} & 92.05 & \underline{91.87} & \underline{91.59} & \underline{91.24} & 90.60 & 89.86 & 87.82 & 84.57 & 77.32\\
& \textsc{MissForest} & 92.04 & 91.70 & 91.42 & 91.15 & 90.49 & \underline{90.07} & \underline{88.81} & 85.51 & 75.35\\
& \textsc{VAE} & \underline{92.11} & 91.84 & 91.50 & 91.08 & 90.46 & 89.29 & 87.47 & 83.45 & 67.85\\
& \textsc{GAIN} & 92.04 & 91.78 & 91.49 & 91.14 & \underline{90.63} & 89.94 & 88.60 & 85.41 & 76.48\\
& \textsc{GINN} & 92.09 & 91.83 & 91.53 & 91.16 & 90.43 & 89.61 & 87.77 & 84.53 & 77.14\\
& \textsc{GCNmf} & \textbf{92.45} & \textbf{92.32} & \textbf{92.08} & \textbf{91.88} & \textbf{91.52} & \textbf{90.89} & \textbf{90.39} & \textbf{89.64} & \textbf{86.09}\\ \cmidrule{2-11}
& \multirow{3}{*}{\shortstack{Performance gain \\ (\%)}} & 0.37 & 0.49 & 0.53 & 0.70 & 0.98 & 0.91 & 1.78 & 4.51 & 6.97\\
&  & \textbar & \textbar & \textbar & \textbar & \textbar & \textbar & \textbar & \textbar & \textbar\\
&  & 0.76 & 1.24 & 1.69 & 2.85 & 4.00 & 4.24 & 6.50 & 9.37 & 26.88\\
\midrule
\multicolumn{2}{c|}{RGCN} & 91.50 & 90.81 & 88.37 & 85.52 & 75.17 & 84.89 & 87.67 & 89.95 & 90.56\\
\midrule
\multicolumn{2}{c|}{GCN}& \multicolumn{9}{c}{92.35}\\
\multicolumn{2}{c|}{GCN w/o node features}& \multicolumn{9}{c}{88.77}\\
\bottomrule
    \end{tabular}
    }
\end{table*}

\begin{table*}[!pt]
    \centering
    \caption{The accuracy results for node classification task in AmaComp.}
    \label{table_amacomp}
    \scalebox{.85}{
    \begin{tabular}{c|c|ccccccccc} \toprule
Missing type & Missing rate & 10\% & 20\% & 30\% & 40\% & 50\% & 60\% & 70\% & 80\% & 90\%\\ \midrule
\multirow{13}{*}{\shortstack{Uniform\\Randomly\\Missing}} & \textsc{MEAN} & 82.79 & 82.36 & 81.51 & 80.53 & 79.30 & 77.22 & 74.56 & 61.60 & 5.92\\
& \textsc{K-NN} & 82.89 & 82.73 & 82.18 & 82.00 & 81.54 & 80.58 & 79.34 & 76.81 & 66.04\\
& \textsc{MFT} & 82.82 & 82.54 & 82.05 & 81.58 & 80.76 & 79.28 & 77.11 & 72.31 & 49.42\\
& \textsc{SoftImp} & \underline{82.99} & 82.75 & 82.37 & 82.06 & 81.48 & 80.48 & 79.27 & 77.29 & 69.04\\
& \textsc{MICE} & 82.83 & 82.76 & 82.43 & \underline{82.28} & \underline{81.66} & 80.59 & 78.63 & 75.00 & 63.60\\
& \textsc{MissForest} & -- & -- & -- & -- & 80.89 & 79.57 & 78.22 & 76.00 & 71.98\\
& \textsc{VAE} & 82.65 & 82.47 & 81.72 & 81.15 & 80.47 & 79.99 & 78.55 & 75.80 & 67.26\\
& \textsc{GAIN} & 82.94 & \underline{82.78} & \underline{82.44} & 81.96 & 81.56 & \underline{80.71} & \underline{79.96} & \underline{78.38} & \underline{76.15}\\
& \textsc{GINN} & 82.94 & \underline{82.78} & 82.27 & 81.65 & 80.89 & 78.53 & 76.46 & 73.24 & 58.34\\
& \textsc{GCNmf} & \textbf{86.32} & \textbf{86.07} & \textbf{85.98} & \textbf{85.77} & \textbf{85.46} & \textbf{84.94} & \textbf{84.03} & \textbf{82.38} & \textbf{77.52}\\ \cmidrule{2-11}
& \multirow{3}{*}{\shortstack{Performance gain \\ (\%)}} & 4.01 & 3.97 & 4.29 & 4.24 & 4.65 & 5.24 & 5.09 & 5.10 & 1.80\\
&  & \textbar & \textbar & \textbar & \textbar & \textbar & \textbar & \textbar & \textbar & \textbar\\
&  & 4.44 & 4.50 & 5.48 & 6.51 & 7.77 & 10.00 & 12.70 & 33.73 & 1209.46\\
\midrule
\multirow{13}{*}{\shortstack{Biased\\Randomly\\Missing}} & \textsc{Mean} & 83.03 & \underline{83.07} & 82.49 & 81.82 & 81.17 & 79.76 & 78.16 & 73.79 & 8.68\\
& \textsc{K-NN} & 83.01 & 82.79 & 82.43 & 82.14 & 81.57 & \underline{81.40} & 80.24 & 77.86 & 66.45\\
& \textsc{MFT} & 82.98 & 82.86 & 82.39 & 81.93 & 81.30 & 80.18 & 78.66 & 74.96 & 50.53\\
& \textsc{SoftImp} & 83.07 & 82.88 & 82.13 & 81.87 & 81.23 & 80.53 & 78.98 & 76.74 & 73.91\\
& \textsc{MICE} & 83.07 & 82.77 & 82.44 & 81.94 & 81.56 & 80.84 & 79.40 & 76.71 & 64.11\\
& \textsc{MissForest} & -- & -- & 81.88 & -- & 80.52 & 79.62 & 78.27 & 76.66 & 71.74\\
& \textsc{VAE} & 82.93 & 82.66 & 82.27 & 81.57 & 81.04 & 80.28 & 78.50 & 76.43 & 72.58\\
& \textsc{GAIN} & 83.04 & 82.90 & \underline{82.70} & \underline{82.15} & \underline{81.69} & 81.35 & \underline{80.45} & \underline{78.88} & \underline{76.47}\\
& \textsc{GINN} & \underline{83.10} & 82.71 & 82.58 & 81.94 & 81.63 & 80.81 & 79.29 & 76.53 & 58.18\\
& \textsc{GCNmf} & \textbf{86.41} & \textbf{86.35} & \textbf{86.27} & \textbf{86.16} & \textbf{85.83} & \textbf{85.37} & \textbf{84.84} & \textbf{83.00} & \textbf{79.58}\\ \cmidrule{2-11}
& \multirow{3}{*}{\shortstack{Performance gain \\ (\%)}} & 3.98 & 3.95 & 4.32 & 4.88 & 5.07 & 4.88 & 5.46 & 5.22 & 4.07\\
&  & \textbar & \textbar & \textbar & \textbar & \textbar & \textbar & \textbar & \textbar & \textbar\\
&  & 4.20 & 4.46 & 5.36 & 5.63 & 6.59 & 7.22 & 8.55 & 12.48 & 816.82\\
\midrule
\multirow{13}{*}{\shortstack{Structurally\\Missing}} & \textsc{MEAN} & 82.53 & 82.09 & 81.35 & 80.62 & 79.59 & 77.75 & 75.06 & 69.67 & 23.42\\
& \textsc{K-NN} & 82.59 & 82.15 & 81.57 & 81.07 & 80.25 & \underline{78.86} & \underline{76.91} & \underline{72.89} & 42.23\\
& \textsc{MFT} & 82.48 & 81.91 & 81.43 & 80.58 & 79.40 & 77.64 & 75.19 & 69.97 & 27.33\\
& \textsc{SoftImp} & 82.64 & 81.97 & 81.32 & 80.83 & 79.68 & 77.66 & 75.92 & 56.62 & 52.75\\
& \textsc{MICE} & 82.71 & 82.13 & 81.51 & 80.62 & 79.36 & 77.35 & 74.57 & 67.59 & 45.07\\
& \textsc{MissForest} & 82.65 & 82.20 & 81.84 & 81.04 & 79.18 & 78.66 & 75.98 & 71.91 & 12.05\\
& \textsc{VAE} & \underline{82.76} & 82.40 & 81.72 & 80.88 & 79.23 & 77.62 & 73.76 & 66.33 & 41.37\\
& \textsc{GAIN} & \underline{82.76} & \underline{82.53} & \underline{82.11} & \underline{81.68} & \underline{80.76} & 78.65 & 74.38 & 67.38 & \underline{54.24}\\
& \textsc{GINN} & 82.55 & 82.10 & 81.46 & 80.75 & 79.59 & 77.67 & 75.08 & 70.40 & 26.10\\
& \textsc{GCNmf} & \textbf{86.37} & \textbf{86.22} & \textbf{85.80} & \textbf{85.43} & \textbf{85.24} & \textbf{84.73} & \textbf{84.06} & \textbf{80.63} & \textbf{73.42}\\ \cmidrule{2-11}
& \multirow{3}{*}{\shortstack{Performance gain \\ (\%)}} & 4.36 & 4.47 & 4.49 & 4.59 & 5.55 & 7.44 & 9.30 & 10.62 & 35.36\\
&  & \textbar & \textbar & \textbar & \textbar & \textbar & \textbar & \textbar & \textbar & \textbar\\
&  & 4.72 & 5.26 & 5.51 & 6.02 & 7.65 & 9.54 & 13.96 & 42.41 & 509.29\\
\midrule
\multicolumn{2}{c|}{RGCN} & 79.18 & 76.39 & 74.01 & 63.19 & 14.24 & 63.24 & 72.44 & 75.33 & 77.18\\
\midrule
\multicolumn{2}{c|}{GCN}& \multicolumn{9}{c}{82.94}\\
\multicolumn{2}{c|}{GCN w/o node features}& \multicolumn{9}{c}{81.60}\\
\bottomrule
    \end{tabular}
    }
\end{table*}

\begin{figure*}[!t]
    \centering
    \begin{subfigure}{0.31\textwidth}
        \centering
      \includegraphics[width=\textwidth]{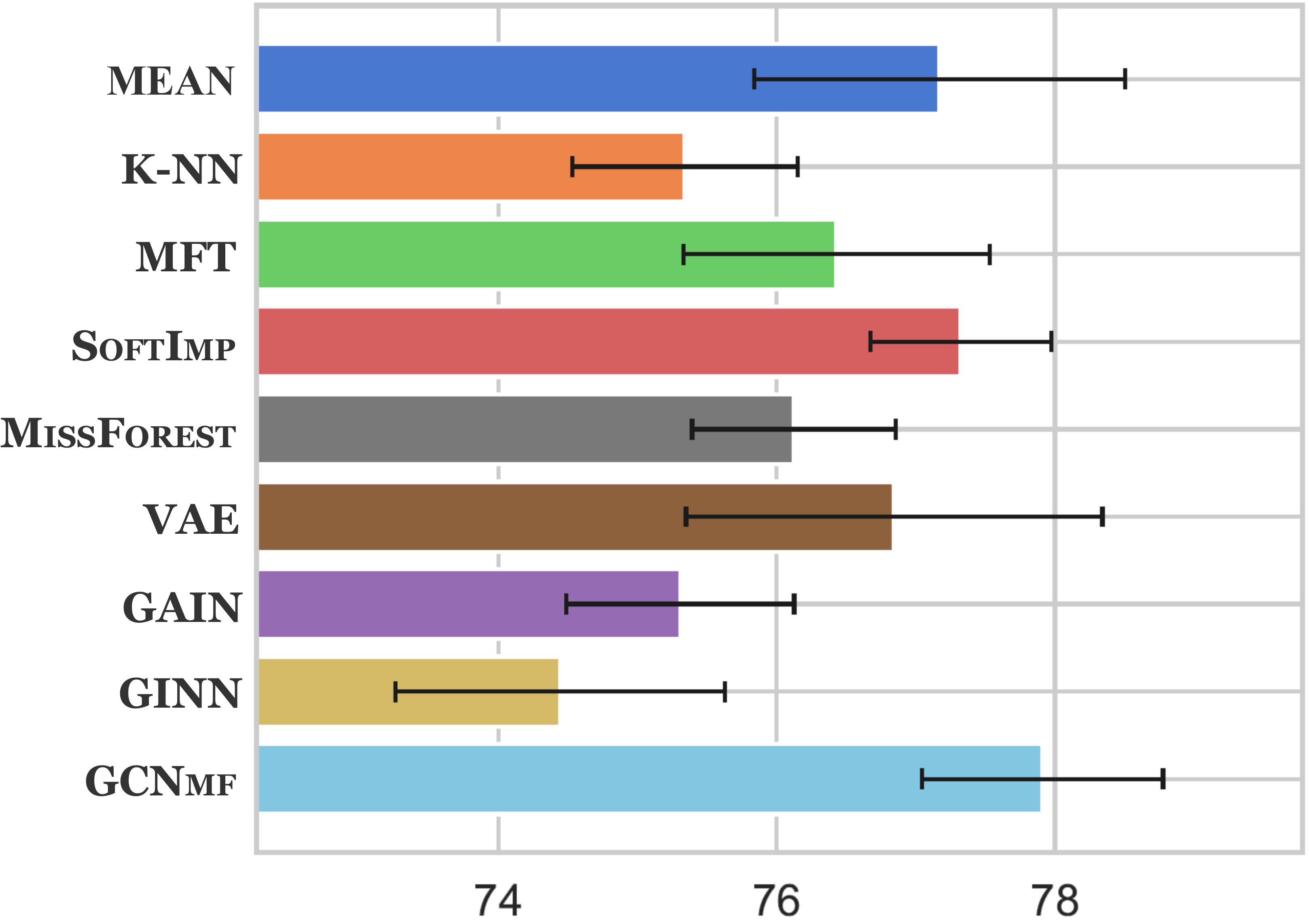}
      \caption{Cora (uniform randomly)}
      \label{subfig:}
    \end{subfigure}%
    \quad
    \begin{subfigure}{0.31\textwidth}
        \centering
      \includegraphics[width=\textwidth]{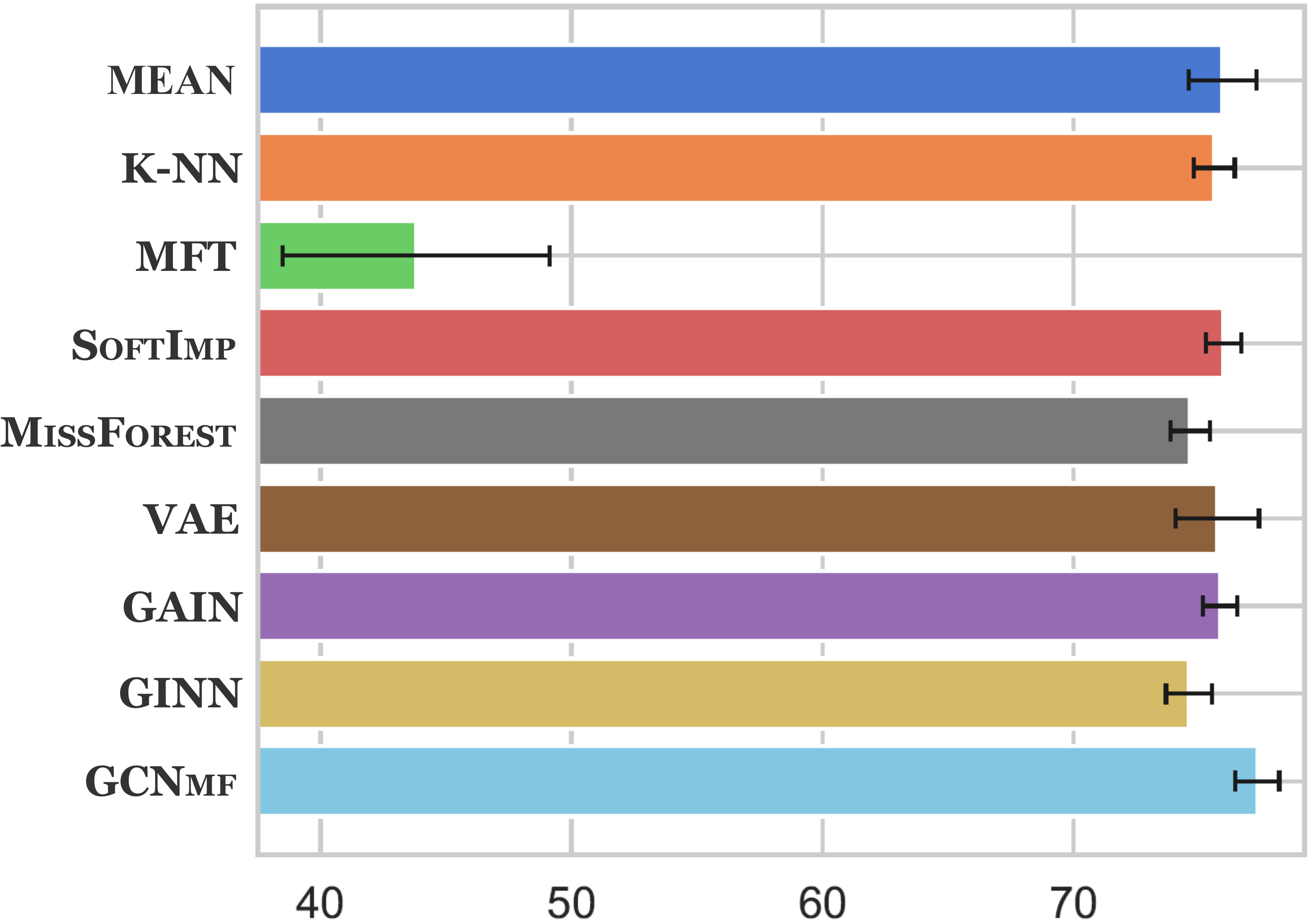}
      \caption{Cora (biased randomly)}
      \label{subfig:}
    \end{subfigure}%
    \quad
    \begin{subfigure}{0.31\textwidth}
        \centering
      \includegraphics[width=\textwidth]{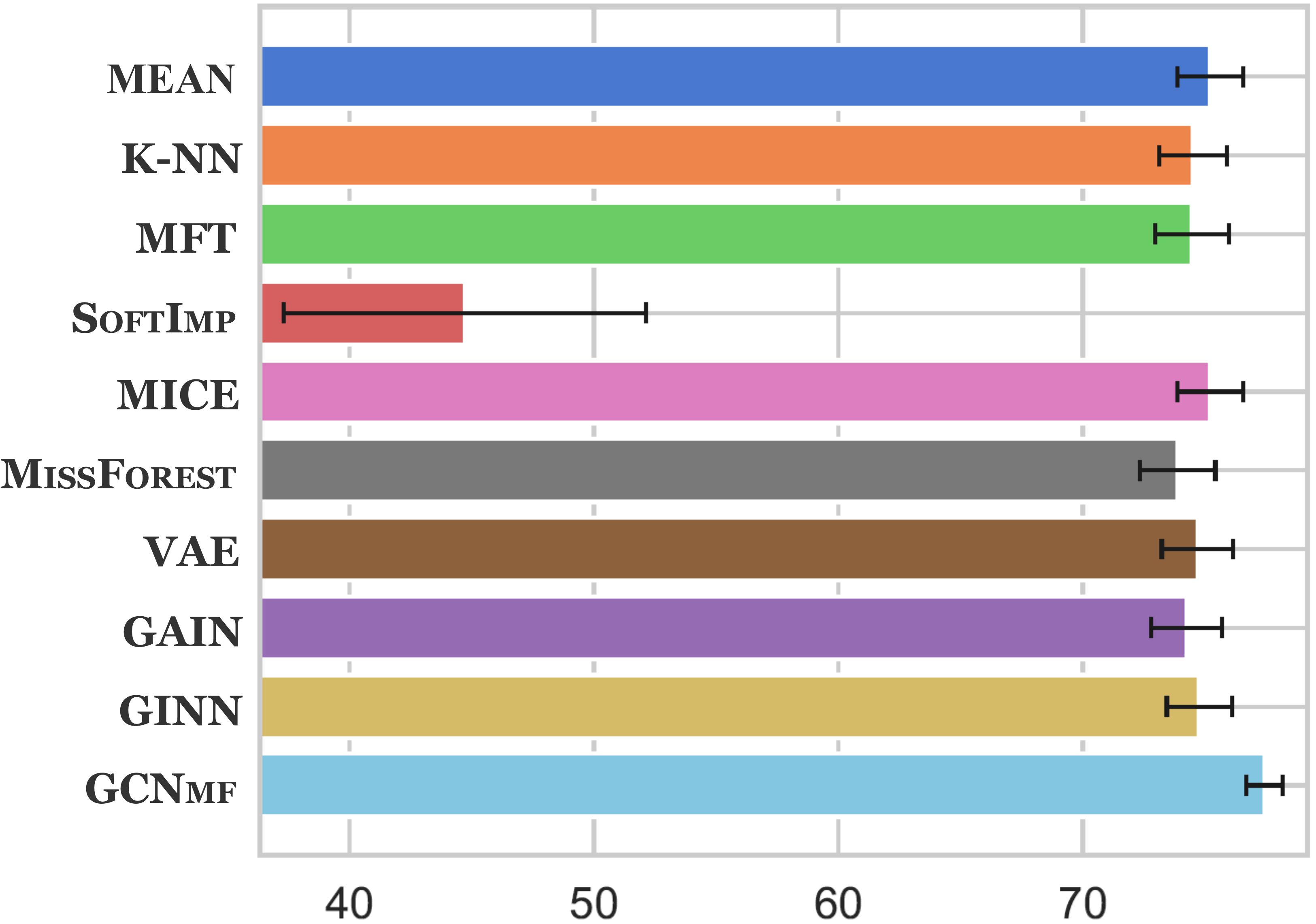}
      \caption{Cora (structurally)}
      \label{subfig:}
    \end{subfigure}%
    \\
    \vspace*{0.60em}
    \begin{subfigure}{0.31\textwidth}
        \centering
      \includegraphics[width=\textwidth]{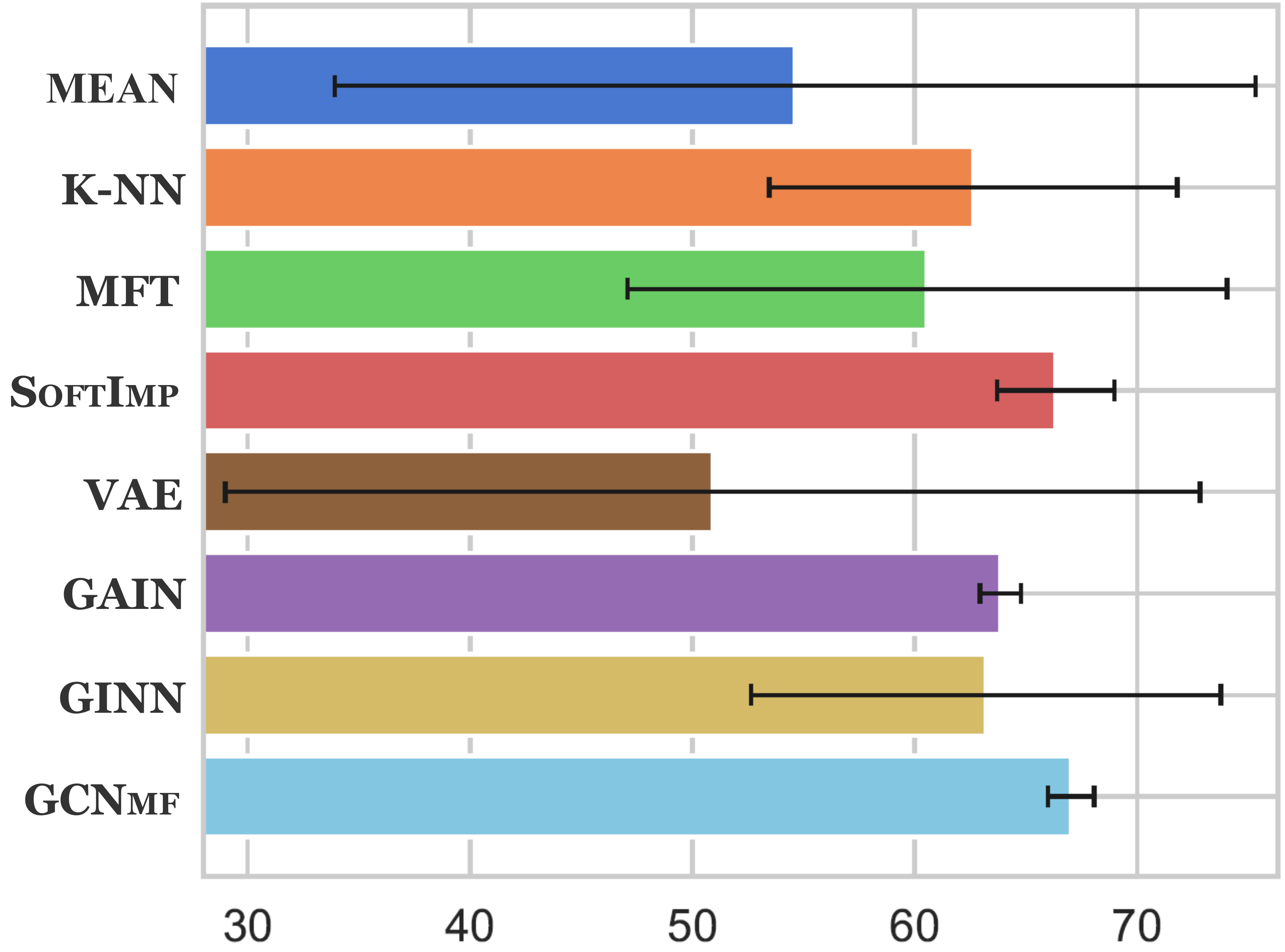}
      \caption{Citeseer (uniform randomly)}
      \label{subfig:}
    \end{subfigure}%
    \quad
    \begin{subfigure}{0.31\textwidth}
        \centering
      \includegraphics[width=\textwidth]{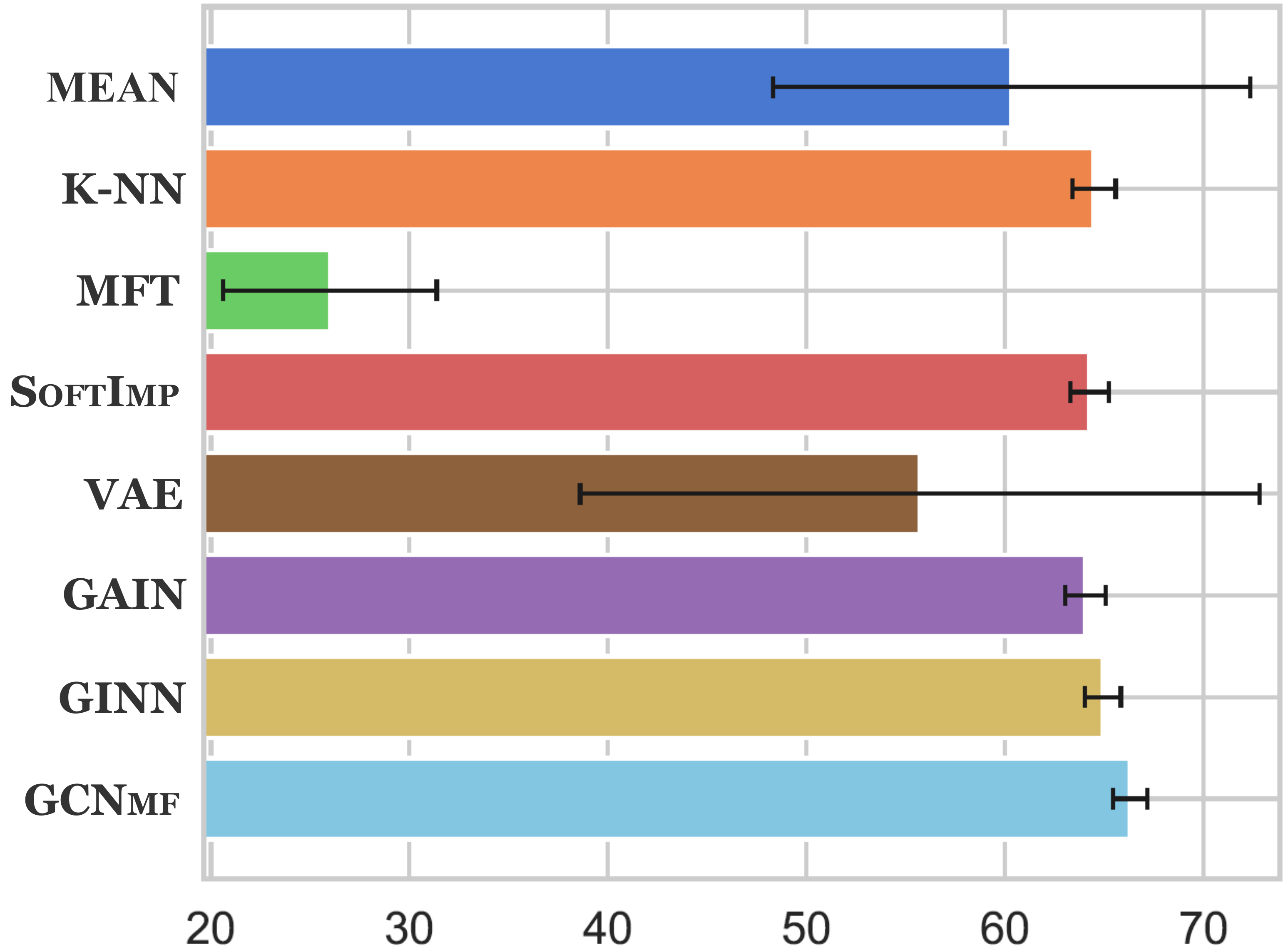}
      \caption{Citeseer (biased randomly)}
      \label{subfig:}
    \end{subfigure}%
    \quad
    \begin{subfigure}{0.31\textwidth}
        \centering
      \includegraphics[width=\textwidth]{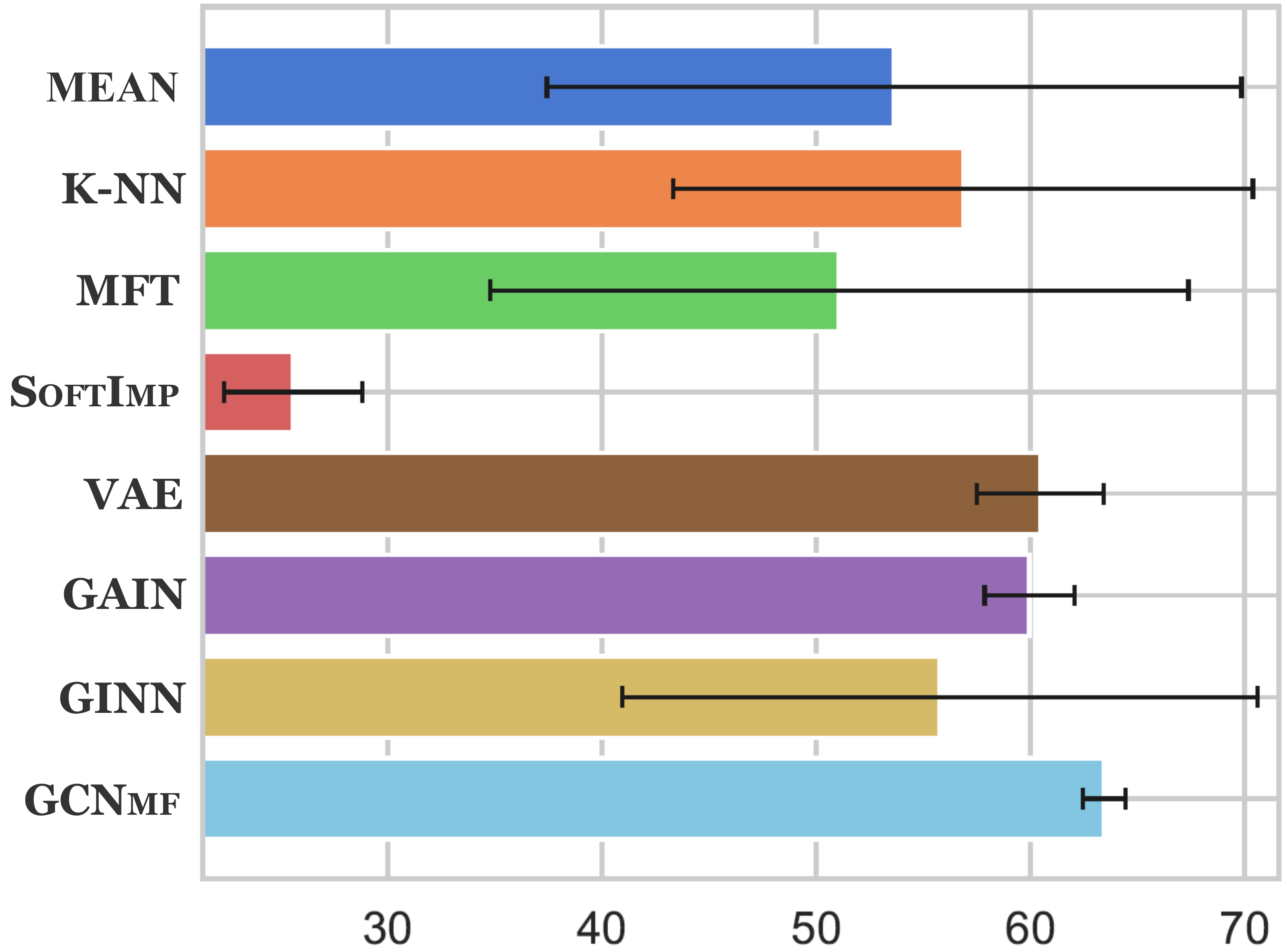}
      \caption{Citeseer (structurally)}
      \label{subfig:}
    \end{subfigure}%
    \\
    \vspace*{0.60em}
    \begin{subfigure}{0.31\textwidth}
        \centering
      \includegraphics[width=\textwidth]{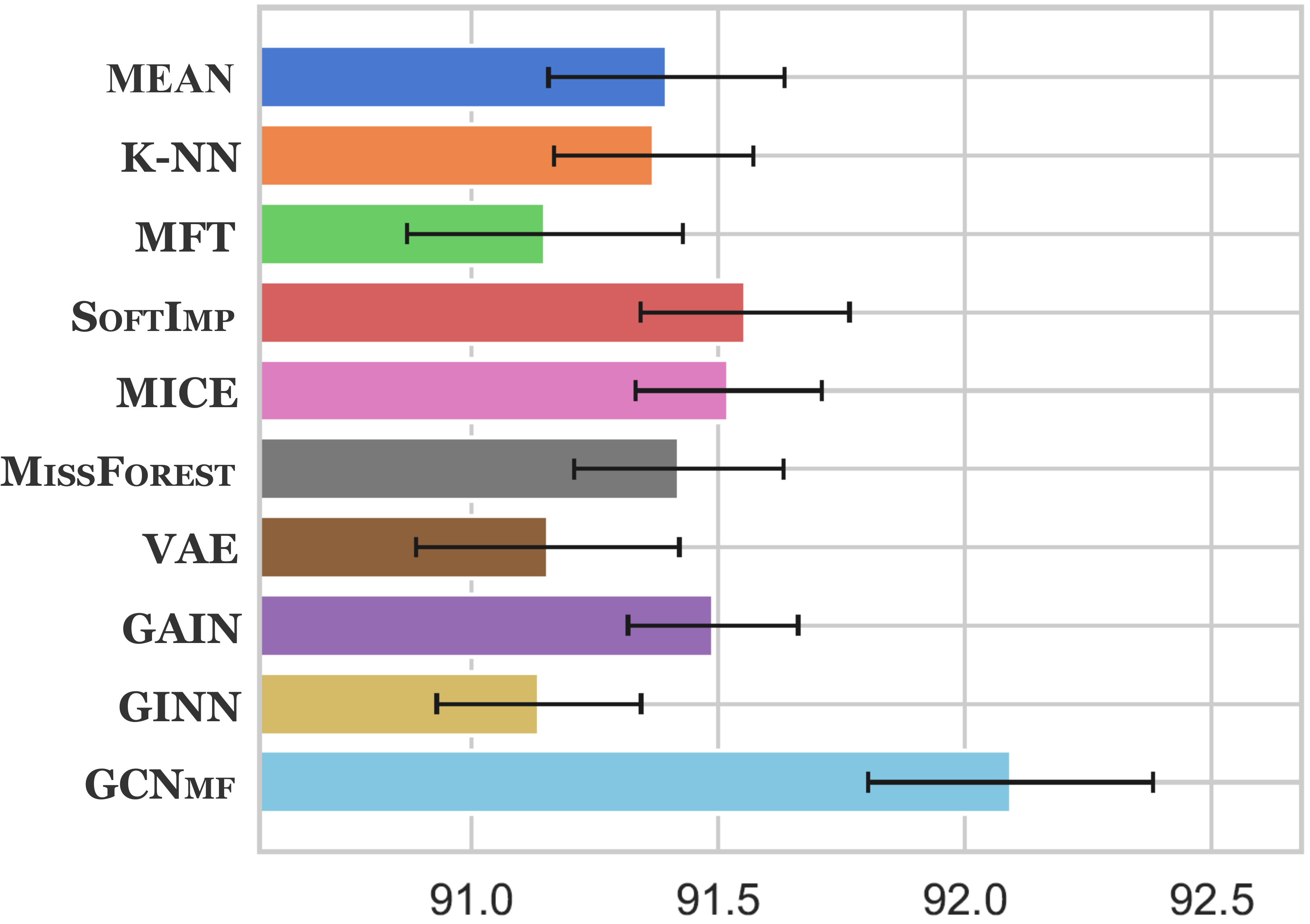}
      \caption{AmaPhoto (uniform randomly)}
      \label{subfig:}
    \end{subfigure}%
    \quad
    \begin{subfigure}{0.31\textwidth}
        \centering
      \includegraphics[width=\textwidth]{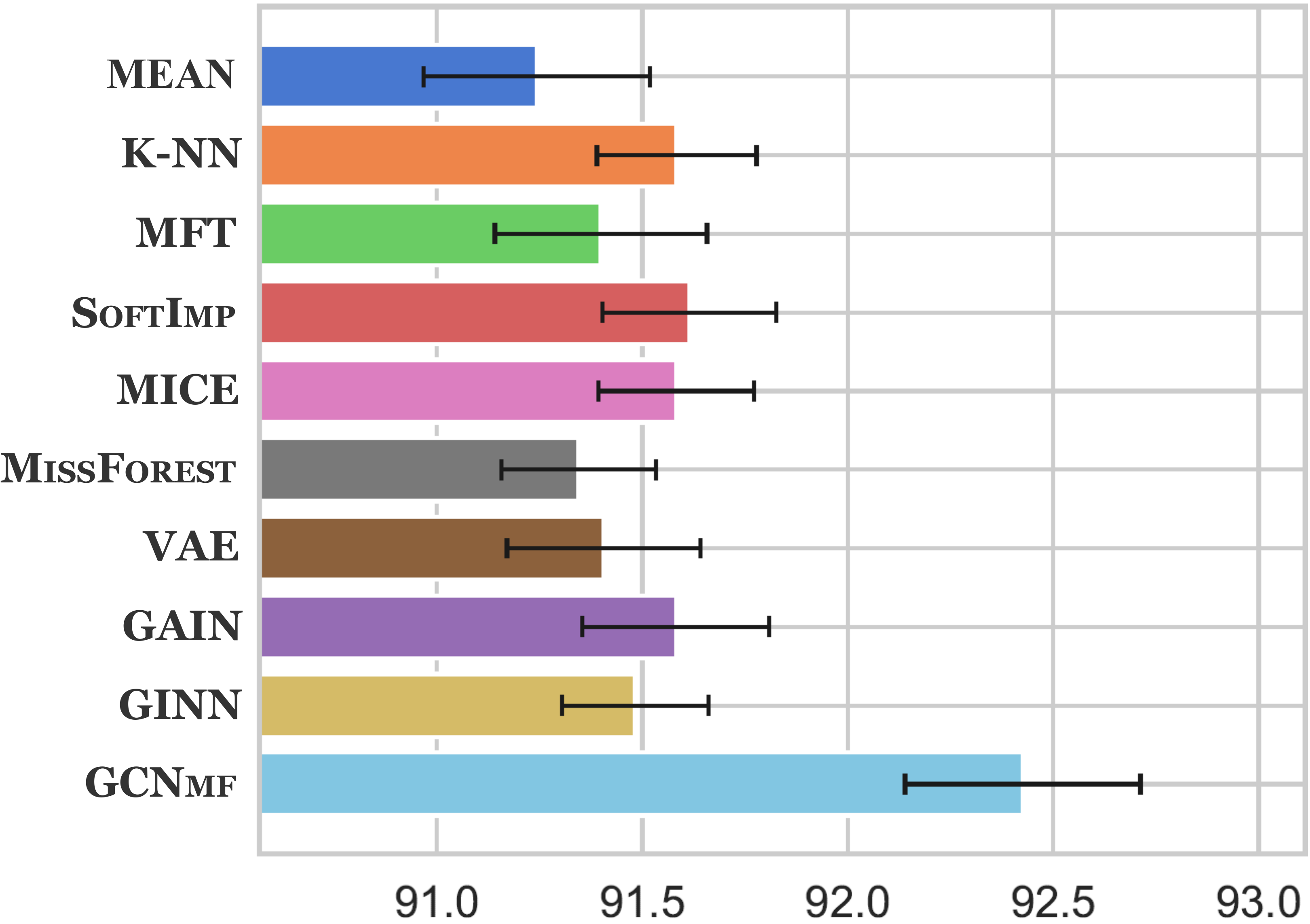}
      \caption{AmaPhoto (biased randomly)}
      \label{subfig:}
    \end{subfigure}%
    \quad
    \begin{subfigure}{0.31\textwidth}
        \centering
      \includegraphics[width=\textwidth]{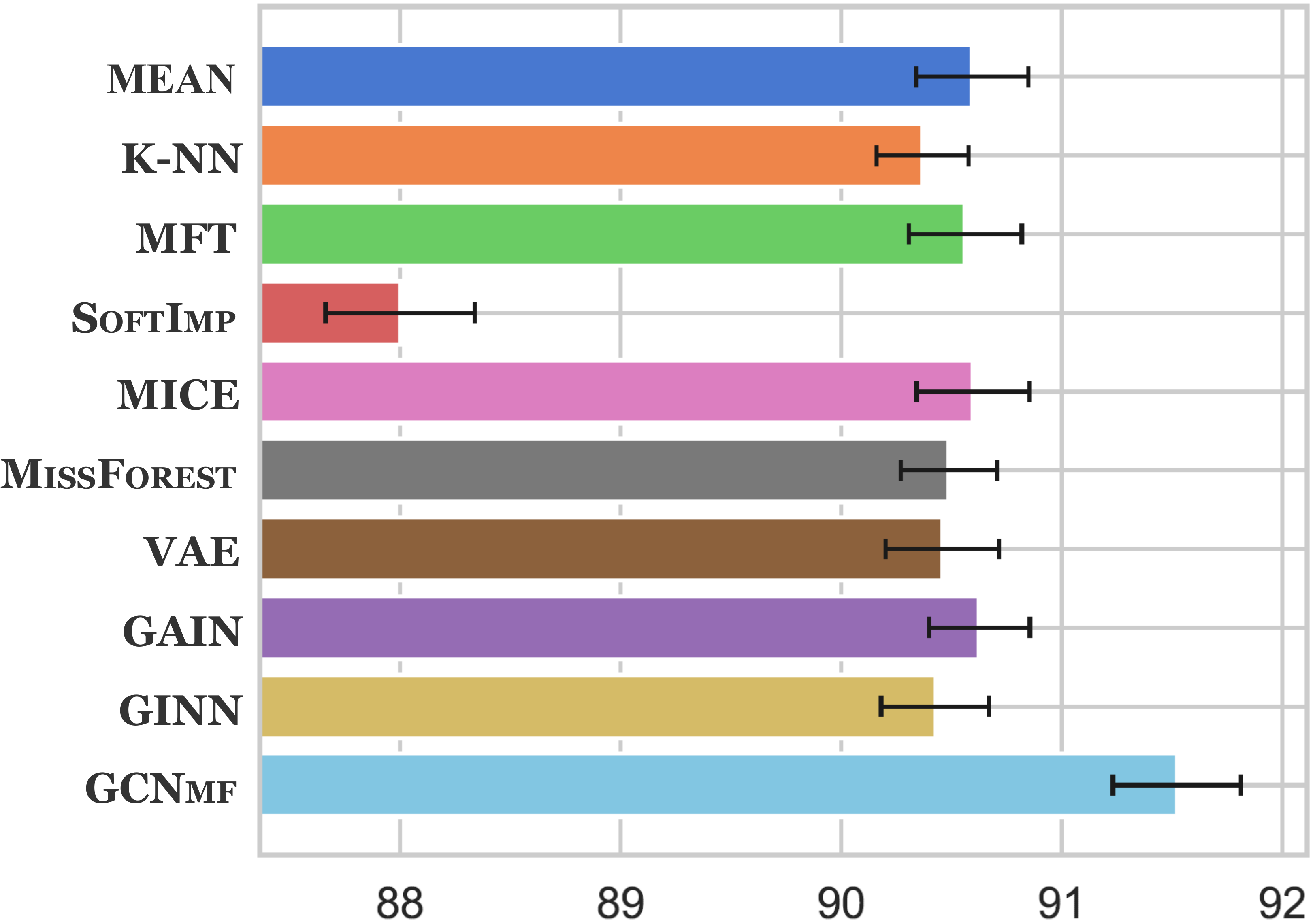}
      \caption{AmaPhoto (structurally)}
      \label{subfig:}
    \end{subfigure}%
    \\
    \vspace*{0.60em}
    \begin{subfigure}{0.31\textwidth}
        \centering
      \includegraphics[width=\textwidth]{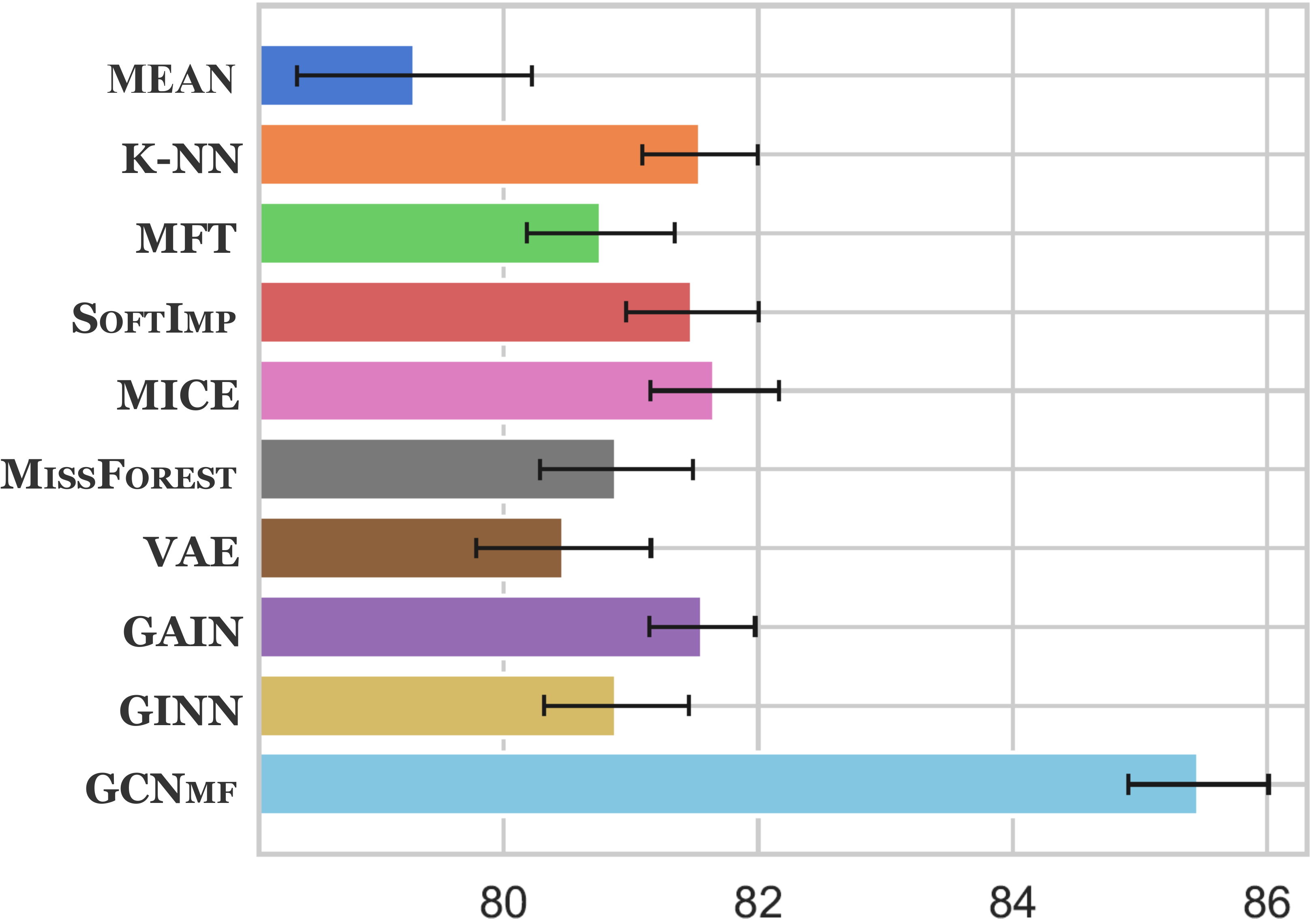}
      \caption{AmaComp (uniform randomly)}
      \label{subfig:}
    \end{subfigure}%
    \quad
    \begin{subfigure}{0.31\textwidth}
        \centering
      \includegraphics[width=\textwidth]{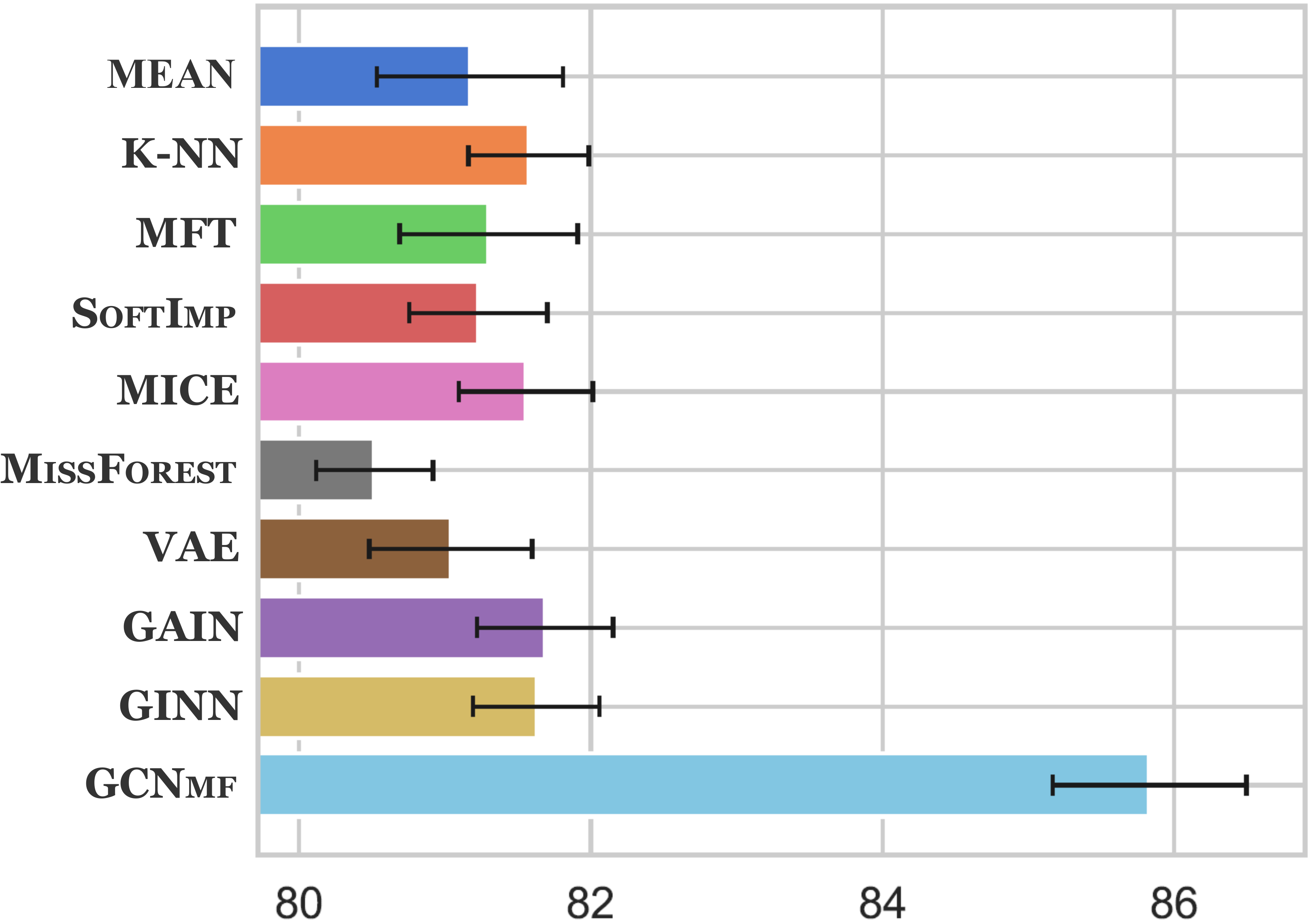}
      \caption{AmaComp (biased randomly)}
      \label{subfig:}
    \end{subfigure}%
    \quad
    \begin{subfigure}{0.31\textwidth}
        \centering
      \includegraphics[width=\textwidth]{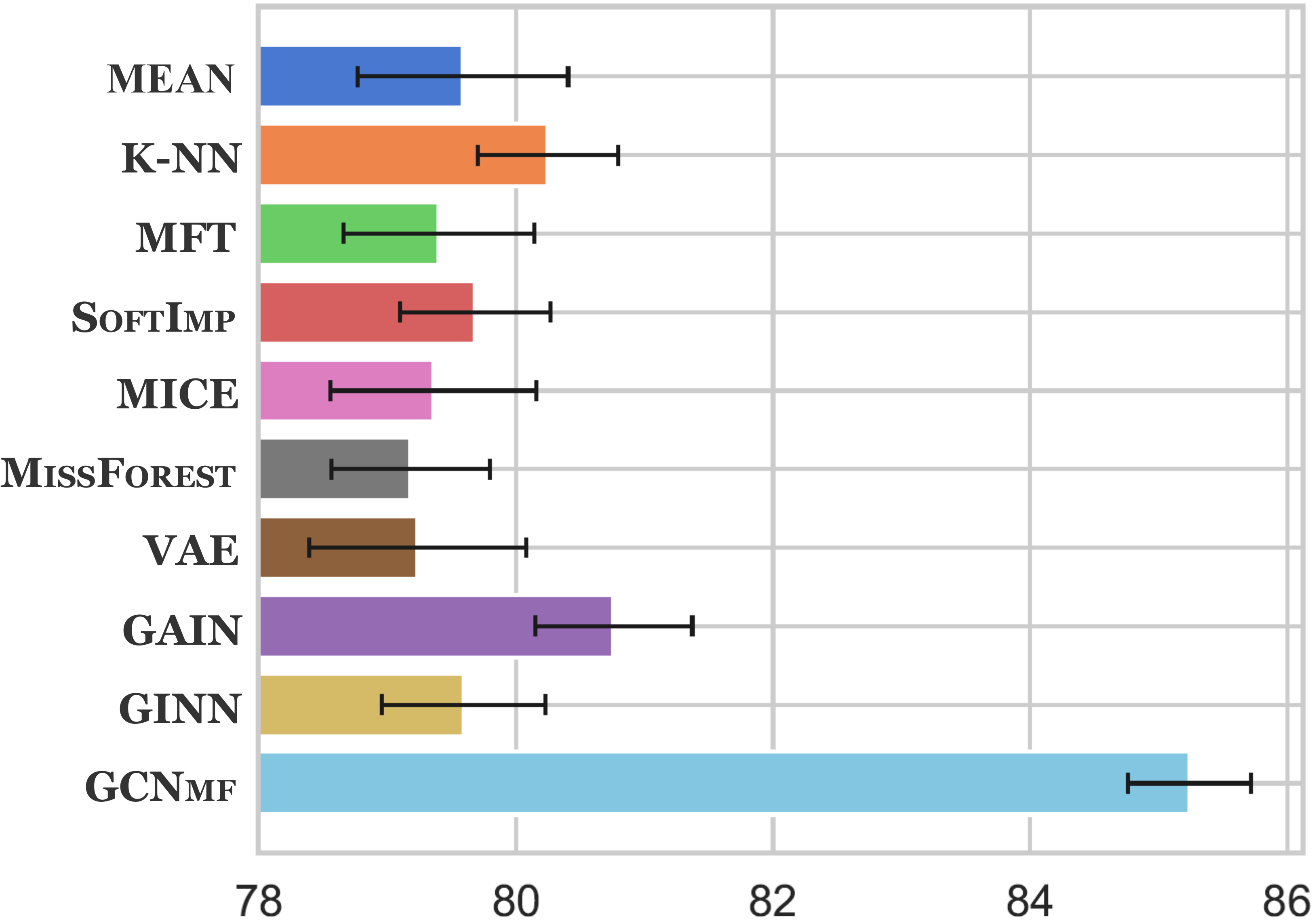}
      \caption{AmaComp (structurally)}
      \label{subfig:}
    \end{subfigure}%
\caption{Performance variance for node classification task ($mr=50\%$).}
\label{fig:random_std}
\end{figure*}

\begin{figure*}[!t]
    \centering
    \includegraphics[width=0.5\textwidth]{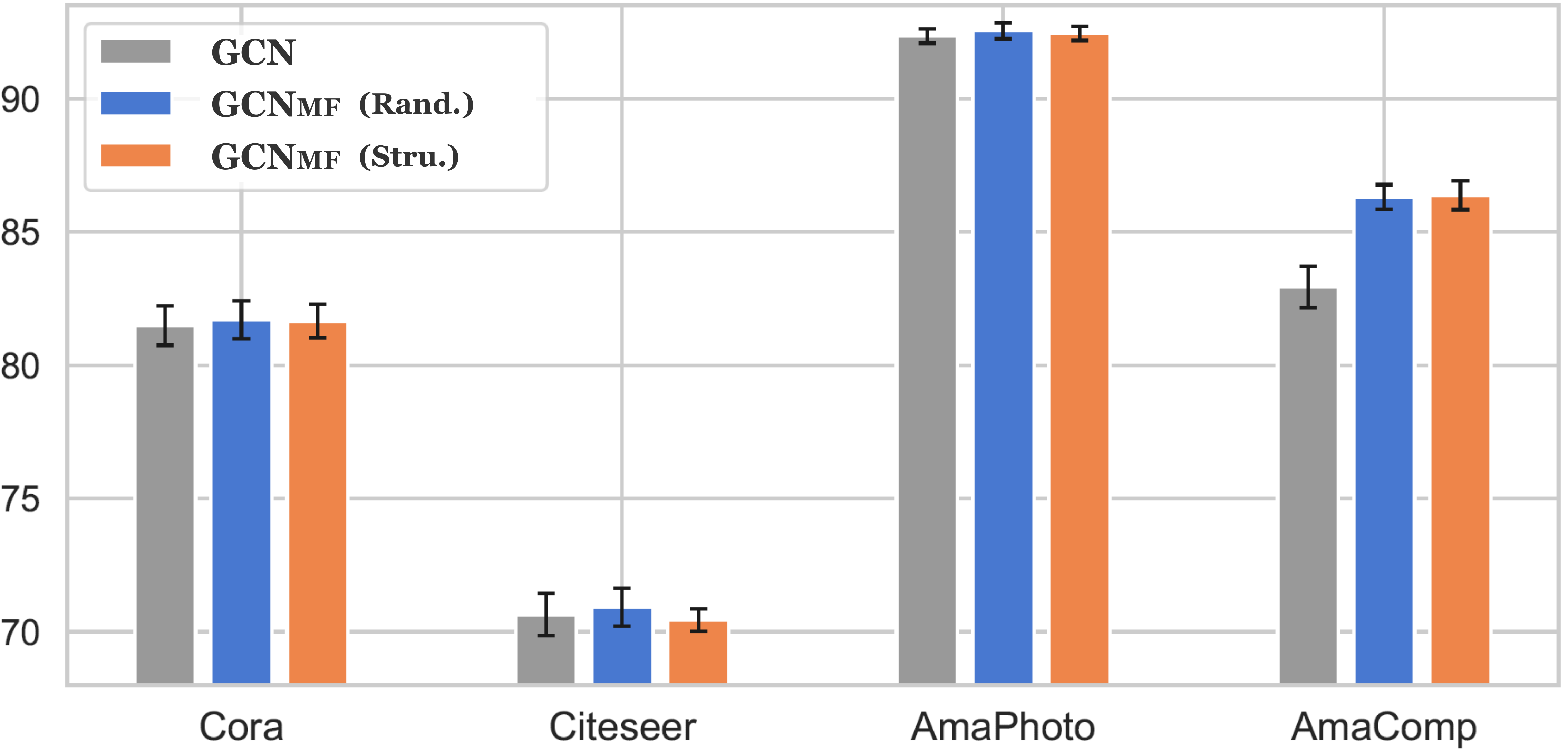}
    \caption{Performance variance of \textsc{GCNmf} ($mr=10\%$) and \textsc{GCN} ($mr=0\%$) for node classification task.}
    \label{fig:gcn_std}
\end{figure*}

\subsection{Node Classification}
We conducted experiments for the node classification task. We followed the data splits of previous work \cite{Yang2016revisiting} on Cora and Citeseer. As for AmaPhoto and AmaComp, we randomly chose 40 nodes per class for training, 500 nodes for validation, and the remaining for testing. We gradually increased the missing rate $mr$ from 10\% to 90\%. With each missing rate, we generated five instances of missing data and evaluated the performance twenty times for each instance. To ensure a fair comparison, we employed the following parameter settings of GCN model for all approaches: we set the number of layers to 2, the number of hidden units to 16 (Cora and Citeseer) and 64 (AmaPhoto and AmaComp). Moreover, we adopted an early stopping strategy with a patience of 100 epochs to avoid over-fitting \cite{velivckovic2017GAT}.

Table~\ref{table_cora} - Table~\ref{table_amacomp} lists the accuracy obtained by different methods. Bold and underline indicate the best and the second best score for each setting. Moreover, we provide the performance results of another three methods as a reference: 1) GCN in the setting of complete features ($\mathcal{S}=\emptyset$); 2) GCN without using node features (using the identity matrix instead of the node feature matrix $\mathbf{X}$); 3) RGCN \cite{zhu2019robust} in the setting that node features are under adversarial attacks (we deliberately perturbated the features that map to the same set of the uniform randomly missing features, and modified the node feature matrix $\mathbf{X}^*$; then we feed $\mathbf{X}^*$ to RGCN).

Note that some results of MICE (in Cora and Citeseer) and \textsc{MissForest} (in Citeseer and AmaComp) are not available because we encountered unexpected runtime errors or the program takes more than 24 hours to terminate. We have the following observations.

First, \textsc{GCNmf} demonstrates the best performance and there is no method that clearly wins the second place. \textsc{GCNmf} achieves the highest accuracy for almost all of the missing rates and across all datasets, with only four exceptions. For the uniform randomly missing case, \model{} is markedly superior to the others. It achieves improvement of up to $8.69\%$, $11.82\%$, $2.30\%$, and $5.24\%$ when compared with the best accuracy scores among baselines in the four datasets, respectively.
For the biased randomly missing case, the improvement is up to $10.64\%$, $9.39\%$, $2.57\%$, and $5.46\%$, respectively.
For the structurally missing case, this advantage becomes even greater, with the corresponding maximum improvement raising to $99.43\%$, $102.96\%$, $6.97\%$, and $35.36\%$, respectively.
Most strikingly, when the missing rate reaches 80\%, i.e., the features of $80\%$ nodes are not known, \textsc{GCNmf} can still achieve an accuracy of 68.00\% in Cora, while all baselines fail. 
 
Secondly, \textsc{GCNmf} is more appealing when a large portion of features are missing. This can be explained by the fact that the performance gain, on the whole, becomes larger and larger as the missing rate increases. In contrast, the imputation based method becomes less reliable at high missing rates. For example, the accuracy of baselines (except for \textsc{SoftImp}) falls to below 20.0\% when the missing rate reaches 90\% for the structurally missing case in Cora. 
 
Thirdly, it is interesting to note that \textsc{GCNmf} even outperforms \textsc{GCN} when only a small number of features are missing. For example, \textsc{GCNmf} holds a slim advantage over \textsc{GCN} when the missing rate is 10\% in the four datasets. This indicates that \textsc{GCNmf} is robust against low-level missing features. Moreover, \textsc{GCNmf} achieves much higher accuracy than RGCN. This is easy to understand because the task is more challenging for RGCN than \textsc{GCNmf}. 

Figure~\ref{fig:random_std} - Figure~\ref{fig:gcn_std} show the variability of the performance for different methods. We can see that \textsc{GCNmf} is more robust than the baselines, especially in Cora and Citeseer, where there is a high level of variability. Moreover, \textsc{GCNmf} and \textsc{GCN} are on the same level of variability. This implies that representing incomplete features by GMM and calculating the expected activation of neurons do not undermine the robustness of GCN.

\begin{table*}[!t]
    \centering
    \caption{The AUC results for link prediction task in Cora.}
    \label{table:linkcora}
    \scalebox{.84}{
    \begin{tabular}{c|c|ccccccccc} \toprule
Missing type & Missing rate & 10\% & 20\% & 30\% & 40\% & 50\% & 60\% & 70\% & 80\% & 90\%\\ \midrule
\multirow{13}{*}{\shortstack{Uniform\\Randomly\\Missing}} & \textsc{Mean} & 90.72 & 90.41 & 90.10 & 89.79 & 89.11 & 88.40 & 87.13 & 84.47 & 74.97\\
& \textsc{K-NN} & 92.20 & 91.86 & 91.34 & 90.93 & 90.19 & 89.03 & 87.62 & 85.69 & 81.55\\
& \textsc{MFT} & 92.16 & 91.86 & 91.37 & 90.91 & 90.14 & 88.37 & 86.11 & 84.10 & 79.94\\
& \textsc{SoftImp} & 90.88 & 90.79 & 90.64 & 90.40 & 89.98 & 89.22 & 88.37 & 86.75 & \underline{84.13}\\
& \textsc{MICE} & -- & -- & -- & -- & -- & -- & -- & -- & --\\
& \textsc{MissForest} & \underline{92.32} & \underline{92.04} & 91.61 & 90.95 & 90.33 & 89.34 & 88.33 & 86.41 & 82.78\\
& \textsc{VAE} & 92.23 & 91.91 & 91.33 & 90.54 & 89.28 & 86.98 & 82.52 & 77.74 & 77.27\\
& \textsc{GAIN} & 92.17 & 91.87 & 91.46 & \underline{91.00} & \underline{90.57} & \underline{89.78} & \underline{89.17} & \underline{88.13} & \textbf{86.01}\\
& \textsc{GINN} & 92.15 & 91.96 & \underline{91.62} & \underline{91.00} & 90.28 & 88.94 & 87.66 & 84.73 & 74.90\\
& \textsc{GCNmf} & \textbf{94.09} & \textbf{93.50} & \textbf{93.05} & \textbf{92.40} & \textbf{92.29} & \textbf{91.79} & \textbf{90.77} & \textbf{88.32} & 81.46\\ \cmidrule{2-11}
& \multirow{3}{*}{\shortstack{Performance gain \\ (\%)}} & 1.92 & 1.59 & 1.56 & 1.54 & 1.90 & 2.24 & 1.79 & 0.22 & -5.29\\
&  & \textbar & \textbar & \textbar & \textbar & \textbar & \textbar & \textbar & \textbar & \textbar\\
&  & 3.71 & 3.42 & 3.27 & 2.91 & 3.57 & 5.53 & 10.00 & 13.61 & 8.76\\
\midrule
\multirow{13}{*}{\shortstack{Biased\\Randomly\\Missing}} & \textsc{Mean} & 92.18 & 92.08 & 92.14 & 91.89 & 91.43 & 91.01 & 89.55 & 87.19 & 76.96\\
& \textsc{K-NN} & 92.17 & 92.06 & 92.02 & 91.83 & 91.47 & 90.92 & 89.84 & 87.85 & 81.65\\
& \textsc{MFT} & 92.17 & 91.44 & 90.65 & 90.00 & 89.50 & 88.91 & 87.48 & 85.36 & 80.20\\
& \textsc{SoftImp} & \underline{92.35} & \underline{92.34} & \underline{92.35} & \underline{92.08} & \underline{91.74} & \underline{91.36} & \textbf{90.03} & \underline{88.44} & \underline{86.17}\\
& \textsc{MICE} & -- & -- & -- & -- & -- & -- & -- & -- & --\\
& \textsc{MissForest} & 92.27 & 92.22 & 92.22 & 91.80 & 91.22 & 90.34 & 88.90 & 86.86 & 83.58\\
& \textsc{VAE} & 92.19 & 92.05 & 91.83 & 91.37 & 90.75 & 89.71 & 87.37 & 84.95 & 76.71\\
& \textsc{GAIN} & 92.18 & 92.01 & 91.88 & 91.70 & 91.28 & 90.75 & \underline{89.87} & \textbf{88.75} & \textbf{86.69}\\
& \textsc{GINN} & 92.15 & 92.11 & 92.04 & 91.88 & 91.52 & 90.89 & 89.45 & 87.36 & 75.38\\
& \textsc{GCNmf} & \textbf{94.35} & \textbf{94.20} & \textbf{93.90} & \textbf{93.15} & \textbf{92.43} & \textbf{91.46} & \textbf{90.03} & 86.10 & 81.72\\ \cmidrule{2-11}
& \multirow{3}{*}{\shortstack{Performance gain \\ (\%)}} & 2.17 & 2.01 & 1.68 & 1.16 & 0.75 & 0.11 & 0.18 & -2.99 & -5.73\\
&  & \textbar & \textbar & \textbar & \textbar & \textbar & \textbar & \textbar & \textbar & \textbar\\
&  & 2.39 & 3.02 & 3.59 & 3.50 & 3.27 & 2.87 & 3.04 & 1.35 & 8.41\\
\midrule
\multirow{13}{*}{\shortstack{Structurally\\Missing}} & \textsc{Mean} & 90.34 & 89.79 & 89.12 & 88.26 & 87.12 & 85.33 & 83.23 & 79.61 & 71.79\\
& \textsc{K-NN} & \underline{91.60} & 91.08 & \underline{90.38} & \underline{89.36} & \underline{88.34} & \textbf{87.16} & \textbf{85.40} & \textbf{82.09} & 76.12\\
& \textsc{MFT} & 91.51 & 91.00 & 89.95 & 89.11 & 87.36 & 85.81 & 82.90 & 77.73 & 73.72\\
& \textsc{SoftImp} & 90.29 & 89.67 & 88.86 & 87.86 & 86.77 & 85.36 & 83.07 & \underline{81.53} & \textbf{77.38}\\
& \textsc{MICE} & 91.58 & \underline{91.11} & 90.30 & 89.34 & 88.18 & \underline{86.70} & \underline{84.24} & 80.31 & 72.63\\
& \textsc{MissForest} & 91.57 & 91.05 & 90.23 & \underline{89.36} & \underline{88.34} & \textbf{87.16} & \textbf{85.40} & \textbf{82.09} & \underline{76.22}\\
& \textsc{VAE} & 91.49 & 90.76 & 89.49 & 87.27 & 83.81 & 80.07 & 73.46 & 67.55 & 65.80\\
& \textsc{GAIN} & \underline{91.60} & 91.08 & \underline{90.38} & \underline{89.36} & \underline{88.34} & \textbf{87.16} & \textbf{85.40} & \textbf{82.09} & 76.12\\
& \textsc{GINN} & 91.51 & 90.85 & 89.68 & 87.34 & 83.23 & 76.22 & 66.55 & 63.88 & 64.91\\
& \textsc{GCNmf} & \textbf{93.55} & \textbf{92.65} & \textbf{91.68} & \textbf{90.55} & \textbf{88.54} & 86.19 & 81.96 & 76.35 & 67.86\\ \cmidrule{2-11}
& \multirow{3}{*}{\shortstack{Performance gain \\ (\%)}} & 2.13 & 1.69 & 1.44 & 1.33 & 0.23 & -1.11 & -4.03 & -6.99 & -12.30\\
&  & \textbar & \textbar & \textbar & \textbar & \textbar & \textbar & \textbar & \textbar & \textbar\\
&  & 3.61 & 3.32 & 3.17 & 3.76 & 6.38 & 13.08 & 23.16 & 19.52 & 4.54\\
\midrule
\multicolumn{2}{c|}{GCN}& \multicolumn{9}{c}{92.42}\\
\multicolumn{2}{c|}{GCN w/o node features}& \multicolumn{9}{c}{85.90}\\ 
\bottomrule
    \end{tabular}
    }
\end{table*}

\begin{table*}[!t]
    \centering
    \caption{The AUC results for link prediction task in Citeseer.}
    \label{table:linkciteseer}
    \scalebox{.84}{
    \begin{tabular}{c|c|ccccccccc} \toprule
Missing type & Missing rate & 10\% & 20\% & 30\% & 40\% & 50\% & 60\% & 70\% & 80\% & 90\%\\ \midrule
\multirow{13}{*}{\shortstack{Uniform\\Randomly\\Missing}} & \textsc{Mean} & 89.01 & 88.56 & 88.01 & 87.33 & 86.42 & 85.30 & 83.77 & 81.43 & 75.47\\
& \textsc{K-NN} & 90.00 & 89.60 & 89.10 & 88.34 & 87.32 & 85.68 & 83.39 & 81.16 & 78.60\\
& \textsc{MFT} & 89.86 & 89.43 & 88.81 & 87.72 & 85.76 & 83.24 & 81.20 & 79.97 & 77.94\\
& \textsc{SoftImp} & \underline{90.19} & \underline{90.15} & \underline{89.81} & \underline{89.55} & \underline{88.97} & \underline{88.17} & \underline{86.80} & \underline{84.99} & 81.66\\
& \textsc{MICE} & -- & -- & -- & -- & -- & -- & -- & -- & --\\
& \textsc{MissForest} & -- & -- & -- & -- & -- & -- & -- & -- & --\\
& \textsc{VAE} & 89.85 & 89.09 & 88.13 & 87.22 & 85.36 & 83.55 & 80.64 & 74.89 & 64.69\\
& \textsc{GAIN} & 89.96 & 89.53 & 89.07 & 88.36 & 87.51 & 86.52 & 85.35 & 83.93 & \underline{81.70}\\
& \textsc{GINN} & 90.02 & 89.64 & 89.04 & 87.91 & 86.56 & 84.64 & 83.32 & 81.82 & 77.19\\
& \textsc{GCNmf} & \textbf{93.20} & \textbf{92.96} & \textbf{92.30} & \textbf{92.19} & \textbf{90.45} & \textbf{90.08} & \textbf{88.91} & \textbf{87.28} & \textbf{83.68}\\ \cmidrule{2-11}
& \multirow{3}{*}{\shortstack{Performance gain \\ (\%)}} & 3.34 & 3.12 & 2.77 & 2.95 & 1.66 & 2.17 & 2.43 & 2.69 & 2.42\\
&  & \textbar & \textbar & \textbar & \textbar & \textbar & \textbar & \textbar & \textbar & \textbar\\
&  & 4.71 & 4.97 & 4.87 & 5.70 & 5.96 & 8.22 & 10.26 & 16.54 & 29.36\\
\midrule
\multirow{13}{*}{\shortstack{Biased\\Randomly\\Missing}} & \textsc{Mean} & 89.94 & 89.88 & 89.63 & 89.33 & 89.25 & 88.55 & \underline{87.57} & 85.28 & 78.23\\
& \textsc{K-NN} & 90.00 & 89.98 & 89.81 & 89.54 & 89.31 & 88.52 & 87.47 & 84.97 & 78.85\\
& \textsc{MFT} & 89.98 & 87.50 & 85.88 & 85.07 & 84.32 & 83.76 & 82.85 & 81.54 & 78.23\\
& \textsc{SoftImp} & \underline{90.31} & \underline{90.25} & \underline{90.23} & \underline{89.99} & \underline{89.90} & \underline{89.03} & 87.12 & \underline{85.96} & 80.63\\
& \textsc{MICE} & -- & -- & -- & -- & -- & -- & -- & -- & --\\
& \textsc{MissForest} & -- & -- & -- & -- & -- & -- & -- & -- & --\\
& \textsc{VAE} & 89.90 & 89.25 & 88.33 & 87.32 & 86.26 & 83.78 & 83.05 & 80.71 & 62.51\\
& \textsc{GAIN} & 89.97 & 89.87 & 89.60 & 89.32 & 88.89 & 87.85 & 87.00 & 85.05 & \textbf{81.95}\\
& \textsc{GINN} & 90.27 & 89.99 & 89.85 & 89.47 & 89.10 & 88.15 & 87.21 & 84.47 & 76.83\\
& \textsc{GCNmf} & \textbf{93.53} & \textbf{93.38} & \textbf{92.81} & \textbf{92.48} & \textbf{91.68} & \textbf{91.25} & \textbf{89.54} & \textbf{86.73} & \underline{81.43}\\ \cmidrule{2-11}
& \multirow{3}{*}{\shortstack{Performance gain \\ (\%)}} & 3.57 & 3.47 & 2.86 & 2.77 & 1.98 & 2.49 & 2.25 & 0.90 & -0.63\\
&  & \textbar & \textbar & \textbar & \textbar & \textbar & \textbar & \textbar & \textbar & \textbar\\
&  & 4.04 & 6.72 & 8.07 & 8.71 & 8.73 & 8.94 & 8.07 & 7.46 & 30.27\\
\midrule
\multirow{13}{*}{\shortstack{Structurally\\Missing}} & \textsc{Mean} & 88.16 & 86.95 & 85.76 & 84.20 & 82.43 & 80.83 & 78.92 & 75.79 & 69.76\\
& \textsc{K-NN} & \underline{89.50} & \underline{88.36} & \underline{87.01} & \underline{85.52} & \underline{83.85} & \underline{82.11} & \underline{79.81} & \underline{76.49} & 70.86\\
& \textsc{MFT} & 89.24 & 87.96 & 86.53 & 84.76 & 83.19 & 80.67 & 78.35 & 75.97 & \underline{72.64}\\
& \textsc{SoftImp} & \underline{89.50} & \underline{88.36} & \underline{87.01} & \underline{85.52} & \underline{83.85} & \underline{82.11} & \underline{79.81} & \underline{76.49} & 70.86\\
& \textsc{MICE} & -- & -- & -- & -- & -- & -- & -- & -- & --\\
& \textsc{MissForest} & -- & -- & -- & -- & -- & -- & -- & -- & --\\
& \textsc{VAE} & 88.57 & 86.83 & 84.32 & 80.96 & 77.49 & 74.01 & 67.84 & 63.06 & 60.39\\
& \textsc{GAIN} & \underline{89.50} & \underline{88.36} & \underline{87.01} & \underline{85.52} & \underline{83.85} & \underline{82.11} & \underline{79.81} & \underline{76.49} & 70.86\\
& \textsc{GINN} & 87.48 & 83.35 & 77.50 & 70.06 & 64.31 & 59.45 & 57.95 & 54.88 & 50.81\\
& \textsc{GCNmf} & \textbf{92.23} & \textbf{90.54} & \textbf{88.77} & \textbf{85.74} & \textbf{84.78} & \textbf{84.59} & \textbf{82.00} & \textbf{77.21} & \textbf{73.31}\\ \cmidrule{2-11}
& \multirow{3}{*}{\shortstack{Performance gain \\ (\%)}} & 3.05 & 2.47 & 2.02 & 0.26 & 1.11 & 3.02 & 2.74 & 0.94 & 0.92\\
&  & \textbar & \textbar & \textbar & \textbar & \textbar & \textbar & \textbar & \textbar & \textbar\\
&  & 5.43 & 8.63 & 14.54 & 22.38 & 31.83 & 42.29 & 41.50 & 40.69 & 44.28\\
\midrule
\multicolumn{2}{c|}{GCN}& \multicolumn{9}{c}{90.25}\\
\multicolumn{2}{c|}{GCN w/o node features}& \multicolumn{9}{c}{79.94}\\ 
\bottomrule
    \end{tabular}
    }
\end{table*}

\subsection{Link Prediction}

The second experiment is for the link prediction task in the Cora and Citeseer citation graphs. We took VGAE \cite{kipf2016variational} as the base model, which is a variational graph autoencoder and employs GCN as an encoder. We gradually increased the missing rate $mr$ from 10\% to 90\% and compare \textsc{GCNmf} against baselines within the base model framework. Following the previous work \cite{kipf2016variational}, we randomly chose 10\% edges for testing, 5\% edges for validation, and the remaining edges for training; we used a 32-dim hidden layer and 16-dim latent variables in the base model. 

Tables~\ref{table:linkcora} and~\ref{table:linkciteseer} show the average AUC scores obtained by different methods. Bold and underline indicate the best and the second best score for each setting. We also provide the performance results of 1) GCN in the setting of complete features ($\mathcal{S}=\emptyset$), and 2) GCN without node features (using the identity matrix instead of the node feature matrix $\mathbf{X}$) as a reference.

We can reach a similar conclusion as the node classification task. \textsc{GCNmf} exhibits the best overall performance. In particular, \textsc{GCNmf} demonstrates excellent performance and is overwhelmingly superior to all baselines in Citeseer; \textsc{GCNmf} outperforms the baselines in most cases in Cora, with only several exceptions when the missing rate reaches high. Again, we can observe the robustness merit of \textsc{GCNmf}, as it even outperforms GCN when the missing rate is low.

We attribute the superiority of \textsc{GCNmf} to the joint learning of GMM and network parameters. Actually, our approach can be understood as calculating the expected activation of neurons over the imputations drawn from missing data density in the first layer. It is the end-to-end joint learning of the parameters that make our approach less likely to converge to sub-optimal solutions.

\begin{table*}[!t]
    \centering
    \caption{The running time (seconds) of different approaches for the uniform randomly missing case ($mr=50\%$). The figure in the parentheses indicates the time for initialization of GMM parameters.}
    \label{tab:runtime}
    \begin{tabular}{l|rrrr}
        \toprule
        & Cora & Citeseer & AmaPhoto & AmaComp\\
        \midrule
        \textsc{MEAN} & 1.10 & 1.24 & 12.09 & 14.14\\
        \textsc{K-NN} & 125.04 & 480.19 & 482.73 & 1505.14\\
        \textsc{MFT} & 141.14 & 567.50 & 428.95 & 906.52\\
        \textsc{SoftImp} & 115.15 & 850.55 & 59.26 & 95.14\\
        \textsc{MICE} & -- & -- & 3879.59 & 6705.73 \\
        \textsc{MissForest} & 4039.10 & -- & 32528.25 & 48264.58\\
        \textsc{VAE} & 7.91 & 8.64 & 14.23 & 18.78\\
        \textsc{GAIN} & 79.35 & 426.10 & 36.06 & 35.19\\
        \textsc{GINN} & 300.64 & 839.96 & 998.03 & 3199.96\\
        \textsc{GCNmf} & 7.43 (0.59) & 13.52 (2.60) & 22.64 (4.11) & 42.38 (9.75)\\
        \midrule
        \textsc{GCN} & 0.86 & 0.91 & 6.82 & 7.79\\
        \bottomrule
    \end{tabular}
\end{table*}

\subsection{Running Time Comparison} \label{subsec:runtime}
We compare the running time of different approaches in Table \ref{tab:runtime}. The numbers represent the sum of time for parameter initialization, missing value imputation, and model training. We also provide a reference time of GCN when $\mathcal{S}=\emptyset$. We can observe that \textsc{GCNmf} algorithm runs in reasonable time, with model training taking the majority of time (the time for initialization of GMM parameters only accounts for less than $25\%$). In comparison, \textsc{GCNmf} is slower than \textsc{MEAN} and VAE, but is much faster than the other seven methods. We note that some imputation techniques suffer due to the high dimension of features. For example, \textsc{MissForest} did not finish within 24 hours in Citeseer.

\subsection{Analysis of \textsc{GCNmf}}
In this section, we provide study of \textsc{GCNmf} in terms of hyper-parameter sensitivity, optimization analysis, and quality of the reconstructed features.

\subsubsection{Hyper-parameter analysis}
Figure~\ref{fig:para_analysis} depicts the performance results with different assignments on the Gaussian components $K$ and the number of hidden units $D^{(1)}$ in Cora and AmaPhoto datasets. We can observe that the performance reaches a plateau when we have enough number of hidden units to transcribe the information, i.e., $D^{(1)} \geq 16$ for Cora and $D^{(1)} \geq 32$ for AmaPhoto. On the other hand, the performance is not sensitive to $K$, with differences between the best and worst less than $0.82\%$ when $D^{(1)} \geq 16$ in Cora and $0.30\%$ when $D^{(1)} \geq 32$ in AmaPhoto, respectively.

\begin{figure*}[!t]
    \centering
    \begin{subfigure}{0.48\textwidth}
        \centering
      \includegraphics[width=\textwidth]{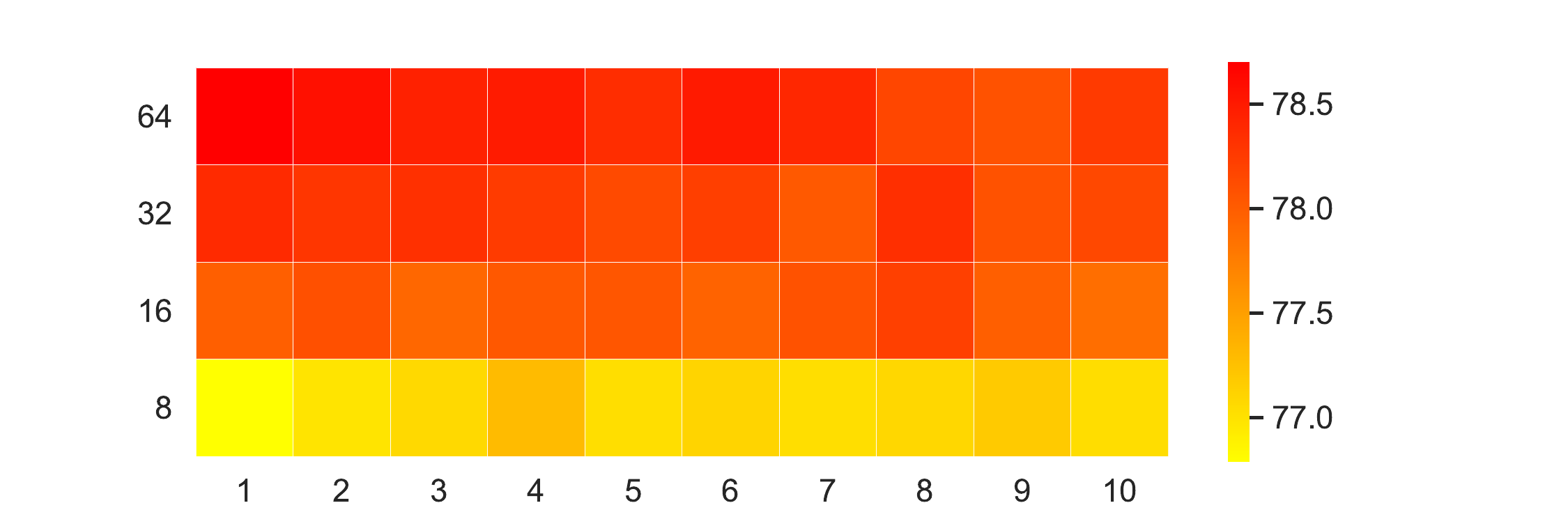}
      \caption{Cora for uniform randomly missing features}
      \label{subfig:cora_para_case1}
    \end{subfigure}%
    \quad 
    \begin{subfigure}{0.48\textwidth}
        \centering
      \includegraphics[width=\textwidth]{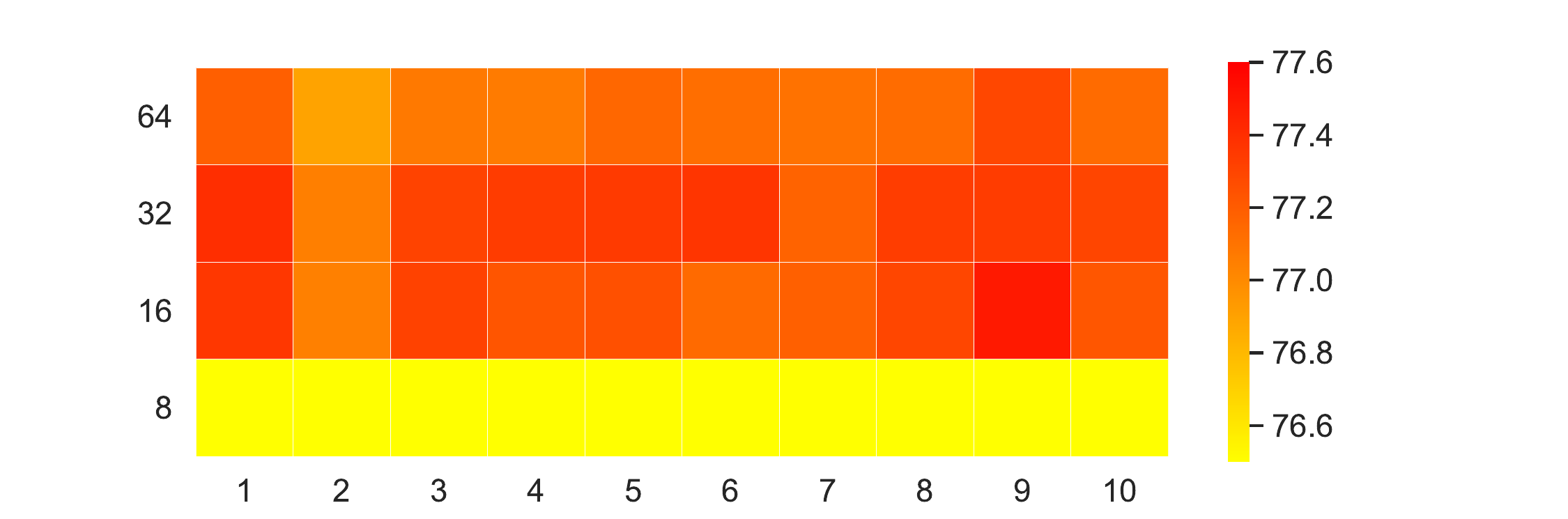}
      \caption{Cora for structurally missing features}
      \label{subfig:cora_para_case2}
    \end{subfigure}%
    \\
    \vspace*{0.60em}
    \begin{subfigure}{0.48\textwidth}
        \centering
      \includegraphics[width=\textwidth]{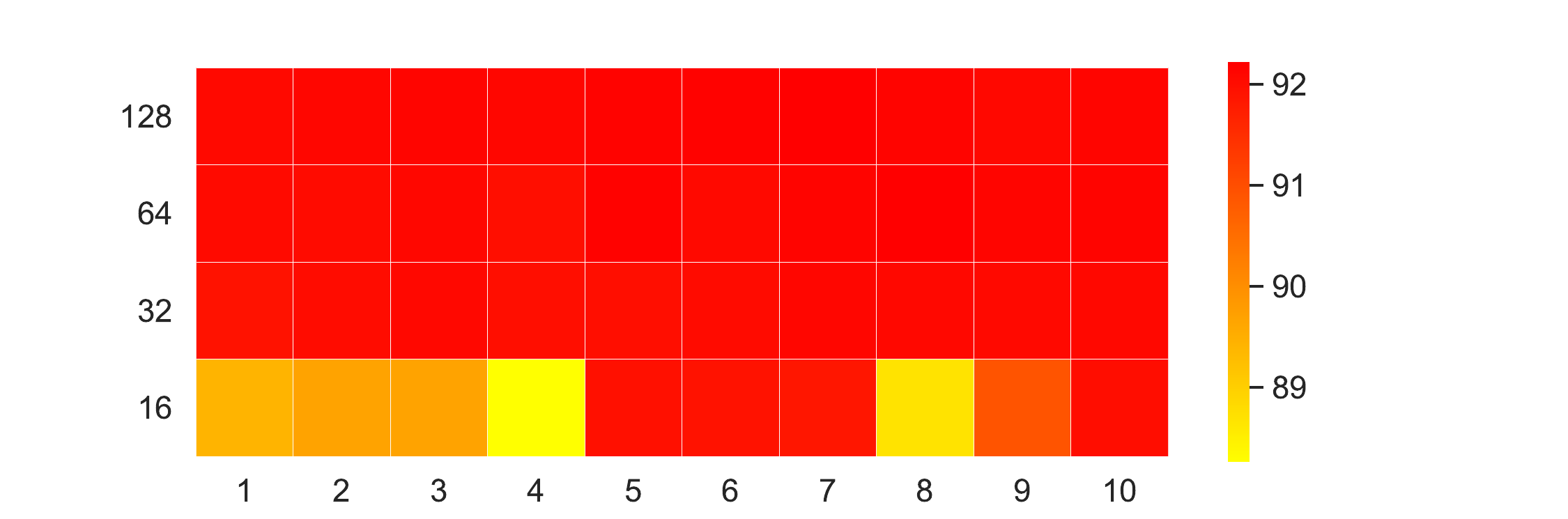}
      \caption{AmaPhoto for uniform randomly missing features}
      \label{subfig:photo_para_case1}
    \end{subfigure}%
    \quad 
    \begin{subfigure}{0.48\textwidth}
        \centering
      \includegraphics[width=\textwidth]{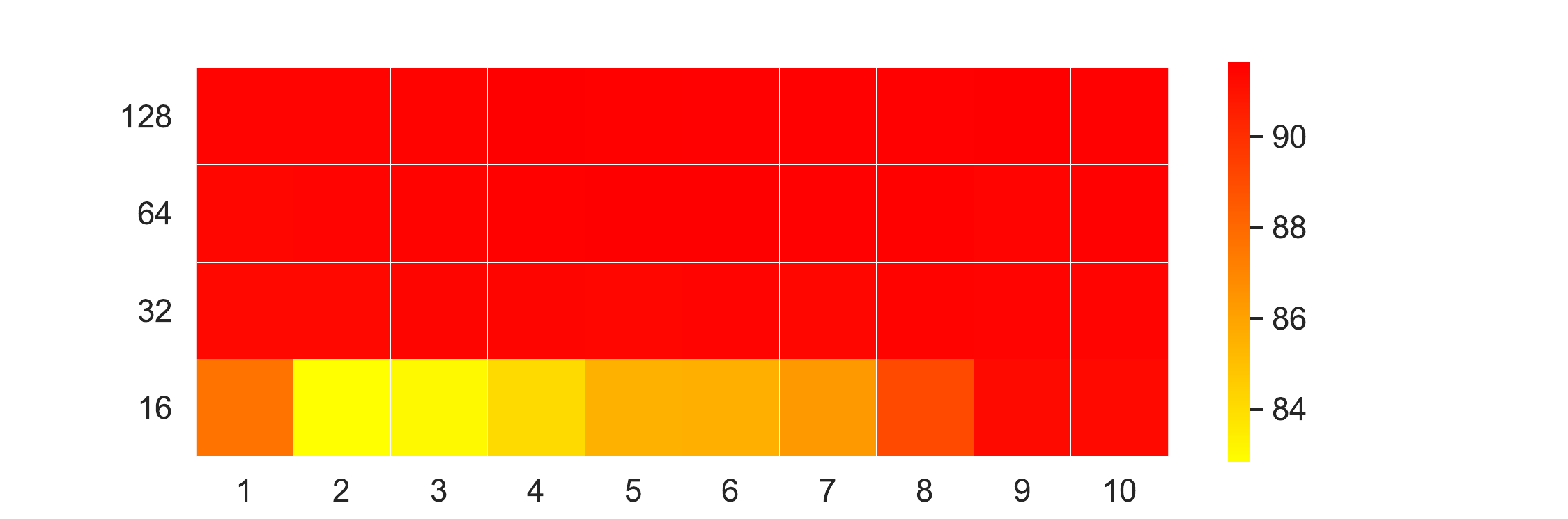}
      \caption{AmaPhoto for structurally missing features}
      \label{subfig:photo_para_case2}
    \end{subfigure}%
\caption{Node classification results for \textsc{GCNmf} with different assignments on the Gaussian components $K$ and the number of hidden units $D^{(1)}$ ($mr=50\%$). The x-axis represents $K$. The y-axis represents $D^{(1)}$.}
\label{fig:para_analysis}
\end{figure*}

\subsubsection{Analysis of Optimization}
\textsc{GCNmf} employs a joint optimization of GMM and GCN within the same network architecture. Alternatively, we can consider a two-step optimization strategy: in the first step we optimize GMM parameters with input node features using EM algorithm; in the second step we optimize GCN parameters by gradient descent algorithm while fixing the GMM parameters.

We compare the two optimization strategies in Table~\ref{table:2step}. We can observe that the joint optimization clearly beats the two-step optimization. The advantage becomes greater and greater as the missing rate increases. In particular, when the missing rate becomes high, the two-step optimization fails to learn the ``right'' model parameters and the performance deteriorates sharply.

\begin{table*}[!t]
  \centering
  \caption{The accuracy results for joint optimization and two-step optimization in node classification task.}
  \label{table:2step}
  \scalebox{.83}{
  \begin{tabular}{cc|ccccccccc} \toprule
\topmidheader{11}{Uniform randomly missnig features}
Dataset & Missing rate & 10\% & 20\% & 30\% & 40\% & 50\% & 60\% & 70\% & 80\% & 90\%\\ \midrule
\multirow{2}{*}{Cora}
& Joint Opt. & \textbf{81.70} & \textbf{81.66} & \textbf{80.41} & \textbf{79.52} & \textbf{77.91} & \textbf{76.67} & \textbf{74.38} & \textbf{70.57} & \textbf{63.49}\\
& Two-step Opt. & 81.50 & 81.43 & 79.81 & 79.35 & 76.75 & 76.04 & 73.97 & 69.16 & 61.46\\
\midrule
\multirow{2}{*}{Citeseer} 
& Joint Opt. & \textbf{70.93} & \textbf{70.82} & \textbf{69.84} & 68.83 & \textbf{67.03} & \textbf{64.78} & \textbf{60.70} & \textbf{55.38} & \textbf{47.78}\\
& Two-step Opt. & 70.54 & 70.73 & 69.66 & \textbf{69.20} & 66.59 & 64.52 & 60.07 & 53.68 & 46.53\\
\midrule
\midheader{11}{Structurally missing features}
Dataset & Missing rate & 10\% & 20\% & 30\% & 40\% & 50\% & 60\% & 70\% & 80\% & 90\%\\ \midrule
\multirow{2}{*}{AmaPhoto}
& Joint Opt. & 92.45 & \textbf{92.32} & \textbf{92.08} & \textbf{91.88} & \textbf{91.52} & \textbf{90.89} & \textbf{90.39} & \textbf{89.64} & \textbf{86.09}\\
& Two-step Opt. & \textbf{92.49} & 92.23 & 91.90 & 91.40 & 90.73 & 87.94 & 84.93 & 62.46 & 32.44\\
\midrule
\multirow{2}{*}{AmaComp} 
& Joint Opt. & \textbf{86.37} & \textbf{86.22} & \textbf{85.80} & \textbf{85.43} & \textbf{85.24} & \textbf{84.73} & \textbf{84.06} & \textbf{80.63} & \textbf{73.42}\\
& Two-step Opt. & 86.30 & 85.99 & 85.49 & 84.69 & 83.79 & 82.76 & 80.23 & 71.80 & 39.98\\
\bottomrule
  \end{tabular}
  }
\end{table*}

\begin{figure}[!t]
    \centering
    \includegraphics[width=0.5\textwidth]{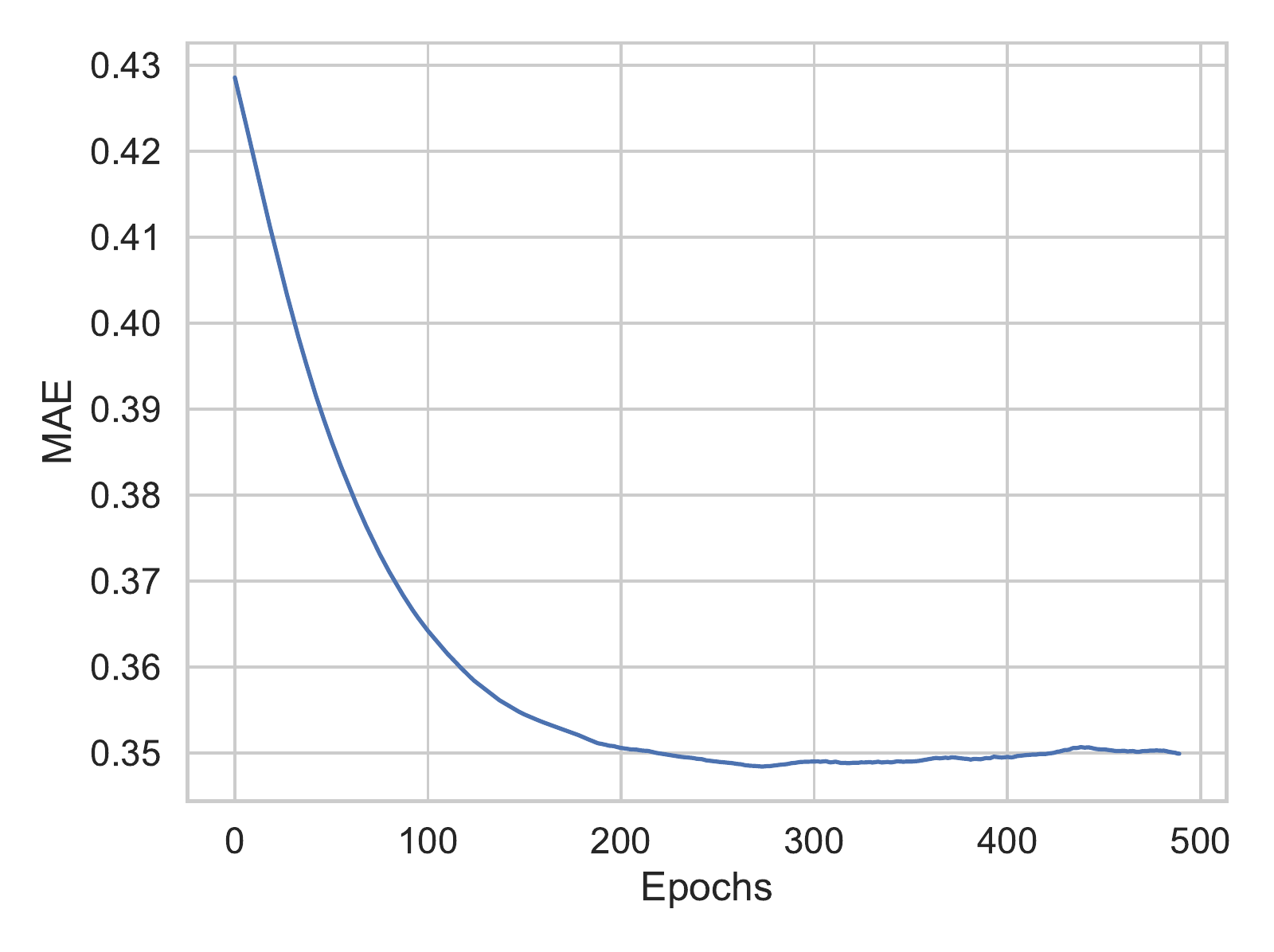}
    \caption{MAE of the reconstructed features during the training process for node classification task (structurally missing features, mr = 20\%) in AmaPhoto.}
    \label{fig:casestudy}
\end{figure}

\subsubsection{Analysis of Reconstructed Node Features}
Finally, we conducted a study on how well the reconstructed features by \textsc{GCNmf}. The reconstructed features are the mean of GMM, namely the weighted average of the mean vectors. Figure~\ref{fig:casestudy} depicts the Mean Absolute Error (MAE) of the reconstructed features and true features during the training process of the node classification task in AmaPhoto. We can observe that MAE decreases as the number of training epochs increases, and it converges to around 0.35 after 200 epochs. This suggests that the trained GMM captures the density of features more accurately than the initial state optimized by EM algorithm. Although the training aims at learning node labels, it helps to reconstruct the missing features.


\section{Conclusion}
\label{sec:conclusion}
We proposed \model{} to supplement a severe deficiency of current GCN models---inability to handle graphs containing missing features. In contrast to the traditional strategy of imputing missing features before applying GCN, \model{} integrates the processing of missing features and graph learning within the same neural network architecture. Specifically, we propose a novel way to unify the representation of missing features and calculation of the expected activation of the first layer neurons in GCN. We empirically demonstrate that 1) \model{} is robust against low level of missing features, 2) \model{} significantly outperforms the imputation based methods in the node classification and link prediction tasks.

\begin{acks}
This work is partly supported by JSPS Grant-in-Aid for Early-Career Scientists (Grant Number 19K20352), JSPS Grant-in-Aid for Scientific Research(B) (Grant Number 17H01785), JST CREST (Grant Number JPMJCR1687), and the New Energy and Industrial Technology Development Organization (NEDO).
\end{acks}

\bibliographystyle{ACM-Reference-Format}
\bibliography{ref.bib}


\begin{thebibliography}{68}


\ifx \showCODEN    \undefined \def \showCODEN     #1{\unskip}     \fi
\ifx \showDOI      \undefined \def \showDOI       #1{#1}\fi
\ifx \showISBNx    \undefined \def \showISBNx     #1{\unskip}     \fi
\ifx \showISBNxiii \undefined \def \showISBNxiii  #1{\unskip}     \fi
\ifx \showISSN     \undefined \def \showISSN      #1{\unskip}     \fi
\ifx \showLCCN     \undefined \def \showLCCN      #1{\unskip}     \fi
\ifx \shownote     \undefined \def \shownote      #1{#1}          \fi
\ifx \showarticletitle \undefined \def \showarticletitle #1{#1}   \fi
\ifx \showURL      \undefined \def \showURL       {\relax}        \fi
\providecommand\bibfield[2]{#2}
\providecommand\bibinfo[2]{#2}
\providecommand\natexlab[1]{#1}
\providecommand\showeprint[2][]{arXiv:#2}

\bibitem[\protect\citeauthoryear{Abu-El-Haija, Perozzi, Kapoor, Harutyunyan,
  Alipourfard, Lerman, Steeg, and Galstyan}{Abu-El-Haija et~al\mbox{.}}{2019}]%
        {sami2019mixhop}
\bibfield{author}{\bibinfo{person}{Sami Abu-El-Haija}, \bibinfo{person}{Bryan
  Perozzi}, \bibinfo{person}{Amol Kapoor}, \bibinfo{person}{Hrayr Harutyunyan},
  \bibinfo{person}{Nazanin Alipourfard}, \bibinfo{person}{Kristina Lerman},
  \bibinfo{person}{Greg~Ver Steeg}, {and} \bibinfo{person}{Aram Galstyan}.}
  \bibinfo{year}{2019}\natexlab{}.
\newblock \showarticletitle{MixHop: Higher-Order Graph Convolution
  Architectures via Sparsified Neighborhood Mixing}. In
  \bibinfo{booktitle}{\emph{Proceedings of ICML}}.
\newblock


\bibitem[\protect\citeauthoryear{Akiba, Sano, Yanase, Ohta, and Koyama}{Akiba
  et~al\mbox{.}}{2019}]%
        {optuna_2019}
\bibfield{author}{\bibinfo{person}{Takuya Akiba}, \bibinfo{person}{Shotaro
  Sano}, \bibinfo{person}{Toshihiko Yanase}, \bibinfo{person}{Takeru Ohta},
  {and} \bibinfo{person}{Masanori Koyama}.} \bibinfo{year}{2019}\natexlab{}.
\newblock \showarticletitle{Optuna: A Next-generation Hyperparameter
  Optimization Framework}. In \bibinfo{booktitle}{\emph{Proceedings of KDD}}.
\newblock


\bibitem[\protect\citeauthoryear{Bai, Ding, Bian, Chen, Sun, and Wang}{Bai
  et~al\mbox{.}}{2019}]%
        {bai2019Simgnn}
\bibfield{author}{\bibinfo{person}{Yunsheng Bai}, \bibinfo{person}{Hao Ding},
  \bibinfo{person}{Song Bian}, \bibinfo{person}{Ting Chen},
  \bibinfo{person}{Yizhou Sun}, {and} \bibinfo{person}{Wei Wang}.}
  \bibinfo{year}{2019}\natexlab{}.
\newblock \showarticletitle{Simgnn: A neural network approach to fast graph
  similarity computation}. In \bibinfo{booktitle}{\emph{Proceedings of WSDM}}.
  \bibinfo{pages}{384--392}.
\newblock


\bibitem[\protect\citeauthoryear{Bastings, Titov, Aziz, Marcheggiani, and
  Sima’an}{Bastings et~al\mbox{.}}{2017}]%
        {bastings2017MachineTrans_GCN}
\bibfield{author}{\bibinfo{person}{Joost Bastings}, \bibinfo{person}{Ivan
  Titov}, \bibinfo{person}{Wilker Aziz}, \bibinfo{person}{Diego Marcheggiani},
  {and} \bibinfo{person}{Khalil Sima’an}.} \bibinfo{year}{2017}\natexlab{}.
\newblock \showarticletitle{Graph convolutional encoders for syntax-aware
  neural machine translation}. In \bibinfo{booktitle}{\emph{Proceedings of
  EMNLP}}. \bibinfo{pages}{1957--1967}.
\newblock


\bibitem[\protect\citeauthoryear{Batista, Monard, et~al\mbox{.}}{Batista
  et~al\mbox{.}}{2002}]%
        {batista2002study}
\bibfield{author}{\bibinfo{person}{Gustavo~EAPA Batista},
  \bibinfo{person}{Maria~Carolina Monard}, {et~al\mbox{.}}}
  \bibinfo{year}{2002}\natexlab{}.
\newblock \showarticletitle{A Study of K-Nearest Neighbour as an Imputation
  Method.}
\newblock \bibinfo{journal}{\emph{HIS}} \bibinfo{volume}{87},
  \bibinfo{number}{251--260} (\bibinfo{year}{2002}), \bibinfo{pages}{48}.
\newblock


\bibitem[\protect\citeauthoryear{Buuren and Groothuis-Oudshoorn}{Buuren and
  Groothuis-Oudshoorn}{2010}]%
        {buuren2010mice}
\bibfield{author}{\bibinfo{person}{S~van Buuren} {and} \bibinfo{person}{Karin
  Groothuis-Oudshoorn}.} \bibinfo{year}{2010}\natexlab{}.
\newblock \showarticletitle{mice: Multivariate imputation by chained equations
  in R}.
\newblock \bibinfo{journal}{\emph{Journal of statistical software}}
  (\bibinfo{year}{2010}), \bibinfo{pages}{1--68}.
\newblock


\bibitem[\protect\citeauthoryear{Che, Purushotham, Cho, Sontag, and Liu}{Che
  et~al\mbox{.}}{2018}]%
        {Che2018rec}
\bibfield{author}{\bibinfo{person}{Zhengping Che}, \bibinfo{person}{Sanjay
  Purushotham}, \bibinfo{person}{Kyunghyun Cho}, \bibinfo{person}{David
  Sontag}, {and} \bibinfo{person}{Yan Liu}.} \bibinfo{year}{2018}\natexlab{}.
\newblock \showarticletitle{Recurrent Neural Networks for Multivariate Time
  Series with Missing Values}.
\newblock \bibinfo{journal}{\emph{Scientific Reports}} \bibinfo{volume}{8},
  \bibinfo{number}{1} (\bibinfo{year}{2018}), \bibinfo{pages}{6085}.
\newblock


\bibitem[\protect\citeauthoryear{Chen, Yin, Chen, Nguyen, Peng, and Li}{Chen
  et~al\mbox{.}}{2019}]%
        {chen2019CentralityInfoGraphConvolution}
\bibfield{author}{\bibinfo{person}{Hongxu Chen}, \bibinfo{person}{Hongzhi Yin},
  \bibinfo{person}{Tong Chen}, \bibinfo{person}{Quoc Viet~Hung Nguyen},
  \bibinfo{person}{Wen-Chih Peng}, {and} \bibinfo{person}{Xue Li}.}
  \bibinfo{year}{2019}\natexlab{}.
\newblock \showarticletitle{Exploiting centrality information with graph
  convolutions for network representation learning}. In
  \bibinfo{booktitle}{\emph{Proceedings of ICDE}}. \bibinfo{pages}{590--601}.
\newblock


\bibitem[\protect\citeauthoryear{Chen, Ma, and Xiao}{Chen
  et~al\mbox{.}}{2018}]%
        {chen2018fastgcn}
\bibfield{author}{\bibinfo{person}{Jie Chen}, \bibinfo{person}{Tengfei Ma},
  {and} \bibinfo{person}{Cao Xiao}.} \bibinfo{year}{2018}\natexlab{}.
\newblock \showarticletitle{Fast{GCN}: Fast Learning with Graph Convolutional
  Networks via Importance Sampling}. In \bibinfo{booktitle}{\emph{Proceedings
  of ICLR}}.
\newblock


\bibitem[\protect\citeauthoryear{Chiang, Liu, Si, Li, Bengio, and Hsieh}{Chiang
  et~al\mbox{.}}{2019}]%
        {chiang2019cluster}
\bibfield{author}{\bibinfo{person}{Wei-Lin Chiang}, \bibinfo{person}{Xuanqing
  Liu}, \bibinfo{person}{Si Si}, \bibinfo{person}{Yang Li},
  \bibinfo{person}{Samy Bengio}, {and} \bibinfo{person}{Cho-Jui Hsieh}.}
  \bibinfo{year}{2019}\natexlab{}.
\newblock \showarticletitle{Cluster-GCN: An Efficient Algorithm for Training
  Deep and Large Graph Convolutional Networks}. In
  \bibinfo{booktitle}{\emph{Proceedings of KDD}}. \bibinfo{pages}{257--266}.
\newblock


\bibitem[\protect\citeauthoryear{Choong, Liu, and Murata}{Choong
  et~al\mbox{.}}{2018}]%
        {Choong2018CommunityVAE}
\bibfield{author}{\bibinfo{person}{Jun~Jin Choong}, \bibinfo{person}{Xin Liu},
  {and} \bibinfo{person}{Tsuyoshi Murata}.} \bibinfo{year}{2018}\natexlab{}.
\newblock \showarticletitle{Learning community structure with variational
  autoencoder}. In \bibinfo{booktitle}{\emph{Proceedings of ICDM}}.
  \bibinfo{pages}{69--78}.
\newblock


\bibitem[\protect\citeauthoryear{Cui, Wang, Pei, and Zhu}{Cui
  et~al\mbox{.}}{2018}]%
        {cui2018survey_net_emb}
\bibfield{author}{\bibinfo{person}{Peng Cui}, \bibinfo{person}{Xiao Wang},
  \bibinfo{person}{Jian Pei}, {and} \bibinfo{person}{Wenwu Zhu}.}
  \bibinfo{year}{2018}\natexlab{}.
\newblock \showarticletitle{A survey on network embedding}.
\newblock \bibinfo{journal}{\emph{IEEE Transactions on Knowledge and Data
  Engineering}} \bibinfo{volume}{31}, \bibinfo{number}{5}
  (\bibinfo{year}{2018}), \bibinfo{pages}{833--852}.
\newblock


\bibitem[\protect\citeauthoryear{Dai, Li, Tian, Huang, Wang, Zhu, and Song}{Dai
  et~al\mbox{.}}{2018}]%
        {dai2018adversarial}
\bibfield{author}{\bibinfo{person}{Hanjun Dai}, \bibinfo{person}{Hui Li},
  \bibinfo{person}{Tian Tian}, \bibinfo{person}{Xin Huang},
  \bibinfo{person}{Lin Wang}, \bibinfo{person}{Jun Zhu}, {and}
  \bibinfo{person}{Le Song}.} \bibinfo{year}{2018}\natexlab{}.
\newblock \showarticletitle{Adversarial Attack on Graph Structured Data}
  \emph{(\bibinfo{series}{Proceedings of ICML})}. \bibinfo{pages}{1115--1124}.
\newblock


\bibitem[\protect\citeauthoryear{Dempster, Laird, and Rubin}{Dempster
  et~al\mbox{.}}{1977}]%
        {dempster1977}
\bibfield{author}{\bibinfo{person}{A.~P. Dempster}, \bibinfo{person}{N.~M.
  Laird}, {and} \bibinfo{person}{D.~B. Rubin}.}
  \bibinfo{year}{1977}\natexlab{}.
\newblock \showarticletitle{Maximum Likelihood from Incomplete Data via the EM
  Algorithm}.
\newblock \bibinfo{journal}{\emph{Journal of the Royal Statistical Society.
  Series B (Methodological)}} \bibinfo{volume}{39}, \bibinfo{number}{1}
  (\bibinfo{year}{1977}), \bibinfo{pages}{1--38}.
\newblock


\bibitem[\protect\citeauthoryear{Duvenaud, Maclaurin, Iparraguirre, Bombarell,
  Hirzel, Aspuru-Guzik, and Adams}{Duvenaud et~al\mbox{.}}{2015}]%
        {duvenaud2015GCN_MolecularFingerprints}
\bibfield{author}{\bibinfo{person}{David~K Duvenaud}, \bibinfo{person}{Dougal
  Maclaurin}, \bibinfo{person}{Jorge Iparraguirre}, \bibinfo{person}{Rafael
  Bombarell}, \bibinfo{person}{Timothy Hirzel}, \bibinfo{person}{Al{\'a}n
  Aspuru-Guzik}, {and} \bibinfo{person}{Ryan~P Adams}.}
  \bibinfo{year}{2015}\natexlab{}.
\newblock \showarticletitle{Convolutional networks on graphs for learning
  molecular fingerprints}. In \bibinfo{booktitle}{\emph{Proceedings of
  NeurIPS}}. \bibinfo{pages}{2224--2232}.
\newblock


\bibitem[\protect\citeauthoryear{Fout, Byrd, Shariat, and Ben-Hur}{Fout
  et~al\mbox{.}}{2017}]%
        {fout2017ProteinInterfacePrediction_GCN}
\bibfield{author}{\bibinfo{person}{Alex Fout}, \bibinfo{person}{Jonathon Byrd},
  \bibinfo{person}{Basir Shariat}, {and} \bibinfo{person}{Asa Ben-Hur}.}
  \bibinfo{year}{2017}\natexlab{}.
\newblock \showarticletitle{Protein interface prediction using graph
  convolutional networks}. In \bibinfo{booktitle}{\emph{Proceedings of
  NeurIPS}}. \bibinfo{pages}{6530--6539}.
\newblock


\bibitem[\protect\citeauthoryear{Garc{\'\i}a-Laencina, Sancho-G{\'o}mez, and
  Figueiras-Vidal}{Garc{\'\i}a-Laencina et~al\mbox{.}}{2010}]%
        {GarciaLaencina2010}
\bibfield{author}{\bibinfo{person}{Pedro~J Garc{\'\i}a-Laencina},
  \bibinfo{person}{Jos{\'e}-Luis Sancho-G{\'o}mez}, {and}
  \bibinfo{person}{An{\'\i}bal~R Figueiras-Vidal}.}
  \bibinfo{year}{2010}\natexlab{}.
\newblock \showarticletitle{Pattern classification with missing data: a
  review}.
\newblock \bibinfo{journal}{\emph{Neural Computing and Applications}}
  \bibinfo{volume}{19}, \bibinfo{number}{2} (\bibinfo{year}{2010}),
  \bibinfo{pages}{263--282}.
\newblock


\bibitem[\protect\citeauthoryear{Glorot and Bengio}{Glorot and Bengio}{2010}]%
        {glorot10understanding}
\bibfield{author}{\bibinfo{person}{Xavier Glorot} {and} \bibinfo{person}{Yoshua
  Bengio}.} \bibinfo{year}{2010}\natexlab{}.
\newblock \showarticletitle{Understanding the difficulty of training deep
  feedforward neural networks}. In \bibinfo{booktitle}{\emph{Proceedings of
  AISTATS}}. \bibinfo{pages}{249--256}.
\newblock


\bibitem[\protect\citeauthoryear{Hu, Pan, Long, Lu, Zhu, and Jiang}{Hu
  et~al\mbox{.}}{2020}]%
        {wan2020goingdeep}
\bibfield{author}{\bibinfo{person}{Ruiqi Hu}, \bibinfo{person}{Shirui Pan},
  \bibinfo{person}{Guodong Long}, \bibinfo{person}{Qinghua Lu},
  \bibinfo{person}{Liming Zhu}, {and} \bibinfo{person}{Jing Jiang}.}
  \bibinfo{year}{2020}\natexlab{}.
\newblock \showarticletitle{Going Deep: Graph Convolutional Ladder-Shape
  Networks}. In \bibinfo{booktitle}{\emph{Proceedings of AAAI}}.
\newblock


\bibitem[\protect\citeauthoryear{Ji, Pan, Cambria, Marttinen, and Yu}{Ji
  et~al\mbox{.}}{2020}]%
        {ji2020survey}
\bibfield{author}{\bibinfo{person}{Shaoxiong Ji}, \bibinfo{person}{Shirui Pan},
  \bibinfo{person}{Erik Cambria}, \bibinfo{person}{Pekka Marttinen}, {and}
  \bibinfo{person}{Philip~S. Yu}.} \bibinfo{year}{2020}\natexlab{}.
\newblock \bibinfo{title}{A Survey on Knowledge Graphs: Representation,
  Acquisition and Applications}.
\newblock
\newblock
\showeprint[arxiv]{2002.00388}


\bibitem[\protect\citeauthoryear{Jin, Liu, Li, He, and Zhang}{Jin
  et~al\mbox{.}}{2019}]%
        {jin2019GCN_MRF_CommunityDet}
\bibfield{author}{\bibinfo{person}{Di Jin}, \bibinfo{person}{Ziyang Liu},
  \bibinfo{person}{Weihao Li}, \bibinfo{person}{Dongxiao He}, {and}
  \bibinfo{person}{Weixiong Zhang}.} \bibinfo{year}{2019}\natexlab{}.
\newblock \showarticletitle{Graph convolutional networks meet Markov random
  fields: Semi-supervised community detection in attribute networks}. In
  \bibinfo{booktitle}{\emph{Proceedings of AAAI}}. \bibinfo{pages}{152--159}.
\newblock


\bibitem[\protect\citeauthoryear{{Kai Jiang}, {Haixia Chen}, and {Senmiao
  Yuan}}{{Kai Jiang} et~al\mbox{.}}{2005}]%
        {jiang2005ensemble}
\bibfield{author}{\bibinfo{person}{{Kai Jiang}}, \bibinfo{person}{{Haixia
  Chen}}, {and} \bibinfo{person}{{Senmiao Yuan}}.}
  \bibinfo{year}{2005}\natexlab{}.
\newblock \showarticletitle{Classification for Incomplete Data Using Classifier
  Ensembles}. In \bibinfo{booktitle}{\emph{Proceedings of ICNNB}}.
  \bibinfo{pages}{559--563}.
\newblock


\bibitem[\protect\citeauthoryear{Kingma and Ba}{Kingma and Ba}{2015}]%
        {KingmaB14Adam}
\bibfield{author}{\bibinfo{person}{Diederik~P. Kingma} {and}
  \bibinfo{person}{Jimmy Ba}.} \bibinfo{year}{2015}\natexlab{}.
\newblock \showarticletitle{Adam: {A} Method for Stochastic Optimization}. In
  \bibinfo{booktitle}{\emph{Proceedings of ICLR}}.
\newblock


\bibitem[\protect\citeauthoryear{Kingma and Welling}{Kingma and
  Welling}{2014}]%
        {Kingma2014}
\bibfield{author}{\bibinfo{person}{Diederik~P. Kingma} {and}
  \bibinfo{person}{Max Welling}.} \bibinfo{year}{2014}\natexlab{}.
\newblock \showarticletitle{{Auto-encoding variational bayes}}.
\newblock \bibinfo{journal}{\emph{Proceedings of ICLR}} (\bibinfo{year}{2014}),
  \bibinfo{pages}{1--14}.
\newblock


\bibitem[\protect\citeauthoryear{Kipf and Welling}{Kipf and Welling}{2016}]%
        {kipf2016variational}
\bibfield{author}{\bibinfo{person}{Thomas~N Kipf} {and} \bibinfo{person}{Max
  Welling}.} \bibinfo{year}{2016}\natexlab{}.
\newblock \showarticletitle{Variational Graph Auto-Encoders}.
\newblock \bibinfo{journal}{\emph{NIPS Workshop on Bayesian Deep Learning}}
  (\bibinfo{year}{2016}).
\newblock


\bibitem[\protect\citeauthoryear{Kipf and Welling}{Kipf and Welling}{2017}]%
        {kipf2016GCN}
\bibfield{author}{\bibinfo{person}{Thomas~N Kipf} {and} \bibinfo{person}{Max
  Welling}.} \bibinfo{year}{2017}\natexlab{}.
\newblock \showarticletitle{Semi-supervised classification with graph
  convolutional networks}. In \bibinfo{booktitle}{\emph{Proceedings of ICLR}}.
\newblock


\bibitem[\protect\citeauthoryear{{Koren}, {Bell}, and {Volinsky}}{{Koren}
  et~al\mbox{.}}{2009}]%
        {koren2009mf}
\bibfield{author}{\bibinfo{person}{Y. {Koren}}, \bibinfo{person}{R. {Bell}},
  {and} \bibinfo{person}{C. {Volinsky}}.} \bibinfo{year}{2009}\natexlab{}.
\newblock \showarticletitle{Matrix Factorization Techniques for Recommender
  Systems}.
\newblock \bibinfo{journal}{\emph{Computer}} \bibinfo{volume}{42},
  \bibinfo{number}{8} (\bibinfo{year}{2009}), \bibinfo{pages}{30--37}.
\newblock


\bibitem[\protect\citeauthoryear{Li, Wang, Zhu, and Huang}{Li
  et~al\mbox{.}}{2018}]%
        {li2018AdaptiveGCN}
\bibfield{author}{\bibinfo{person}{Ruoyu Li}, \bibinfo{person}{Sheng Wang},
  \bibinfo{person}{Feiyun Zhu}, {and} \bibinfo{person}{Junzhou Huang}.}
  \bibinfo{year}{2018}\natexlab{}.
\newblock \showarticletitle{Adaptive graph convolutional neural networks}. In
  \bibinfo{booktitle}{\emph{Proceedings of AAAI}}.
\newblock


\bibitem[\protect\citeauthoryear{Li, Jiang, and Marlin}{Li
  et~al\mbox{.}}{2019}]%
        {li2018learning}
\bibfield{author}{\bibinfo{person}{Steven Cheng-Xian Li}, \bibinfo{person}{Bo
  Jiang}, {and} \bibinfo{person}{Benjamin Marlin}.}
  \bibinfo{year}{2019}\natexlab{}.
\newblock \showarticletitle{Learning from Incomplete Data with Generative
  Adversarial Networks}. In \bibinfo{booktitle}{\emph{Proceedings of ICLR}}.
\newblock


\bibitem[\protect\citeauthoryear{Liu, Murata, Kim, Kotarasu, and Zhuang}{Liu
  et~al\mbox{.}}{2019}]%
        {Liu2019NetEmbGRAbetta}
\bibfield{author}{\bibinfo{person}{Xin Liu}, \bibinfo{person}{Tsuyoshi Murata},
  \bibinfo{person}{Kyoung-Sook Kim}, \bibinfo{person}{Chatchawan Kotarasu},
  {and} \bibinfo{person}{Chenyi Zhuang}.} \bibinfo{year}{2019}\natexlab{}.
\newblock \showarticletitle{A general view for network embedding as matrix
  factorization}. In \bibinfo{booktitle}{\emph{Proceedings of WSDM}}.
  \bibinfo{pages}{375--383}.
\newblock


\bibitem[\protect\citeauthoryear{Liu, Dou, Yu, Deng, and Peng}{Liu
  et~al\mbox{.}}{2020}]%
        {liu2020alleviating}
\bibfield{author}{\bibinfo{person}{Zhiwei Liu}, \bibinfo{person}{Yingtong Dou},
  \bibinfo{person}{Philip~S. Yu}, \bibinfo{person}{Yutong Deng}, {and}
  \bibinfo{person}{Hao Peng}.} \bibinfo{year}{2020}\natexlab{}.
\newblock \showarticletitle{Alleviating the Inconsistency Problem of Applying
  Graph Neural Network to Fraud Detection}. In
  \bibinfo{booktitle}{\emph{Proceedings of SIGIR}}.
\newblock


\bibitem[\protect\citeauthoryear{Luo, Cai, ZHANG, Xu, and xiaojie}{Luo
  et~al\mbox{.}}{2018}]%
        {Luo2018Multi}
\bibfield{author}{\bibinfo{person}{Yonghong Luo}, \bibinfo{person}{Xiangrui
  Cai}, \bibinfo{person}{Ying ZHANG}, \bibinfo{person}{Jun Xu}, {and}
  \bibinfo{person}{Yuan xiaojie}.} \bibinfo{year}{2018}\natexlab{}.
\newblock \showarticletitle{Multivariate Time Series Imputation with Generative
  Adversarial Networks}. In \bibinfo{booktitle}{\emph{Proceedings of NeurIPS}}.
  \bibinfo{pages}{1596--1607}.
\newblock


\bibitem[\protect\citeauthoryear{Maurya, Liu, and Murata}{Maurya
  et~al\mbox{.}}{2019}]%
        {sunil2019BetweennessGNN}
\bibfield{author}{\bibinfo{person}{Sunil~Kumar Maurya}, \bibinfo{person}{Xin
  Liu}, {and} \bibinfo{person}{Tsuyoshi Murata}.}
  \bibinfo{year}{2019}\natexlab{}.
\newblock \showarticletitle{Fast Approximations of Betweenness Centrality with
  Graph Neural Networks}. In \bibinfo{booktitle}{\emph{Proceedings of CIKM}}.
  \bibinfo{pages}{2149--2152}.
\newblock


\bibitem[\protect\citeauthoryear{Mazumder, Hastie, and Tibshirani}{Mazumder
  et~al\mbox{.}}{2010}]%
        {mazumder2010soft}
\bibfield{author}{\bibinfo{person}{Rahul Mazumder}, \bibinfo{person}{Trevor
  Hastie}, {and} \bibinfo{person}{Robert Tibshirani}.}
  \bibinfo{year}{2010}\natexlab{}.
\newblock \showarticletitle{Spectral Regularization Algorithms for Learning
  Large Incomplete Matrices}.
\newblock \bibinfo{journal}{\emph{J. Mach. Learn. Res.}}  \bibinfo{volume}{11}
  (\bibinfo{year}{2010}), \bibinfo{pages}{2287--2322}.
\newblock


\bibitem[\protect\citeauthoryear{McAuley, Targett, Shi, and van~den
  Hengel}{McAuley et~al\mbox{.}}{2015}]%
        {amazondata2015}
\bibfield{author}{\bibinfo{person}{Julian McAuley},
  \bibinfo{person}{Christopher Targett}, \bibinfo{person}{Qinfeng Shi}, {and}
  \bibinfo{person}{Anton van~den Hengel}.} \bibinfo{year}{2015}\natexlab{}.
\newblock \showarticletitle{Image-Based Recommendations on Styles and
  Substitutes}. In \bibinfo{booktitle}{\emph{Proceedings of SIGIR}}.
  \bibinfo{pages}{43--52}.
\newblock


\bibitem[\protect\citeauthoryear{Narasimhan, Lazebnik, and Schwing}{Narasimhan
  et~al\mbox{.}}{2018}]%
        {narasimhan2018VisualQA_GCN}
\bibfield{author}{\bibinfo{person}{Medhini Narasimhan},
  \bibinfo{person}{Svetlana Lazebnik}, {and} \bibinfo{person}{Alexander
  Schwing}.} \bibinfo{year}{2018}\natexlab{}.
\newblock \showarticletitle{Out of the box: Reasoning with graph convolution
  nets for factual visual question answering}. In
  \bibinfo{booktitle}{\emph{Proceedings of NeurIPS}}.
  \bibinfo{pages}{2654--2665}.
\newblock


\bibitem[\protect\citeauthoryear{Pan, Hu, Long, Jiang, Yao, and Zhang}{Pan
  et~al\mbox{.}}{2018}]%
        {pan2018adversarially}
\bibfield{author}{\bibinfo{person}{Shirui Pan}, \bibinfo{person}{Ruiqi Hu},
  \bibinfo{person}{Guodong Long}, \bibinfo{person}{Jing Jiang},
  \bibinfo{person}{Lina Yao}, {and} \bibinfo{person}{Chengqi Zhang}.}
  \bibinfo{year}{2018}\natexlab{}.
\newblock \showarticletitle{Adversarially Regularized Graph Autoencoder for
  Graph Embedding}. In \bibinfo{booktitle}{\emph{Proceedings of IJCAI}}.
  \bibinfo{pages}{2609--2615}.
\newblock


\bibitem[\protect\citeauthoryear{Pelckmans, Brabanter, Suykens, and
  Moor}{Pelckmans et~al\mbox{.}}{2005}]%
        {pelckmans2005svm}
\bibfield{author}{\bibinfo{person}{K. Pelckmans}, \bibinfo{person}{J.~De
  Brabanter}, \bibinfo{person}{J.A.K. Suykens}, {and} \bibinfo{person}{B.~De
  Moor}.} \bibinfo{year}{2005}\natexlab{}.
\newblock \showarticletitle{Handling missing values in support vector machine
  classifiers}.
\newblock \bibinfo{journal}{\emph{Neural Networks}} \bibinfo{volume}{18},
  \bibinfo{number}{5} (\bibinfo{year}{2005}), \bibinfo{pages}{684--692}.
\newblock


\bibitem[\protect\citeauthoryear{Perozzi, Al-Rfou, and Skiena}{Perozzi
  et~al\mbox{.}}{2014}]%
        {Perozzi2014deep}
\bibfield{author}{\bibinfo{person}{Bryan Perozzi}, \bibinfo{person}{Rami
  Al-Rfou}, {and} \bibinfo{person}{Steven Skiena}.}
  \bibinfo{year}{2014}\natexlab{}.
\newblock \showarticletitle{DeepWalk: Online Learning of Social
  Representations}. In \bibinfo{booktitle}{\emph{Proceedings of KDD}}.
  \bibinfo{pages}{701--710}.
\newblock


\bibitem[\protect\citeauthoryear{Qiu, Dong, Ma, Li, Wang, and Tang}{Qiu
  et~al\mbox{.}}{2018}]%
        {qiu2017unifying_dw_n2v_line}
\bibfield{author}{\bibinfo{person}{Jiezhong Qiu}, \bibinfo{person}{Yuxiao
  Dong}, \bibinfo{person}{Hao Ma}, \bibinfo{person}{Jian Li},
  \bibinfo{person}{Kuansan Wang}, {and} \bibinfo{person}{Jie Tang}.}
  \bibinfo{year}{2018}\natexlab{}.
\newblock \showarticletitle{Network embedding as matrix factorization: unifying
  DeepWalk, LINE, PTE, and node2vec}. In \bibinfo{booktitle}{\emph{Proceedings
  of WSDM}}. \bibinfo{pages}{459--467}.
\newblock


\bibitem[\protect\citeauthoryear{Rubin}{Rubin}{2004}]%
        {rubin2004multiple}
\bibfield{author}{\bibinfo{person}{Donald~B Rubin}.}
  \bibinfo{year}{2004}\natexlab{}.
\newblock \bibinfo{booktitle}{\emph{Multiple imputation for nonresponse in
  surveys}}. Vol.~\bibinfo{volume}{81}.
\newblock \bibinfo{publisher}{John Wiley \& Sons}.
\newblock


\bibitem[\protect\citeauthoryear{Scarselli, Gori, Tsoi, Hagenbuchner, and
  Monfardini}{Scarselli et~al\mbox{.}}{2008}]%
        {scarselli2008gnn_model}
\bibfield{author}{\bibinfo{person}{Franco Scarselli}, \bibinfo{person}{Marco
  Gori}, \bibinfo{person}{Ah~Chung Tsoi}, \bibinfo{person}{Markus
  Hagenbuchner}, {and} \bibinfo{person}{Gabriele Monfardini}.}
  \bibinfo{year}{2008}\natexlab{}.
\newblock \showarticletitle{The graph neural network model}.
\newblock \bibinfo{journal}{\emph{IEEE Transactions on Neural Networks}}
  \bibinfo{volume}{20}, \bibinfo{number}{1} (\bibinfo{year}{2008}),
  \bibinfo{pages}{61--80}.
\newblock


\bibitem[\protect\citeauthoryear{Sen, Namata, Bilgic, Getoor, Galligher, and
  Eliassi-Rad}{Sen et~al\mbox{.}}{2008}]%
        {sen2008col}
\bibfield{author}{\bibinfo{person}{Prithviraj Sen}, \bibinfo{person}{Galileo
  Namata}, \bibinfo{person}{Mustafa Bilgic}, \bibinfo{person}{Lise Getoor},
  \bibinfo{person}{Brian Galligher}, {and} \bibinfo{person}{Tina Eliassi-Rad}.}
  \bibinfo{year}{2008}\natexlab{}.
\newblock \showarticletitle{Collective Classification in Network Data}.
\newblock \bibinfo{journal}{\emph{AI Magazine}} \bibinfo{volume}{29},
  \bibinfo{number}{3} (\bibinfo{year}{2008}), \bibinfo{pages}{93}.
\newblock


\bibitem[\protect\citeauthoryear{Shi, Tang, Zhu, and Liu}{Shi
  et~al\mbox{.}}{2019}]%
        {shi2019featureattention}
\bibfield{author}{\bibinfo{person}{Min Shi}, \bibinfo{person}{Yufei Tang},
  \bibinfo{person}{Xingquan Zhu}, {and} \bibinfo{person}{Jianxun Liu}.}
  \bibinfo{year}{2019}\natexlab{}.
\newblock \bibinfo{title}{Feature-Attention Graph Convolutional Networks for
  Noise Resilient Learning}.
\newblock
\newblock
\showeprint[arxiv]{1912.11755}


\bibitem[\protect\citeauthoryear{Shuman, Narang, Frossard, Ortega, and
  Vandergheynst}{Shuman et~al\mbox{.}}{2013}]%
        {Shuman2013signal_processing_graph}
\bibfield{author}{\bibinfo{person}{David~I Shuman}, \bibinfo{person}{Sunil~K
  Narang}, \bibinfo{person}{Pascal Frossard}, \bibinfo{person}{Antonio Ortega},
  {and} \bibinfo{person}{Pierre Vandergheynst}.}
  \bibinfo{year}{2013}\natexlab{}.
\newblock \showarticletitle{The emerging field of signal processing on graphs:
  Extending high-dimensional data analysis to networks and other irregular
  domains}.
\newblock \bibinfo{journal}{\emph{IEEE signal processing magazine}}
  \bibinfo{volume}{30}, \bibinfo{number}{3} (\bibinfo{year}{2013}),
  \bibinfo{pages}{83--98}.
\newblock


\bibitem[\protect\citeauthoryear{Simonovsky and Komodakis}{Simonovsky and
  Komodakis}{2017}]%
        {simonovsky2017DynamicFilterGCN}
\bibfield{author}{\bibinfo{person}{Martin Simonovsky} {and}
  \bibinfo{person}{Nikos Komodakis}.} \bibinfo{year}{2017}\natexlab{}.
\newblock \showarticletitle{Dynamic edge-conditioned filters in convolutional
  neural networks on graphs}. In \bibinfo{booktitle}{\emph{Proceedings of
  CVPR}}. \bibinfo{pages}{3693--3702}.
\newblock


\bibitem[\protect\citeauthoryear{{\'S}mieja, Struski, Tabor, and
  Marzec}{{\'S}mieja et~al\mbox{.}}{2019}]%
        {Smieja2019}
\bibfield{author}{\bibinfo{person}{Marek {\'S}mieja}, \bibinfo{person}{{\L}ukas
  Struski}, \bibinfo{person}{Jacek Tabor}, {and} \bibinfo{person}{Mateusz
  Marzec}.} \bibinfo{year}{2019}\natexlab{}.
\newblock \showarticletitle{Generalized RBF kernel for incomplete data}.
\newblock \bibinfo{journal}{\emph{Knowledge-Based Systems}}
  \bibinfo{volume}{173} (\bibinfo{year}{2019}), \bibinfo{pages}{150--162}.
\newblock


\bibitem[\protect\citeauthoryear{{\'S}mieja, Struski, Tabor, Zieli{\'n}ski, and
  Spurek}{{\'S}mieja et~al\mbox{.}}{2018}]%
        {smieja2018GMMC}
\bibfield{author}{\bibinfo{person}{Marek {\'S}mieja},
  \bibinfo{person}{{\L}ukasz Struski}, \bibinfo{person}{Jacek Tabor},
  \bibinfo{person}{Bartosz Zieli{\'n}ski}, {and} \bibinfo{person}{Przemys{\l}aw
  Spurek}.} \bibinfo{year}{2018}\natexlab{}.
\newblock \showarticletitle{Processing of missing data by neural networks}. In
  \bibinfo{booktitle}{\emph{Proceedings of NeurIPS}}.
  \bibinfo{pages}{2719--2729}.
\newblock


\bibitem[\protect\citeauthoryear{Smola, Vishwanathan, and Hofmann}{Smola
  et~al\mbox{.}}{2005}]%
        {Smola2005KernelMF}
\bibfield{author}{\bibinfo{person}{Alexander~J. Smola},
  \bibinfo{person}{S.~V.~N. Vishwanathan}, {and} \bibinfo{person}{Thomas
  Hofmann}.} \bibinfo{year}{2005}\natexlab{}.
\newblock \showarticletitle{Kernel Methods for Missing Variables}. In
  \bibinfo{booktitle}{\emph{Proceedings of AISTATS}}.
\newblock


\bibitem[\protect\citeauthoryear{Spinelli, Scardapane, and Aurelio}{Spinelli
  et~al\mbox{.}}{2020}]%
        {spinelli2019ginn}
\bibfield{author}{\bibinfo{person}{Indro Spinelli}, \bibinfo{person}{Simone
  Scardapane}, {and} \bibinfo{person}{Uncini Aurelio}.}
  \bibinfo{year}{2020}\natexlab{}.
\newblock \showarticletitle{Missing Data Imputation with Adversarially-trained
  Graph Convolutional Networks}.
\newblock \bibinfo{journal}{\emph{Neural Networks}}  \bibinfo{volume}{129}
  (\bibinfo{year}{2020}), \bibinfo{pages}{249--260}.
\newblock


\bibitem[\protect\citeauthoryear{Stekhoven and Bühlmann}{Stekhoven and
  Bühlmann}{2011}]%
        {Stekhoven2011missforest}
\bibfield{author}{\bibinfo{person}{Daniel~J. Stekhoven} {and}
  \bibinfo{person}{Peter Bühlmann}.} \bibinfo{year}{2011}\natexlab{}.
\newblock \showarticletitle{{MissForest—non-parametric missing value
  imputation for mixed-type data}}.
\newblock \bibinfo{journal}{\emph{Bioinformatics}} \bibinfo{volume}{28},
  \bibinfo{number}{1} (\bibinfo{year}{2011}), \bibinfo{pages}{112--118}.
\newblock


\bibitem[\protect\citeauthoryear{Tang, Qu, Wang, Zhang, Yan, and Mei}{Tang
  et~al\mbox{.}}{2015}]%
        {tang2015line}
\bibfield{author}{\bibinfo{person}{Jian Tang}, \bibinfo{person}{Meng Qu},
  \bibinfo{person}{Mingzhe Wang}, \bibinfo{person}{Ming Zhang},
  \bibinfo{person}{Jun Yan}, {and} \bibinfo{person}{Qiaozhu Mei}.}
  \bibinfo{year}{2015}\natexlab{}.
\newblock \showarticletitle{Line: Large-scale information network embedding}.
  In \bibinfo{booktitle}{\emph{Proceedings of WWW}}.
  \bibinfo{pages}{1067--1077}.
\newblock


\bibitem[\protect\citeauthoryear{Veli{\v{c}}kovi{\'c}, Cucurull, Casanova,
  Romero, Lio, and Bengio}{Veli{\v{c}}kovi{\'c} et~al\mbox{.}}{2018}]%
        {velivckovic2017GAT}
\bibfield{author}{\bibinfo{person}{Petar Veli{\v{c}}kovi{\'c}},
  \bibinfo{person}{Guillem Cucurull}, \bibinfo{person}{Arantxa Casanova},
  \bibinfo{person}{Adriana Romero}, \bibinfo{person}{Pietro Lio}, {and}
  \bibinfo{person}{Yoshua Bengio}.} \bibinfo{year}{2018}\natexlab{}.
\newblock \showarticletitle{Graph attention networks}. In
  \bibinfo{booktitle}{\emph{Proceedings of ICLR}}.
\newblock


\bibitem[\protect\citeauthoryear{Wang, Cui, and Zhu}{Wang
  et~al\mbox{.}}{2016}]%
        {wang2016SDNE}
\bibfield{author}{\bibinfo{person}{Daixin Wang}, \bibinfo{person}{Peng Cui},
  {and} \bibinfo{person}{Wenwu Zhu}.} \bibinfo{year}{2016}\natexlab{}.
\newblock \showarticletitle{Structural deep network embedding}. In
  \bibinfo{booktitle}{\emph{Proceedings of KDD}}. \bibinfo{pages}{1225--1234}.
\newblock


\bibitem[\protect\citeauthoryear{{Wang}, {Mao}, {Wang}, and {Guo}}{{Wang}
  et~al\mbox{.}}{2017}]%
        {wang2017kgsurvey}
\bibfield{author}{\bibinfo{person}{Q. {Wang}}, \bibinfo{person}{Z. {Mao}},
  \bibinfo{person}{B. {Wang}}, {and} \bibinfo{person}{L. {Guo}}.}
  \bibinfo{year}{2017}\natexlab{}.
\newblock \showarticletitle{Knowledge Graph Embedding: A Survey of Approaches
  and Applications}.
\newblock \bibinfo{journal}{\emph{IEEE Transactions on Knowledge and Data
  Engineering}} \bibinfo{volume}{29}, \bibinfo{number}{12}
  (\bibinfo{year}{2017}), \bibinfo{pages}{2724--2743}.
\newblock


\bibitem[\protect\citeauthoryear{Williams, Liao, Xue, and Carin}{Williams
  et~al\mbox{.}}{2005}]%
        {david2005lr}
\bibfield{author}{\bibinfo{person}{David Williams}, \bibinfo{person}{Xuejun
  Liao}, \bibinfo{person}{Ya Xue}, {and} \bibinfo{person}{Lawrence Carin}.}
  \bibinfo{year}{2005}\natexlab{}.
\newblock \showarticletitle{Incomplete-Data Classification Using Logistic
  Regression}. In \bibinfo{booktitle}{\emph{Proceedings of ICML}}.
  \bibinfo{pages}{972--979}.
\newblock


\bibitem[\protect\citeauthoryear{{Wu}, {Pan}, {Chen}, {Long}, {Zhang}, and
  {Yu}}{{Wu} et~al\mbox{.}}{2020}]%
        {Wu2019survey_gnn}
\bibfield{author}{\bibinfo{person}{Z. {Wu}}, \bibinfo{person}{S. {Pan}},
  \bibinfo{person}{F. {Chen}}, \bibinfo{person}{G. {Long}}, \bibinfo{person}{C.
  {Zhang}}, {and} \bibinfo{person}{P.~S. {Yu}}.}
  \bibinfo{year}{2020}\natexlab{}.
\newblock \showarticletitle{A Comprehensive Survey on Graph Neural Networks}.
\newblock \bibinfo{journal}{\emph{IEEE Transactions on Neural Networks and
  Learning Systems}} (\bibinfo{year}{2020}), \bibinfo{pages}{1--21}.
\newblock


\bibitem[\protect\citeauthoryear{Yang, Kang, Cao, Jin, Yang, and Guo}{Yang
  et~al\mbox{.}}{2019}]%
        {yang2019topology}
\bibfield{author}{\bibinfo{person}{Liang Yang}, \bibinfo{person}{Zesheng Kang},
  \bibinfo{person}{Xiaochun Cao}, \bibinfo{person}{Di Jin}, \bibinfo{person}{Bo
  Yang}, {and} \bibinfo{person}{Yuanfang Guo}.}
  \bibinfo{year}{2019}\natexlab{}.
\newblock \showarticletitle{Topology Optimization based Graph Convolutional
  Network}. In \bibinfo{booktitle}{\emph{Proceedings of IJCAI}}.
  \bibinfo{pages}{4054--4061}.
\newblock


\bibitem[\protect\citeauthoryear{Yang, Wu, Gu, Wang, Cao, Jin, and Guo}{Yang
  et~al\mbox{.}}{2020}]%
        {yang2020gaton}
\bibfield{author}{\bibinfo{person}{Liang Yang}, \bibinfo{person}{Fan Wu},
  \bibinfo{person}{Junhua Gu}, \bibinfo{person}{Chuan Wang},
  \bibinfo{person}{Xiaochun Cao}, \bibinfo{person}{Di Jin}, {and}
  \bibinfo{person}{Yuanfang Guo}.} \bibinfo{year}{2020}\natexlab{}.
\newblock \showarticletitle{Graph Attention Topic Modeling Network}. In
  \bibinfo{booktitle}{\emph{Proceedings of The Web Conference}}.
  \bibinfo{pages}{144--154}.
\newblock


\bibitem[\protect\citeauthoryear{Yang, Cohen, and Salakhutdinov}{Yang
  et~al\mbox{.}}{2016}]%
        {Yang2016revisiting}
\bibfield{author}{\bibinfo{person}{Zhilin Yang}, \bibinfo{person}{William~W.
  Cohen}, {and} \bibinfo{person}{Ruslan Salakhutdinov}.}
  \bibinfo{year}{2016}\natexlab{}.
\newblock \showarticletitle{Revisiting Semi-Supervised Learning with Graph
  Embeddings}. In \bibinfo{booktitle}{\emph{Proceedings of ICML}}.
  \bibinfo{pages}{40--48}.
\newblock


\bibitem[\protect\citeauthoryear{Yi, Lee, Hwang, and Yang}{Yi
  et~al\mbox{.}}{2020}]%
        {yi2020why}
\bibfield{author}{\bibinfo{person}{Joonyoung Yi}, \bibinfo{person}{Juhyuk Lee},
  \bibinfo{person}{Sungju Hwang}, {and} \bibinfo{person}{Eunho Yang}.}
  \bibinfo{year}{2020}\natexlab{}.
\newblock \showarticletitle{Why Not to Use Zero Imputation? Correcting Sparsity
  Bias in Training Neural Networks}. In \bibinfo{booktitle}{\emph{Proceedings
  of ICLR}}.
\newblock


\bibitem[\protect\citeauthoryear{Ying, He, Chen, Eksombatchai, Hamilton, and
  Leskovec}{Ying et~al\mbox{.}}{2018}]%
        {ying2018GCN_Recommendation}
\bibfield{author}{\bibinfo{person}{Rex Ying}, \bibinfo{person}{Ruining He},
  \bibinfo{person}{Kaifeng Chen}, \bibinfo{person}{Pong Eksombatchai},
  \bibinfo{person}{William~L Hamilton}, {and} \bibinfo{person}{Jure Leskovec}.}
  \bibinfo{year}{2018}\natexlab{}.
\newblock \showarticletitle{Graph convolutional neural networks for web-scale
  recommender systems}. In \bibinfo{booktitle}{\emph{Proceedings of KDD}}.
  \bibinfo{pages}{974--983}.
\newblock


\bibitem[\protect\citeauthoryear{Yoon, Jordon, and van~der Schaar}{Yoon
  et~al\mbox{.}}{2018}]%
        {yoon2018gain}
\bibfield{author}{\bibinfo{person}{Jinsung Yoon}, \bibinfo{person}{James
  Jordon}, {and} \bibinfo{person}{Mihaela van~der Schaar}.}
  \bibinfo{year}{2018}\natexlab{}.
\newblock \showarticletitle{{GAIN}: Missing Data Imputation using Generative
  Adversarial Nets}. In \bibinfo{booktitle}{\emph{Proceedings of ICML}},
  Vol.~\bibinfo{volume}{80}. \bibinfo{pages}{5689--5698}.
\newblock


\bibitem[\protect\citeauthoryear{Zhang, Cui, Neumann, and Chen}{Zhang
  et~al\mbox{.}}{2018}]%
        {zhang2018DLGraphClassification}
\bibfield{author}{\bibinfo{person}{Muhan Zhang}, \bibinfo{person}{Zhicheng
  Cui}, \bibinfo{person}{Marion Neumann}, {and} \bibinfo{person}{Yixin Chen}.}
  \bibinfo{year}{2018}\natexlab{}.
\newblock \showarticletitle{An end-to-end deep learning architecture for graph
  classification}. In \bibinfo{booktitle}{\emph{Proceedings of AAAI}}.
\newblock


\bibitem[\protect\citeauthoryear{Zhou, Cui, Zhang, Yang, Liu, Wang, Li, and
  Sun}{Zhou et~al\mbox{.}}{2018}]%
        {zhou2018gnn_review}
\bibfield{author}{\bibinfo{person}{Jie Zhou}, \bibinfo{person}{Ganqu Cui},
  \bibinfo{person}{Zhengyan Zhang}, \bibinfo{person}{Cheng Yang},
  \bibinfo{person}{Zhiyuan Liu}, \bibinfo{person}{Lifeng Wang},
  \bibinfo{person}{Changcheng Li}, {and} \bibinfo{person}{Maosong Sun}.}
  \bibinfo{year}{2018}\natexlab{}.
\newblock \showarticletitle{Graph neural networks: A review of methods and
  applications}.
\newblock \bibinfo{journal}{\emph{arXiv:1812.08434}} (\bibinfo{year}{2018}).
\newblock


\bibitem[\protect\citeauthoryear{Zhu, Zhang, Cui, and Zhu}{Zhu
  et~al\mbox{.}}{2019}]%
        {zhu2019robust}
\bibfield{author}{\bibinfo{person}{Dingyuan Zhu}, \bibinfo{person}{Ziwei
  Zhang}, \bibinfo{person}{Peng Cui}, {and} \bibinfo{person}{Wenwu Zhu}.}
  \bibinfo{year}{2019}\natexlab{}.
\newblock \showarticletitle{Robust Graph Convolutional Networks Against
  Adversarial Attacks}. In \bibinfo{booktitle}{\emph{Proceedings of KDD}}.
  \bibinfo{pages}{1399–1407}.
\newblock


\bibitem[\protect\citeauthoryear{Z{\"u}gner, Akbarnejad, and
  G{\"u}nnemann}{Z{\"u}gner et~al\mbox{.}}{2018}]%
        {zugner2018adversarial}
\bibfield{author}{\bibinfo{person}{Daniel Z{\"u}gner}, \bibinfo{person}{Amir
  Akbarnejad}, {and} \bibinfo{person}{Stephan G{\"u}nnemann}.}
  \bibinfo{year}{2018}\natexlab{}.
\newblock \showarticletitle{Adversarial Attacks on Neural Networks for Graph
  Data}. In \bibinfo{booktitle}{\emph{SIGKDD}}. \bibinfo{pages}{2847--2856}.
\newblock


\bibitem[\protect\citeauthoryear{Z{\"u}gner and G{\"u}nnemann}{Z{\"u}gner and
  G{\"u}nnemann}{2019}]%
        {zugner2019certifiable}
\bibfield{author}{\bibinfo{person}{Daniel Z{\"u}gner} {and}
  \bibinfo{person}{Stephan G{\"u}nnemann}.} \bibinfo{year}{2019}\natexlab{}.
\newblock \showarticletitle{Certifiable robustness and robust training for
  graph convolutional networks}. In \bibinfo{booktitle}{\emph{Proceedings of
  KDD}}. \bibinfo{pages}{246--256}.
\newblock


\end{thebibliography}

\end{document}